\documentclass{article}
\usepackage[preprint]{style}

\usepackage{algorithm}
\usepackage{algorithmic}

\usepackage[utf8]{inputenc} 
\usepackage[T1]{fontenc}    
\usepackage{hyperref}       
\usepackage{url}            
\usepackage{booktabs}       
\usepackage{amsfonts}       
\usepackage{nicefrac}       
\usepackage{microtype}      

\usepackage{upgreek}
\usepackage{xcolor}
\usepackage{dsfont}
\usepackage{multicol}
\usepackage{amsmath,amssymb,amsfonts,amsthm,bbm,graphicx,mathtools}

\newcommand{\iid}{\text{i.i.d.}\ }
\newcommand{\cX}{\mathcal{X}}
\newcommand{\va}{\varphi}
\newcommand{\Poi}{\mathrm{Poi}}
\newcommand{\Paren}[1]{\left(#1\right)}
\newcommand{\Br}[1]{\left[#1\right]}
\newcommand{\Brace}[1]{\left\{#1\right\}}
\newcommand{\Abs}[1]{\left|#1\right|}
\newcommand{\D}[1]{\left\lceil #1 \right\rfloor}
\newcommand{\EE}{\mathbb{E}}
\DeclareMathOperator*{\E}{\EE}
\DeclareMathOperator*{\ve}{\varepsilon}
\newcommand{\cP}{\mathcal{P}}
\newcommand{\C}[1]{
\boldsymbol\lceil\hspace{-0.075em} 
#1
\hspace{-0.075em} \boldsymbol\rceil}
\newtheorem{Theorem}{Theorem}
\newtheorem{Corollary}{Corollary}
\newtheorem{Lemma}{Lemma}
\newcommand{\bphi}{{\boldsymbol\phi}}
\DeclareMathOperator*{\argmin}{arg\,min}
\newcommand{\Ve}{\text{\large$\ve$}}
\newcommand{\indic}{\mathds{1}}
\newcommand{\Var}{\mathrm{Var}}
\newcommand{\Bern}{\mathrm{Bern}}
\newcommand{\bin}{\mathrm{bin}}
\DeclarePairedDelimiterX{\infdivx}[2]{(}{)}{
  #1\;\delimsize\|\;#2}
\newcommand{\infdiv}{D\infdivx}
\newcommand{\norm}[1]{\left\lVert#1\right\rVert}

\newcommand{\ignore}[1]{}

\usepackage{xcolor}
\hypersetup{
    colorlinks,
    linkcolor={blue!93!black},
    citecolor={blue!50!black},
}

\title{Profile Entropy:
A Fundamental 
Measure for 
the Learnability
and Compressibility of\\ Discrete Distributions }

\author{
  Yi~Hao\\
  Dept. of Electrical and Computer Engineering \\
 University of California, San Diego\\
  \texttt{yih179@ucsd.edu}
 \And
Alon~Orlitsky \\
 Dept. of  Electrical and Computer Engineering\\
 University of California, San Diego\\
  \texttt{alon@ucsd.edu}
}

\begin{document}

\maketitle

\begin{abstract}

The profile of a sample is the multiset of its symbol frequencies. We show that for samples of discrete distributions, profile entropy is a fundamental measure unifying the concepts of estimation, inference, and compression. Specifically, profile entropy  a) determines the speed of estimating the distribution relative to the best natural estimator; b) characterizes the rate of inferring all symmetric properties compared with the best estimator over any label-invariant distribution collection; c) serves as the limit of profile compression, for which we derive optimal near-linear-time block and sequential algorithms. To further our understanding of profile entropy, we investigate its attributes, provide algorithms for approximating its value, and determine its magnitude for numerous structural distribution families. 

\end{abstract}
\vspace{1em}
\tableofcontents

\section{Introduction}
Recent research in statistical machine learning,
ranging from neural-network training and online learning, to 
density estimation and property testing,
has advanced evaluation criteria beyond worst-case analysis.
New performance measures apply more refined metrics
relating the algorithm's accuracy and efficiency
to the problem's inherent structure.

Consider for example learning an unknown discrete distribution
from its \iid samples. 
Classical worst-case analysis 
states that in the worst case, 
the number of samples required to estimate a distribution 
to a given KL-divergence grows linearly in the alphabet size. 

However, this formulation is pessimistic.
Distributions are rarely the worst possible, and many
practical distributions can be estimated with significantly smaller samples.
Furthermore, once the sample is drawn, it reveals the distribution's
complexity and hence the hardness of the learning task.

Going beyond worst-case analysis, 
we design an \emph{adaptive} learning algorithm whose theoretical
guarantees vary according to the problem's simplicity. 
For example,~\cite{orlitsky2015competitive} recently proposed an 
estimator that instance-by-instance achieves
nearly the same performance as  
a genie algorithm designed with prior 
knowledge of the underlying distribution.

We introduce \emph{profile entropy}, a fundamental
measure for the complexity of discrete distributions,
and show that it connects three vital scientific tasks:
estimation, inference, and compression.
The resulting algorithms have guarantees  
directly relating to the data profile entropy, 
hence also adapt to the intrinsic simplicity of the tasks at hand. 

The next subsections, formalize the relevant concepts 
and present relevant prior works. 

\subsection{Sample Profiles and Their Entropy}  
Consider an arbitrary sequence $x^n$ over 
a finite or countably infinite alphabet $\cX$. 
The \emph{multiplicity} $\mu_y( x^{n})$ of a symbol $y\in\cX$ 
is the number of times $y$ appears in $x^n$.
The \emph{prevalence} of an integer $\mu$ 
is the number $\va_\mu(x^n)$ of symbols in $x^n$ with multiplicity $\mu$. 
The \emph{profile} of $x^n$ is the multiset $\va(x^n)$ 
of multiplicities of the symbols in $x^n$.
We refer to it as a profile of \emph{length} $n$.

The number $\mathcal D(S)$ of distinct elements in a multiset $S$ is 
its \emph{dimension}.
For convenience, 
we also write 
$\mathcal D(x^n)$ for profile dimension. 
Note that the dimension of a length-$n$ profile is 
at most $\min\{\sqrt{2n}, |\cX|\}$.

Let $\Delta_\cX$ be the collection of distributions over $\cX$, 
and $p$ be an arbitrary distribution in $\Delta_\cX$.
The profile $\Phi^n$ of an \iid sample $X^n\sim p$ is 
a random variable whose distribution depends on
only $p$ and $n$. We therefore write $\Phi^n\sim~p$, 
and call $H(\Phi^n)$ the \emph{profile entropy} with
respect to $(p, n)$.
Analogously, we call $\mathcal D_n:=\mathcal D(\Phi^n)$,
the \emph{profile dimension} associated with $(p, n)$
and write $\mathcal D_n\sim\nobreak p$.

Due to the dependence among multiplicities, 
the distributions of $\Phi^n$ and $\mathcal D_n$ are rather complex in general. 
To obtain clean expressions, we can adopt the standard \emph{Poisson sampling} technique 
and make the sample size a Poisson variable $N\sim \Poi(n)$, independent of the sample. 
As an example, 
\[
\E\Br{\mathcal D_N\sim p} = \sum_{i=1}^\infty \Paren{1-\prod_{x\in \cX} \Paren{1- e^{-np_x} \frac{(np_x)^i}{i!}}},
\]
where $p_x$ denotes the probability of symbol $x$ assigned by $p$. 
Note that sometimes we also write $p(x)$ instead of $p_x$ for notational convenience.
Despite the complex landscape of statistical dependency, in Theorem~\ref{thm:entro_eql_dim}, we show that $\mathcal D_n\sim p$ and $H(\Phi^n\!\sim\! p)$ are of the same order, with high probability and for every $p\in \Delta_\cX$. 
In Theorem~\ref{thm:exp_conc} and~\ref{thm:var_conc}, 
we show that $\mathcal D_n\sim p$ highly 
concentrates around a variant of its expectation. 
In Section~\ref{sec:att_pro}, we provide a much simpler quantity ${H^\mathcal{S}_n(p)}$ 
that well approximates the expectation variant of $\mathcal D_n\sim p$. 
Combined, these three results provide a precise characterization of both 
profile entropy and dimension. 
Leveraging this in Section~\ref{sec:pro_struct}, we derive nearly-tight bounds on the 
magnitude of profile entropy for several important structural 
distribution families, including log-concave and power-law.

\subsection{Applications 
and Prior Works}

In this section, we present several learning and compression
applications in which the 
profile entropy would play an important role. 
We also review related prior works with an emphasize on adaptive algorithms, 
and suppress the discussions on the worst-case analysis for brevity.

\subsubsection*{Basics and Significance}

The profile of a sample corresponds to the empirical distribution of symbols, 
and reflects the magnitudes of the actual symbol probabilities. 
Hence, the profile dimension, the number of distinct symbol frequencies, 
characterizes the variability of ranges the probabilities spread over. 
The sample profile's entropy, which by Theorem~\ref{thm:entro_eql_dim} 
is of the same order as its dimension,
admits the same interpretation.

Intuitively, samples from simple distributions tend to have low profile entropy, 
such as those from a $m$-piecewise distribution with small $m$, 
or one whose probability masses concentrate over some sparse set. 
The profile entropy is also likely to decrease as one reduces the 
sample size, since the sample contains less information regarding 
the variability of distribution probabilities. 
See Theorem~\ref{thm:pro_monot},~\ref{thm:mix_logcon}, 
and~\ref{thm:histgram} for a formal justification of these arguments.

From a statistical perspective, the profile of a sample is a sufficient 
statistic for estimating the probability multiset and 
any symmetric functional of the underlying distribution, 
such as entropy and support size. 
When we express a profile as a collection of multiplicity-prevalence pairs, 
the profile dimension is the size of this collection. 
Being of the same magnitude, the profile entropy is thus the effective size 
of a natural sufficient statistic for label-invariant inference. 

Profile entropy also directly connects
to adaptive testing and classification. 
Such connection arises from computing the \emph{profile probability}
~\cite{acharya2011competitive, acharya2012competitive}, 
the probability of observing a sample with the given profile.  
If the profile has entropy of $\mathcal H$, we can show that 
this computation problem 
has a time complexity of $\mathcal{O}(\exp(\tilde\Theta(\mathcal H) \log |\cX|))$. 
The result follows by the equivalence 
of the problem and computing 
the permanent of a rank-$\tilde\Theta(\mathcal H)$ matrix~\cite{barvinok1996two, vontobel2012bethe, vontobel2014bethe,
barvinok2016computing}.

Below, we introduce two important applications that are more involved, 
in which the profile entropy has essential connection to the 
statistical efficacy of adaptive learning.

\subsubsection*{Distribution Estimation}
Estimating unknown distributions from their samples is a statistical-inference cornerstone, 
and has numerous applications, ranging from biological studies~\cite{armananzas2008review} to language modeling~\cite{chen1999empirical}.

A learning algorithm in this setting is often referred to as a \emph{distribution estimator}, 
which is a functional $\hat{p}$ associating with every sequence 
$x^n$ over $\cX$ a distribution $\hat{p}_{x^n}\!\in \Delta_\cX$. Given a sample $X^n\sim p$, 
we measure the performance of $\hat p$ in estimating the (unknown) distribution $p$ 
with a loss function $\ell(p, \hat p_{\!_{X^n}}\!)$,
e.g., the $\ell_1$ distance and KL divergence.

A classical worst-case type result shows 
that any estimator that achieves a small 
$\ell_1$ loss of $\ve>0$ 
over $\Delta_\cX$ in expectation 
requires a sample size of $\Omega(|\cX|/\ve^2)$. 
Recent research further shows that 
the naive empirical-distribution estimator 
attains the optimal sample efficiency, to the right constants. 

The desire to design more efficient estimators for practical
distributions such as Poisson mixtures leads to 
two adaptive estimation frameworks: structural and competitive.

\emph{Structural} estimation focuses on distributions
possessing a natural structure, such as 
monotonicity, $m$-modality, and log-concavity. 
In many cases including the mentioned, structural assumptions lead to effective
estimators that provably perform better on the corresponding distribution 
classes. See~\cite{buhlmann2016learning} for a review of recent literature.

\emph{Competitive} estimation aims to design estimators that are universally near-optimal. 
Without strong structural knowledge, a reasonable estimator should 
\emph{naturally} assign the same probability to symbols appearing equal number of times.
The objective here is to find an estimator that
learns \emph{every} distribution as well as the best natural estimator
designed with knowledge of the true distribution. 
Discussion continues in Section~\ref{sec:comp_est} with a review of relevant works.

\subsubsection*{Property (Functional) Inference}

Instead of recovering the underlying distribution, 
numerous practical applications require only inferring a particular \emph{property value},
such as entropy for graphical modeling~\cite{koller2009probabilistic}, and support size for species richness estimation~\cite{magurran2013measuring}. 

Formally, a \emph{distribution property} over
a distribution collection $\cP\subseteq \Delta_\cX$ is a functional 
$f: \cP \to \mathbb R$
that associates with each distribution in $\cP$ a real value. 
Given a sample $X^n$ from an unknown distribution $p\in \cP$, 
the problem of interest is to infer the value of $f(p)$. 
To do this, we employee another functional $\hat f: \cX^*\to \mathbb R$, 
a \emph{property estimator} that maps every sample to a real value.

The statistical efficiency of $\hat f$ in estimating $f$ with respect to 
the distribution collection $\cP$ is 
measured by its \emph{sample complexity}. 
Specifically, for an accuracy $\ve>0$ and error tolerance $\delta\in(0,1)$, 
the $(\ve,\delta)$-sample complexity of $\hat f$ with respect to $(f, \cP)$ 
is the minimal sample size $n$ for which 
$\Pr_{X^n\sim p}(|\hat f(X^n)-f(p)|> \ve)\le \delta$ 
for all $p\in \cP$. Note that for the special case of $\cP=\{p\}$,
the sample complexity directly characterizes the ability of $\hat f$ 
in estimating $f(p)$. 

Recent years have shown interests in determining the sample complexities 
of inferring distribution properties. 
Built upon worst-case analysis, the major contribution of these works 
is establishing the sufficiency of sample sizes 
sub-linear in $|\cX|$. 
As an example, in the vital sample-sparse regime and over $\Delta_\cX$, the $(\ve, 1/10)$-sample 
complexity of learning entropy is $\Theta(|\cX|/(\ve\log|\cX|))$. 
We refer the readers to~\cite{verdu2019empirical} for a thorough survey of related works.

As the problem involves two components, the property and distribution, 
adaptive analysis also advances in two veins. 

The first vein concerns constructing a universal plug-in estimator 
for all \emph{symmetric properties}. 
A symmetric property is invariant under symbol permutations, 
hence it suffices to obtain an accurate estimate of the probability multiset. 
Recently, following the works of~\cite{das2012competitive, acharya2017unified}, 
\cite{hao2019broad} show that 
for any symmetric property that is additively separable and 
appropriately Lipschitz, the profile maximum likelihood 
(Section~\ref{sec:ada_prop_est}) achieves 
the optimal sample complexity up to small constant factors. 
Other major works include~\cite{valiant2011estimating, valiant2013estimating, 
valiant2016instance, han2018local, charikar2019efficient}

The second vein is an analogy to the competitive distribution estimation framework, 
and aims to compete with the \emph{instance-by-instance} performance
of a genie having access to more information, but reasonably restricted. 
A natural choice for the genie is the best-known and most-used -- the \emph{empirical estimator} 
that evaluates the property at the sample empirical distribution. 
To empower the genie, we grant it access to a
sample whose size is logarithmically larger than that available to the learner. 
One can show that this enables the genie to universally achieve the optimal sample complexities 
for numerous properties and hypothesis classes $\cP$. 
 Under this formulation, \cite{hao2018data, hao2019data} provide a unified learning algorithm
 that achieves the optimal competitiveness guarantees in near-linear time.

In this work, we further both veins of works and show~that: 
1) the PML plug-in estimator possesses the 
amazing ability of adapting to the simplicity of data 
distributions in inferring all symmetric properties,
over any label-invariant classes;
2) when plugged into entropy, the estimator in~\cite{hao2019doubly}
approximates the property as well as the 
plug-in estimator whose distribution component is the best natural, 
for every distribution. 
See Theorem~\ref{thm:adp_pml} and~\ref{thm:comp_entro} 
for the formal statements.

\section{New Results}

We establish essential connections between 
profile entropy and the estimation of distributions,  
inference of their properties, 
and compression of profiles. 
To further our understanding of profile entropy, 
we then investigate its attributes, 
provide algorithms for approximating its value, 
and determine its magnitude for numerous structural distribution families. 

For space considerations, we relegate most 
technical proofs to the appendices. 

\paragraph{Permutation invariance} 
By definition, both the profile of a sequence and its dimension 
are invariant to domain-symbol permutations. 
Since entropy is a symmetric property, 
the profile entropy of an \iid sample is also permutation invariant. 
Consequently, a result in this section that holds for a distribution
will also hold for \emph{any distributions sharing the same 
probability multiset}. 

This is desirable for practical applications, since samples often come as categorical data, 
while the symbol ordering under which the underlying distribution would exhibit certain structure is unknown to the learner. 
For example, in natural language processing, we observe words and punctuation marks. 
Given that the data comes from a power-law distribution~\cite{mitzenmacher2004brief}, 
we often don't know how to order the alphabet to 
realize such a condition. 

Surprisingly, with a few exceptions such as~\cite{hao2019doubly}, most previous works on learning 
structured discrete distributions do not address this crucial matter in their learning algorithms. 
The existing results are rather artificial and more like learning discretized continuous distributions.  
See Section~\ref{sec:pro_struct} for our discussion on distribution discretization. 
\vfill
\pagebreak

\subsection{Profile Dimension and Entropy}
Denote by $\C{x}$ the smallest integer larger than $x$. Then,
\begin{Theorem}\label{thm:entro_eql_dim}
For any distribution $p\in \Delta_\cX$ and $\Phi^n\sim p$,
with probability at least $1-\mathcal{O}(1/\sqrt n)$, 
\[
\C{H\!(\Phi^n)} = \tilde{\Theta}(\mathcal D(\Phi^n)), 
\]
where the notation $\tilde{\Theta}(\boldsymbol \cdot)$ hides logarithmic factors of $n$.
\end{Theorem}
The theorem shows that for every distribution 
and sampling parameter $n$, the induced profile 
entropy and profile dimension are of the same order, with high probability. 

Taking expectation and noting that $\mathcal D(\Phi^n)\!\in\![1, \sqrt{2n}]$ yield
\[
\C{H\!(\Phi^n\sim p)} 
= \tilde{\Theta}(\E_{\Phi^n\sim p}[\mathcal D(\Phi^n)]),
\ \forall p\in \Delta_\cX.
\]

\subsection{Adaptive Property Estimation}\label{sec:ada_prop_est}

\paragraph{Definitions} 
A profile $\boldsymbol\phi$ is said to have length $n$ if there exists
$x^n\in \cX^n$ satisfying $\bphi = \va(x^n)$. 
For every profile $\bphi$ of length $n$ and distribution collection $\cP\subseteq\Delta_\cX$, 
the \emph{profile maximum likelihood} (PML) estimator~\cite{orlitsky2004modeling} 
over $\cP$ maps $\bphi$ to a distribution 
\[
\cP_\bphi:=\argmin_{p\in \cP} \Pr_{X^n\sim p}\!\Paren{\va(X^n)=\bphi}, 
\]
that maximizes the probability of observing the profile~$\bphi$.
For any property $f$, let $\Ve_f(n, \delta, \cP)$ denote 
the smallest error that can be achieved by any estimator 
with a sample size $n$ and tolerance $\delta$ on the error probability.
This definition is equivalent to that of the sample complexity. 
Below, we assume that $\cP$ is \emph{label invariant}, i.e., for any $p\in \cP$, collection $\cP$ contains all its symbol-permuted versions.

We first show that profile-based estimators 
are sufficient for estimating symmetric properties. 
\begin{Theorem}[Sufficiency of profiles]\label{thm1}
Let $f$ be a symmetric property over $\cP$. For any accuracy $\ve>0$ and tolerance 
$\delta\in(0,1)$, if there exists an estimator $\hat{f}$ such that 
\[
\Pr_{X^n\sim p}
\Paren{\Abs{\hat f(X^n)-f(p)}> \ve}< \delta, \ \forall p\in \cP,
\]
there is an estimator $\hat f_\va$ over length-$n$ profiles satisfying 
\[
\Pr_{X^n\sim p}
\Paren{\Abs{\hat f_\va(\va(X^n))-f(p)}> \ve}< \delta, \ \forall p\in \cP.
\]
\end{Theorem} 
Note that both estimators can have independent randomness. 

The second result shows that the PML estimator is adaptive to the 
simplicity of underlying distributions in inferring all symmetric properties,
over any label-invariant $\cP$.   
For clarity, we set $\delta=1/10$ and 
suppress both $\delta$ and $\cP$ in $\Ve_f(n, \delta, \cP)$.  

\begin{Theorem}[Adaptiveness of PML]\label{thm:adp_pml}
Let $f$ be a symmetric property. 
For any $p\in \cP$ and $\Phi^n\sim p$, 
with probability at least $1-\mathcal{O}(1/\sqrt n)$, 
\[
\Abs{f(p)-f(\cP_{\Phi^n}\!)}\le 
2\Ve\!_f\!\Paren{\!
\frac{\tilde\Omega\!\Paren{n}}{\C{H\!(\Phi^n)}} 
\!}. 
\]
\end{Theorem}
Some comments: 
1) The theorem holds for any symmetric properties, 
while nearly all previous works require the property to possess certain forms and be smooth;
2) The theorem trivially implies a weaker result in~\cite{acharya2017unified} 
where $\C{H\!(\Phi^n)}$ is replaced by $\sqrt n$; 
3) There is a polynomial-time approximation~\cite{charikar2019bethe} 
achieving the same guarantee; 
4) We provide a stronger result in Section~\ref{sup:1.5} of the appendices 
for general $\delta$. 

Besides this theorem, we establish in Section~\ref{sec:pml_sort} 
and~\ref{sec:pml_uniform} two additional results on PML.
The first result addresses sorted distribution estimation, and improves over 
that established in~\cite{hao2019broad} (Theorem 5) 
in terms of the lower bound on the accuracy parameter $\varepsilon$.
Let $\Sigma_\cX$ denote the collection of symbol permutations over $\cX$. 
For any $p\in \Delta_\cX$, denote by $p_\sigma\in \Delta_\cX$ the permuted 
distribution satisfying $p_\sigma(x)=p({\sigma(x)})$ for all symbols $x\in \cX$. 
Let $\lambda>0$ be a positive absolute constant 
that can be made arbitrarily small, e.g., $\lambda=0.001$. 
\begin{Lemma}
For any $\varepsilon\in(0,1)$, $p\in \cP=\Delta_\cX$, and $\Phi^n\sim p$, 
if we have $n\ge\Omega(|\cX|/(\varepsilon^2\log |\cX|))$
and $\varepsilon\ge 1/n^{1/8-\lambda}$,
with probability at least $1-\mathcal{O}(\exp(-\sqrt n))$, 
\[
\min_{\sigma\in \Sigma_\cX}\norm{p_\sigma-\cP_{\Phi^n}}_1\le\mathcal{O}(\epsilon).
\]
\end{Lemma}
A few comments in order: 1) The polynomial-time computable variant of PML in~\cite{charikar2019bethe} satisfies the same guarantee, and the proof
for this is also similar to that in Section~\ref{sec:pml_sort}; 
2) Using the existing efficiently computable 
PML-type methods~\cite{charikar2019bethe,charikar2019efficient}, 
the best possible lower bound on $\varepsilon$ is $\Theta(1/n^{1/4})$;
3) Below the $\Theta(1/n^{1/3})$ threshold, the empirical 
distribution estimator is sample optimal 
up to constant factors~\cite{han2018local}.  

The second result shows an intriguing connection 
between the PML method and the task of uniformity testing~\cite{goldreich2011testing}.
See Section~\ref{sec:pml_uniform} for details.

Our last result in this section addresses entropy estimation. 
We show that when plugged into entropy, 
the estimator in~\cite{hao2019doubly}
approximates the property as well as the 
plug-in estimator whose distribution component is the best natural, 
for every distribution.

Recall that a distribution estimator is natural if it assigns 
the same probability to symbols of equal multiplicity, 
and a property estimator is plug-in if it first finds an estimate 
of the distribution and then evaluates the property at this estimate. 
As an off-the-shelf method, the plug-in approach is widely used 
in estimating distribution properties. 

If further the property is symmetric,
then it suffices to obtain an accurate estimate of the probability multiset,
which is intuitively more statistically
efficient than recovering the actual distribution. 
For example, \cite{hao2019broad} recently show that 
for any symmetric property that is additively separable and 
appropriately Lipschitz, the PML multiset estimator 
~\cite{orlitsky2004modeling} achieves 
the optimal sample complexity up to small constant factors. 

However, the analysis and computation (though efficient) 
of such multiset-based estimation methods are often involved~\cite{valiant2011estimating, valiant2013estimating, 
valiant2016instance, han2018local, charikar2019efficient, hao2019broad}. 
For this reason, distribution-based plug-in estimators are still popular in practice, and often, the distribution components are natural.

As an example, for entropy estimation, 
several widely used distribution-based estimators are \emph{natural  plug-in}, 
such as the empirical estimator plugging in the empirical distribution,    
James-Stein shrinkage~\cite{hausser2009entropy} that shrinks 
the distribution estimate towards uniform, 
and Dirichlet-smoothed~\cite{schurmann1996entropy} 
that imposes a Dirichlet prior over $\Delta_\cX$.

The logic behind these estimators is simple: if two distributions are close, then the same is expected to hold for their entropy values. The next theorm shows that for \emph{every} distribution and among all plug-in entropy estimators,
the distribution estimator in~\cite{hao2019doubly}
is as good as the one that performs best in estimating the actual distribution.

Denote by $\mathcal{N}$ the collection of all natural estimators. 
Write $|H(p)-H(q)|$ as $\ell_H(p,q)$ for compactness 
and the KL-divergence between $p,q\in\Delta_\cX$ 
as $\ell_{\text{\tiny KL}}(p,q)$.
\setcounter{Theorem}{3}
\begin{Theorem}[Competitive entropy estimation]\label{thm:comp_entro}
For any distribution $p$, sample $X^n\sim p$ with profile $\Phi^n\!:=\va(X^n)$,
and $\hat p_{\!_{X^n}}^{_{\mathcal{N}}}:=\argmin_{\hat p\in \mathcal{N}}\ell_{\text{\tiny KL}}(p,\hat p_{\!_{X^n}})$, 
we have
\[
\ell_H(p,\hat p^\star_{\!_{X^n}})\!-\!\ell_{H}(p,\hat p_{\!_{X^n}}^{_{\mathcal{N}}})
\le\tilde{\mathcal{O}}\Paren{\!\sqrt{\frac{\C{H\!(\Phi^n)}}{n}}}\!. 
\]
with probability at least $1-\mathcal{O}(1/n)$.
\end{Theorem}

\subsection{Competitive Distribution Estimation}\label{sec:comp_est}

\paragraph{Prior works} 
Competitive estimation calls for an estimator that competes 
with the instance-by-instance performance
of a genie knowing more information, but reasonably restricted. 
Denote by $\ell_{\text{\tiny KL}}(p,q)$ the KL divergence. 
Introduced in~\cite{orlitsky2015competitive}, 
the formulation considers the collection $\mathcal{N}$ 
of all natural estimators, and 
shows that a simple variant $\hat p^{\text{\tiny GT}}$ of the Good-Turing 
estimator achieves 
\[
\ell_{\text{\tiny KL}}(p,\hat p^{\text{\tiny GT}}_{\!_{X^n}})\!
-\!\min_{\hat p\in \mathcal{N}}\ell_{\text{\tiny KL}}(p,\hat p_{\!_{X^n}})
\le\frac{3+o(1)}{n^{1/3}}, 
\]
for every distribution $p$ and with high probability. 
We refer to the left-hand side as the \emph{excess loss} of estimator 
$\hat p_{\text{\tiny GT}}$ with respect to the best natural estimator, 
and note that it vanishes at a rate independent of $p$. 
For a more involved estimator in~\cite{acharya2013optimal}, the excess loss vanishes 
at a faster rate of $\tilde{\mathcal{O}}(\min\{1/\sqrt n, |\cX|/n\})$, 
optimal up to logarithmic factors for every estimator and 
the respective worst-case distribution. 
For the $\ell_1$ distance, \cite{valiant2016instance} 
derive a similar result.

These estimators track the loss of the best natural estimator for each distribution.
Yet an equally important component, the excess loss bound, is still of the worst-case nature. 
For a fully adaptive guarantee, \cite{hao2019doubly} design an estimator $\hat p^\star$ 
that achieves a $\mathcal D_n/n$ excess loss, i.e., 
\[
\ell_{\text{\tiny KL}}(p,\hat p^\star_{\!_{X^n}})\!
-\!\min_{\hat p\in \mathcal{N}}\ell_{\text{\tiny KL}}(p,\hat p_{\!_{X^n}})
\le\tilde{\mathcal{O}}\Paren{\frac{\mathcal D_n}{n}}, 
\]
 for every $p$ and $X^n\sim p$, with high probability. 
Utilizing the adaptiveness of $\mathcal D_n$ to the simplicity of distributions, 
the paper derives excess-loss bounds for several important distribution families, 
and proves the estimator's optimality under various of classical and
modern learning frameworks. 

\paragraph{New results}
While the work of~\cite{hao2019doubly} provides an appealing upper bound 
on the excess loss, 
it is not exactly clear how good this bound is as a matching lower bound is missing. 
In this work, we complete the picture by showing 
that the $\mathcal D_n/n$ bound is essential for competitive 
estimation and optimal up to 
logarithmic factors of $n$. 

\begin{Theorem}[Minimal excess loss] 
For any $n, \mathcal{D}\in \mathbb N$ and 
distribution estimator $\hat p'$, 
there is a distribution $p$ such that
with probability at least $9/10$, we have both  
\[
\mathcal{O}(\log n+\mathcal{D}) \ge  \mathcal D_n  
\] 
and 
\[
\ell_{\text{\tiny KL}}(p,\hat p'_{\!_{X^n}})\!
-\!\min_{\hat p\in \mathcal{N}}\ell_{\text{\tiny KL}}(p,\hat p_{\!_{X^n}})
\ge 
\Omega\Paren{\frac{\mathcal{D}}{n}}. 
\]
\end{Theorem}
By Theorem~\ref{thm:entro_eql_dim}, we can replace $\mathcal D_n$ by  
$\tilde\Theta(H(\Phi^n))$ in both the upper and lower bounds.
  
\subsection{Optimal Profile Compression}\label{sec:pro_compression}
While a labeled sample contains all information, 
for many modern applications, 
such as property estimation and differential privacy, 
it is sufficient~\cite{orlitsky2004modeling} or even necessary to provide only the 
profile~\cite{suresh2019differentially}. 
Hence, this section focuses on the lossless compression of profiles.

For any distribution $p$, it is well-known that 
the minimal expected codeword length (MECL) for 
losslessly compressing a sample 
$X^n\sim p$ is approximately $nH(p)$, 
which increases linearly in $n$
as long as $H(p)$ is bounded away from zero.

On the other hand, by the 
Hardy-Ramanujan formula~\cite{hardy1918asymptotic},
the number $\mathbb{P}(n)$ of integer partitions of $n$, 
which happens to equal to the number
of length-$n$ profiles, satisfies
\[
\log \mathbb{P}(n)= 2\pi\sqrt{\frac{n}{6}}(1+o(1)).
\]
Consequently, the MECL for losslessly compressing the sample 
profile $\Phi^n\sim p$ is at most $\mathcal{O}(\sqrt{n})$, 
a number potentially much smaller than $nH(p)$. 

By Shannon's source coding theorem, 
the profile entropy $H(\Phi^n)$ is the \emph{information-theoretic
limit} of MECL for the lossless compression of profile $\Phi^n\sim p$.
Below, we present explicit block and sequential 
profile compression schemes achieving this 
entropy limit, up to logarithmic factors of $n$. 

\paragraph{Block compression}
The block compression algorithm is intuitive and easy to implement.

Recall that the profile of a sequence $x^n$ is the multiset $\varphi(x^n)$
of multiplicities associated with symbols in $x^n$. 
The ordering of elements in a multiset is not informative. 
Hence equivalently, we can compress $\varphi(x^n)$
into the set $\mathcal C(\varphi(x^n))$ of corresponding multiplicity-prevalence pairs,
i.e.,
\[
\mathcal C(\varphi(x^n)):= \{(\mu, \varphi_\mu(x^n)):\mu\in \varphi(x^n)\}.
\]
The number of pairs in $\mathcal C(\varphi(x^n))$ is equal to the profile dimension 
$\mathcal{D}(\varphi(x^n))$. In addition, both a prevalence and its multiplicity
are integers in $[0,n]$, and storing the pair takes $2\log n$ nats. 
Hence, it takes at most $2(\log n)\cdot \mathcal{D}(\varphi(x^n))$ nats 
to store the compressed profile. 
By Theorem~\ref{thm:entro_eql_dim}, for any distribution $p\in \Delta_\cX$
and $\Phi^n\sim p$, 
\[
\E[2(\log n)\cdot \mathcal{D}(\Phi^n)] 
= 
\tilde{\Theta}(\C{H(\Phi^n)}). 
\]
\paragraph{Sequential compression} 
For any sequence $x^n$, 
the setting for sequential profile compression is that at time step $t\in [n]$, 
the compression algorithm knows only $\varphi(x^t)$ 
and sequentially encodes the new information. 
This is equivalent to providing the algorithm $\mu_{x_t}(x^{t-1})$ at time step $t$.

Suppress $x, x^{t}$ in the expressions for the ease of illustration. 
For efficient compression, 
we sequentially encode the profile $\va$ into a 
\emph{self-balancing binary search tree} $\mathcal T$,
with each node storing a 
multiplicity-prevalence pair $(\mu, \va_\mu)$ and $\mu$ being the search key. 
We present the algorithm details as follows. 
\begin{algorithm}
\caption{Sequential Profile Compression}
\begin{algorithmic} 
\INPUT sequence $(\mu_{x_t}(x^{t-1}))_{t=1}^n$, tree $\mathcal T=\varnothing$
\OUTPUT tree $\mathcal T$ that encodes the input sequence
\FOR {t = 1 to n}
\IF{$\mu:= \mu_{x_t}(x^{t-1}) \in \mathcal T$}
\IF{$\mu+1\in \mathcal T$}
\STATE $\va_{\mu+1}:=\mathcal T(\mu+1)\gets \mathcal T(\mu+1)+1$
\ELSE
\STATE add $(\mu+1, 1)$ to $\mathcal T$
\ENDIF
\STATE \textbf{if} $\va_\mu=1$ \textbf{then} delete $(\mu, \varphi_\mu)$ from $\mathcal T$ 
\STATE \textbf{else} $\va_{\mu}:=\mathcal T(\mu)\gets \mathcal T(\mu)-1$ \textbf{endif}
\ELSE
\STATE \textbf{if} $1\not\in \mathcal T$ \textbf{then} add $(1, 1)$ to $\mathcal T$ 
\STATE \textbf{else} $\mathcal T(1)\gets \mathcal T(1)+1$ \textbf{endif}
\ENDIF
\ENDFOR
\end{algorithmic}
\end{algorithm}

The algorithm runs for exactly $n$ iterations, 
with a $\mathcal{O}(\log n)$ per-iteration time complexity.
For an \iid sample $X^n\sim p$, the expected space complexity 
is again $\tilde{\Theta}(\C{H(\Phi^n)})$.

\subsection{Attributes of Profile Entropy and Dimension}\label{sec:att_pro}
To further our understanding of profile entropy and dimension, 
we investigate the analytical and statistical attributes of these characteristics
concerning their concentration, computation and approximation, 
monotonicity, and Lipschitzness.

\subsubsection*{Concentration}
Recall that the multiplicity $\mu_y(x^n)$ denotes the 
number of times symbol $y$ appearing in $x^n$. Denote by $\bigvee$ the 
logical OR operator. For any distribution $p$ and $X^n\sim p$, 
we have
\[
\mathcal D_n
=
\sum_{\mu=1}^n\  \bigvee_{x\in \cX} \indic_{\mu_x(X^n) = \mu}. 
\]
The statistical dependency landscape of terms in the summation is rather complex, 
since $\mu_x(X^n)$ and $\mu_y(X^n)$ are dependent  for every $(x,y)$ pair due to the fixed 
sample size; and so are $\indic_{\mu_x(X^n) = \mu_1}$ and $\indic_{\mu_x(X^n) = \mu_2}$ for 
every pair of distinct $\mu_1$ and $\mu_2$. 
To simplify the derivations, we relate this quantity to its variant under the aforementioned Poisson sampling scheme, i.e., 
making the sample size an independent $N\sim \Poi(n)$. Specifically, define
\[
\tilde{\mathcal D}_N
:=
\tilde{\mathcal D} (X^N) 
:=
\sum_{U=1}^n\  \bigvee_{x\in \cX} \indic_{\mu_x(X^N) = U}. 
\]
Note that this is not the same as $\mathcal D_N$ 
since the summation index goes up only to $n$. 
Denote the expected value of $\tilde{\mathcal D}_N$  by $E_n(p)$. 
Our result shows that the original $\mathcal D_n$ satisfies a 
Chernoff-Hoeffding type bound centered at $E_n(p)$. 

\begin{Theorem}\label{thm:exp_conc}
Under the above conditions and for any $n\in \mathbb Z^+$, $p\in \Delta_\cX$,
and $\gamma>0$,
\[
\Pr\Paren{\frac{\mathcal D_n}{1+\gamma} \ge   E_n(p)}
\le 3\sqrt n e^{-\min\{\gamma^2, \gamma\} E_n(p)/3},
\]
and for any $\gamma\in (0,1)$,  
\[
\Pr\Paren{\frac{\mathcal D_n}{1-\gamma} \le E_n(p)}
\le 3 \sqrt n e^{-\gamma^2 E_n(p)/2}. 
\]
\end{Theorem}
As a corollary, the value of $\mathcal D_n$ is often close to $E_n(p)$. 
\begin{Corollary}
Under the same conditions as above and for any $n\in \mathbb Z^+$ and
distribution $p\in \Delta_\cX$,
with probability at least $1-6/\sqrt n$, 
\[
\frac12 E_n(p)-4\log n \le \mathcal D_n \le 2 E_n(p)+3\log n.
\] 
\end{Corollary}
In addition, we establish an Efron-Stein type inequality.
\begin{Theorem}\label{thm:var_conc}
For any distribution $p$ and $\mathcal D_n\sim p$, 
\[
\Var(\mathcal D_n) \le \EE[\mathcal D_n]. 
\]
\end{Theorem}

\subsubsection*{Computation and Approximation}

The above results show that $\mathcal D_n\sim p$ is often close to $E_n(p)$, 
with an exponentially small deviation probability. 
Hence, to approximate $\mathcal D_n$, it suffices to accurately compute 
$E_n(p)$, the expectation of its Poissonized version $\tilde{\mathcal D}_N$. 
By independence and the linearity of expectations,
\[
E_n(p) 
= \sum_{i=1}^n \Paren{1-\prod_{x\in \cX} \Paren{1- e^{-np_x} \frac{(np_x)^i}{i!}}}.
\]
The expression is exact but does not relate to $p$ in a simple manner. 
For an intuitive approximation, we 
partition the unit interval into a sequence of ranges,
\[
I_j:=\left((j-1)^2\frac{\log n}{n},\ j^2 \frac{\log n}{n}\right], 1\le j\le \sqrt{\frac{n}{\log n}}.
\]
denote by $p_{I_j}$ the number of probabilities in $I_j$, 
and relate $E_n(p)$ to a shape-reflecting quantity
\[
H^\mathcal{S}_n(p):= \sum_{j\ge 1} \min\Brace{p_{I_j}, j\cdot \log n},
\]
the sum of the effective number of probabilities lying within each range~\cite{hao2019doubly}. 
To compute ${H^\mathcal{S}_n(p)}$, we simply count the number of probabilities in each $I_j$. 
Our main result shows that $H^\mathcal{S}_n(p)$ 
well approximates $E_n(p)$ over the entire $\Delta_\cX$, 
up to logarithmic factors of $n$. 
\begin{Theorem}\label{thm:hs_app_en}
For any $n\in \mathbb Z^+$ and $p\in \Delta_\cX$, 
\[
\frac{1}{\sqrt{\log n}} \cdot \Omega(H^\mathcal{S}_n(p))
\le 
E_n(p)
\le
\mathcal{O}(H^\mathcal{S}_n(p)).
\]
\end{Theorem}

\paragraph{Summary} The simple expression shows that 
$H^\mathcal{S}_n(p)$ characterizes the 
variability of ranges the actual probabilities spread over.
As Theorem~\ref{thm:hs_app_en} shows, 
$H^\mathcal{S}_n(p)$ closely approximates $E_n(p)$, 
the value around which $\mathcal D_n\sim p$ concentrates 
(Theorem~\ref{thm:exp_conc}). 
Henceforth, we use $H^\mathcal{S}_n(p)$ as a proxy for both 
$H(\Phi^n)$ and $\mathcal D_n$,  
and study its attributes and values. 

\subsubsection*{Monotonicity}
Among the many attributes that ${H^\mathcal{S}_n(p)}$ possesses,
monotonicity is perhaps most intuitive. 
One may expect a larger value of ${H^\mathcal{S}_n(p)}$ as the sample size $n$ increases, 
since additional observations reveal more information on the variability of probabilities. 
Below we confirm this intuition. 
\begin{Theorem}\label{thm:pro_monot}
For any $n\ge m\gg1$ and $p\in \Delta_\cX$, 
\[
{H^\mathcal{S}_n(p)}\ge H^\mathcal{S}_m(p). 
\]
\end{Theorem}
Besides the above result that lowerly bounds ${H^\mathcal{S}_n(p)}$ with 
$H^\mathcal{S}_m(p)$ for $m\le n$,
a more desirable result is to upperly bound ${H^\mathcal{S}_n(p)}$ 
with a function of ${H^\mathcal{S}_m(p)}$.
 
Such a result will enable us to draw a sample of size $m\le n$, 
obtain an estimate of ${H^\mathcal{S}_m(p)}$ from $\mathcal D_{m}$, 
and use it to bound the value of ${H^\mathcal{S}_n(p)}$ and thus of $\mathcal D_{n}$ 
for a much larger sample size $n$. 
With such an estimate, 
we can perform numerous tasks such as \emph{predicting}
the performance of algorithms in Section~\ref{sec:ada_prop_est} 
and~\ref{sec:comp_est} when more observations are available,
and the space needed for storing a longer sample profile. 
The next theorem provides a simple and tight bound 
on ${H^\mathcal{S}_n(p)}$ in terms of ${H^\mathcal{S}_m(p)}$. 
\begin{Theorem}
For any $n\ge m\gg1$ and $p\in \Delta_\cX$, 
\[
{H^\mathcal{S}_n(p)}\le
\sqrt{\frac{n\log n}{m\log m}}\cdot {H^\mathcal{S}_m(p)}. 
\]
\end{Theorem}
The aforementioned application of this result is closely related to the 
recent works on \emph{learnability estimation}~\cite{kong2018estimating, kong2019sublinear}.

\subsubsection*{Lipschitzness}
Viewing ${H^\mathcal{S}_n(p)}$ as a distribution property, we establish its Lipschitzness 
with respect to a weighted Hamming distance and the $\ell_1$ distance. 
Given two distributions $p, q\in \Delta_\cX$, 
the vanilla \emph{Hamming distance} is denoted by 
\[
h(p,q) := \sum_{x\in \cX} \indic_{p_x\not= q_x}. 
\]
The distance is suitable for being a statistical distance since there may 
be many symbols at which the two distributions differ, yet those symbols
account for only a negligible total probability and has little effects on 
many induced statistics.
To address this, we propose a \emph{weighted Hamming distance} 
\[
h_{_{\mathcal W}}\!(p,q) 
:= 
\sum_{x\in \cX} \max\{p_x, q_x\}\boldsymbol\cdot\indic_{p_x\not= q_x}. 
\]
The next result measures the Lipschitzness of 
$H^\mathcal{S}_n$ under~$h_{_{\mathcal W}}$.
\begin{Theorem}
For any integer $n$, and distributions $p$ and $q$,
if $h_{_{\mathcal W}}\!(p,q)\le \ve$ for some $\ve\ge 1/n$,
\[
\Abs{{H^\mathcal{S}_n(p)}-{H^\mathcal{S}_n(q)}}\le\mathcal{O}(\sqrt{\ve\! n}). 
\]
\end{Theorem}
Replacing $\max\{p_x, q_x\}$ with $|p_x-q_x|$ results in a 
common similarity measure -- the $\ell_1$ distance. The next theorem is an analog to the above under this classical distance. 
\begin{Theorem}
For any integer $n$, and distributions $p$ and $q$,
if $\ell_1(p,q)\le \ve$ for some $\ve\ge 0$, 
\[
\Abs{{H^\mathcal{S}_n(p)}-c{H^\mathcal{S}_n(q)}}\le\mathcal{O}((\ve\! n)^{2/3}),
\]
where $c$ is a constant in $[1/3, 3]$. 
Note that the inequality is significant iff $\varepsilon\le \tilde\Theta(1/n^{1/4})$, 
since the value of ${H^\mathcal{S}_n(p)}$ is at most $\mathcal{O}(\sqrt{n\log n})$ for all $p$.
\end{Theorem}

\subsection{Profile Entropy for Structured Families}\label{sec:pro_struct}
Following the study of attributes of profile entropy, 
we derive nearly tight bounds for the ${H^\mathcal{S}_n(p)}$ 
values of three important structured families, 
 log-concave, power-law, and histogram.  
These bounds tighten up and significantly improve those
 in~\cite{hao2019doubly}, and show the ability of profile 
entropy in charactering natural shape constraints.

Below, we follow the convention of specifying 
the structured distributions over $\cX=\mathbb Z$. 
\subsubsection*{Log-Concave Distributions}
We say that $p\in\Delta_{\mathbb{Z}}$ 
is \emph{log-concave} if $p$ has a contiguous support 
and $p_x^2\ge p_{x-1} p_{x+1}$ for all $x\in \mathbb Z$.

The log-concave family encompasses a broad range of discrete distributions,
such as Poisson, hyper-Poisson, Poisson binomial, 
binomial, negative binomial, geometric, and hyper-geometric,
with wide applications 
to numerous research areas, including 
statistics~\cite{saumard2014log}, computer science~\cite{lovasz2007geometry}, 
economics~\cite{an1997log}, algebra, and geometry~\cite{stanley1989log}.

The next result upperly bounds the profile entropy of log-concave families, 
and is \emph{tight} up to logarithmic factors of~$n$.

\begin{Theorem}\label{thm:logconc_bound}
For any $n\in\mathbb Z^+$ and distribution $p\in \Delta_{\mathbb Z}$, 
if $p$ is log-concave and has a variance of~$\sigma^2$, 
\[
{H^\mathcal{S}_n(p)}\le \mathcal{O}(\log n)\Paren{1+\min\Brace{\sigma, \frac{n}{\sigma}}}. 
\]
\end{Theorem}
This upper bound is uniformly better than the 
$\min\{\sigma, (n^2/\sigma)^{1/3}\}$ bound in~\cite{hao2019doubly}. 

A similar bound holds for $t$-mixtures of log-concave distributions. 
More concretely, 
\begin{Theorem}\label{thm:mix_logcon}
For any integer $n$ and distribution $p\in \Delta_{\mathbb Z}$, 
if $p$ is a $t$-mixture of log-concave distributions each has a 
variance of $\sigma^2_i$, where $i=1,\ldots, t$, 
\[
{H^\mathcal{S}_n(p)}
\le \mathcal{O}(\log n)
\Paren{1+\min\Brace{\sum_i\sigma_i, \max_i\Brace{\frac{n}{\sigma_i}}}}. 
\]
\end{Theorem}

The introduction about log-concave families 
covers numerous classical discrete distributions, 
yet leaves many more continuous ones untouched~\cite{bagnoli2005log}. 
Below, we present a discretization procedure that 
preserves distribution shapes such as monotonicity, 
modality, and log-concavity. 
Applying this procedure to the Gaussian distribution $\mathcal{N}(\mu,\sigma^2)$ 
further shows the optimality of Theorem~\ref{thm:logconc_bound}. 
 
\subsubsection*{Discretization}
Let $X$ be a continuous random variable over $\mathbb R$ with density function $f(x)$.
For any $x\in \mathbb R$, denote by $\D{x}$ the closest integer $z$ such that $x\in(z-1/2, z+1/2]$.
The $\D{X}$ has a distribution over $\mathbb Z$:
\[
p(z):=\int_{z-\frac12}^{z+\frac12} f(x) dx,\ \forall z\in\mathbb Z.
\]
We refer to $\D{X}$ as the \emph{discretized version} of $X$. 

\paragraph{Shape  preservation}
By definition, one can readily verify that the above transformation preserves 
several important shape characteristics of distributions,
such as monotonicity, modality, and $k$-modality 
(possibly yields a smaller $k$). 
The following theorem covers log-concavity. 
\begin{Theorem}\label{thm:logconc_pre}
For any continuous random variable $X$ over $\mathbb R$ with
a log-concave density $f$, the distribution $p\in \Delta_{\mathbb Z}$ 
associated with $\D{X}$ is also log-concave. 
\end{Theorem}

\paragraph{Moment  preservation} 
Denote by $p$ the distribution of $\D{X}$ for $X\sim f$. 
Let $\mu$ and $\sigma^2$ be the mean and variance of 
density $f$, given that they exist.  
The theorem below shows that the discrete distribution $p$ has,
within small additive absolute constants, a mean of $\mu$ 
and variance of $\Theta(\sigma^2)$. 

\begin{Theorem}\label{thm:moment_pre}
Under the aforementioned conditions, 
the mean of $\D{X}$ satisfies
\[
\EE\D{X} = \mu\pm \frac12, 
\]
and the variance of $\D{X}$ satisfies 
\[
\sigma^2/2-1\le \EE(\D{X}-\EE \D{X})^2 \le 2\sigma^2+1. 
\]
\end{Theorem}

\subsubsection*{Optimality of Theorem~\ref{thm:logconc_bound}}
By the above formula, the discretized Gaussian $\D{\mathcal{N}}\!(\mu,\sigma^2)$ has a distribution in the form of 
\[
p_{\!_G}(z)\!:=\frac{1}{\sqrt{2\pi}\sigma} \!\int_{z-\frac12}^{z+\frac12}\!\exp\Paren{-\frac{(x-\mu)^2}{2\sigma^2}} dx,\ \forall z\in\mathbb Z.
\]
Consolidating Theorem~\ref{thm:logconc_pre} and~\ref{thm:moment_pre} shows that $p_{\!_G}\!$ 
is a log-concave distribution with a variance of $\Theta(\sigma^2)\pm 1$.
Consequently, Theorem~\ref{thm:logconc_bound} yields the following upper bound:
\[
H^\mathcal{S}_n(p_{\!_G})\le \mathcal{O}(\log n)\Paren{1+\min\Brace{\sigma, \frac{n}{\sigma}}}. 
\]
On the other hand, Section~\ref{sup:2.5} of the appendices shows 
\begin{Theorem}
Under the aforementioned conditions,
\[
H^\mathcal{S}_n(p_{\!_G})\ge \mathcal{O}(\log n)^{-1}\!\Paren{1+\min\Brace{\sigma, \frac{n}{\sigma}}}. 
\]
\end{Theorem}
The optimality of Theorem~\ref{thm:logconc_bound} follows by these inequalities. 

\subsubsection*{Power-Law Distributions}
We say that a discrete distribution $p\in\Delta_{\mathbb{Z}}$ 
is a \emph{power-law with power $\alpha>0$} 
if $p$ has a support 
of $[k]:=\{1,\ldots, k\}$ for some $k\in \mathbb Z^+\cup \{\infty\}$   
and $p_x\propto x^{-\alpha}$ for all $x\in [k]$. 

Power-law is a ubiquitous structure appearing in many situations of scientific interest,
ranging from natural phenomena such as the initial mass function of stars~\cite{kroupa2001variation},
species and genera~\cite{humphries2010environmental}, rainfall~\cite{machado1993structural}, 
population dynamics~\cite{taylor1961aggregation},
and brain surface electric potential~\cite{miller2009power}, 
to man-made circumstances such as 
the word frequencies in a text~\cite{baayen2002word}, income rankings~\cite{druagulescu2001exponential}, company sizes~\cite{axtell2001zipf}, 
and internet topology~\cite{faloutsos1999power}. 

Unlike log-concave distributions that concentrate around their mean values, power-laws are known to possess 
``long-tails'' and always log-convex. Hence, one may expect the profile entropy of power-law 
distributions to behave differently from that of log-concave ones. 
The next theorem justifies this intuition and provides tight upper bounds.

\begin{Theorem}
For a power-law distribution $p\in\Delta_{[k]}$ with power $\alpha$, 
we have 
\[
{H^\mathcal{S}_n(p)}\le 7\log n+e^2\cdot \min\{k, \mathcal U_{n}^k(\alpha)\}, 
\]
where 
\[
\mathcal U_{n}^k(\alpha)
\!:= \!
\begin{cases}
n^{\frac1{1+\alpha}}&\!\!\!\! \text{if } \alpha\ge 1\!+\!\frac{1}{\log k};\\
\Paren{\frac{n}{\log n}}^{\frac1{1+\alpha}} &\!\!\!\! \text{if } 1\le \alpha <1\!+\!\frac{1}{\log k};\\ 
\sqrt n\Paren{\frac{k}{\sqrt n}\wedge \Paren{\frac{\sqrt n}{k}}^{\!\frac{1-\alpha}{1+\alpha}}} &\!\!\!\! \text{if } 0\le\alpha<1.
\end{cases}
\]
\end{Theorem}
The above upper bound fully characterizes the profile entropy of power-laws 
and surpasses the basic $\{k, \sqrt {n\log n}\}$ bound for both 
$k\gg\!\sqrt n$ and $k\ll\!\sqrt n$. 
In comparison, \cite{hao2019doubly} yields a $\mathcal{O}(n^{\min\{1/(1+\alpha), 1/2\}})$ 
upper bound, 
which improves over $\sqrt {n\log n}$ for only $\alpha>1$
and is worse than that above for all $\alpha<1\!+\!{1}/{\log k}$.

\subsubsection*{Histogram Distributions}
A distribution $p\in \Delta_\cX$ is a \emph{$t$-histogram distribution} 
if there is a partition of $\cX$ into $t$ parts such that $p$ has 
the same probability value over all symbols in each part.   

Besides the long line of research on histograms reviewed in~\cite{ioannidis2003history}, 
the importance of histogram distributions rises 
with the rapid growth of data sizes in numerous 
engineering and science 
applications in the modern era. 

For example, 
in scenarios where processing the complete data set is inefficient or even impossible, 
a standard solution is to partition/cluster the data into groups 
according to the task specifications and element similarities, 
and randomly sample from each group to obtain a subset of the data to use. 
This naturally induces a histogram distribution,
with each data point being a symbol in the support. 

The work of~\cite{hao2019doubly} studies the class of $t$-histogram distributions 
and obtains the following upper bound 
\[
{H^\mathcal{S}_n(p)}\le \tilde{\mathcal{O}}\Paren{\min \Brace{(nt^2)^{\frac13}, \sqrt{n}}}.
\]
Our contribution is establishing its optimality. 
\begin{Theorem}\label{thm:histgram}
For any $t, n\in \mathbb Z^+$, there exists 
a $t$-histogram distribution $p$ such that 
\[
{H^\mathcal{S}_n(p)}\ge \tilde{\Omega}\Paren{\min \Brace{(nt^2)^{\frac13}, \sqrt{n}}}. 
\]
\end{Theorem}
Note that uniform distributions correspond to $1$-histograms, 
for which the bounds reduce to $\tilde{\Theta}(n^{1/3})$.

\section{Extension and Conclusion}
\subsection{Multi-Dimensional Profiles}
The notion of profile generalizes 
to the multi-sequence setting. 

Let $\cX$ be a finite or countably infinite alphabet. 
For every $\vec n:=(n_1,\ldots, n_d)\in \mathbb N^d$ 
and tuple $ x^{\vec n}:=(x^{n_1}_1,\ldots, x^{n_d}_d)$ of sequences in $\cX^*$, 
the \emph{multiplicity} 
$\mu_y( x^{\vec n})$ of a symbol $y\in\cX$ 
is the vector of its frequencies in the tuple of sequences. 
The \emph{profile} of $x^{\vec n}$ is the multiset  
$\va(x^{\vec n})$ of multiplicities of the observed 
symbols~\cite{acharya2010classification, das2012competitive, charikar2019efficient}, 
and its \emph{dimension} is the number $\mathcal D(x^{\vec n})$ 
of distinct elements in the multiset. 
Drawing independent samples from $\vec p:=(p_1,\ldots, p_d)\in \Delta_\cX^d$, 
the \emph{profile entropy} is simply the entropy of the joint-sample profile.

Many of the previous results potentially generalize to this multi-dimensional setting. 
For example, the $\sqrt {2n}$ bound on 
$\mathcal D(x^{\vec n})$ in the 1-dimensional case becomes
\begin{Theorem}
For any $\cX$, $\vec n$, and $x^{\vec n}\in \cX^{\vec n}$, 
there exists a positive integer $r$ such that
\begin{align*}
\sum_i n_i
\ge
d\cdot \binom{d+r-1}{d+1},
\end{align*}
and
\begin{align*}
\mathcal D
\le 
\binom{d+r}{d}-1.
\end{align*} 
\end{Theorem}
This essentially recovers the $\sqrt{2n}$ bound for $d=1$. 

\subsection{Concluding Remarks}
The classical view on the entropy of an \iid sample 
corresponds to the equation 
\[
H(X^n\!\sim\! p)=nH(p),
\]
which provides little insight for statistical applications. 

This paper presents a different view by decomposing 
the $nH(p)$ information into three pieces: the labeling of the profile elements, 
ordering of them, and profile entropy.
With no bias towards any symbols and under the \iid assumption,
the profile entropy rises as 
a fundamental measure unifying the concepts of estimation, inference, and compression.

\vfill
\pagebreak

\appendix 
\paragraph{Appendices orgnization}
In the appendices, 
we order the results and proofs according to their logical priority. 
In other words, the proof of a theorem or lemma mainly relies on preceding results. 
For the ease of reference, the numbering of the theorems is 
consistent with that in the main paper.

\section{Dimension and Entropy of Sample Profiles}\label{sec:dim_and_entro}

Consider an arbitrary sequence $x^n$ over 
a finite or countably infinite alphabet $\cX$. 
The \emph{multiplicity} $\mu_y( x^{n})$ of a symbol $y\in\cX$ 
is the frequency of $y$ in $x^n$.
The \emph{prevalence} of an integer $\mu$ 
is the number $\va_\mu(x^n)$ of symbols in $x^n$ with multiplicity $\mu$. 
The \emph{profile} of $x^n$ is the multiset $\va(x^n)$ 
of multiplicities of the symbols in $x^n$, 
which we describe as a profile of \emph{length} $n$.

Let $\Delta_\cX$ be a collection of distributions over $\cX$. 
We say that a distribution collection $\cP\subseteq \Delta_\cX$ 
is \emph{label-invariant} if for any $p\in \cP$, 
the collection $\cP$ contains all its symbol-permuted versions.
A \emph{distribution property} over a distribution collection 
$\cP\subseteq \Delta_\cX$ is a functional $f: \cP \to \mathbb R$
that associates with each distribution in $\cP$ a real value. 
For a label-invariant $\cP\subseteq \Delta_\cX$, 
we say that a property $f$ over $\cP$ is symmetric if 
it takes the same value for distributions sharing 
the same probability multiset.  

\subsection{Profile Dimension and Its Concentration}
A profile $\boldsymbol\phi$ is said to have length $n$ if there exists
$x^n\in \cX^n$ satisfying $\bphi = \va(x^n)$. 
For any multiset $S$, we define its \emph{dimension} as 
the number $\mathcal D(S)$ of distinct elements in $S$. 
Recall that profiles of sequences are multisets. For notational convenience, we write both $\mathcal D(\va(x^n))$ and $\mathcal D(x^n)$ for the dimension of profile $\va(x^n)$.

Viewed as a random variable, 
the profile of an \iid sample $X^n\sim p$ 
has a distribution depending only on $p$ and $n$. 
Hence, we denote by $\Phi^n$ such a profile 
and write $\Phi^n\sim p$. 
Analogously, we denote by $\mathcal D_n:=\mathcal D(\Phi^n)$ 
the \emph{profile dimension} associated with $(p, n)$, 
and write $\mathcal D_n\sim p$.

Recall that the multiplicity $\mu_y(x^n)$ denotes the 
number of times symbol $y$ appearing in $x^n$. Denote by $\bigvee$ the 
logical OR operator. For any distribution $p$ and $X^n\sim p$, 
we have 
\[
\mathcal D_n
=
\sum_{\mu=1}^n\  \bigvee_{x\in \cX} \indic_{\mu_x(X^n) = \mu}.
\]
The statistical dependency landscape of terms in the summation is rather complex, 
since $\mu_x(X^n)$ and $\mu_y(X^n)$ are dependent  for every $(x,y)$ pair due to the fixed sample size; and so are $\indic_{\mu_x(X^n) = \mu_1}$ and $\indic_{\mu_x(X^n) = \mu_2}$ for every pair of distinct $\mu_1$ and $\mu_2$. 
To simplify the derivations, we relate this quantity to its variant under the aforementioned Poisson sampling scheme, i.e., 
making the sample size an independent $N\sim \Poi(n)$. Specifically, define 
\[
\tilde{\mathcal D}_N
:=
\tilde{\mathcal D} (X^N) 
:=
\sum_{U=1}^n\  \bigvee_{x\in \cX} \indic_{\mu_x(X^N) = U}. 
\]
Note that this is not the same as $\mathcal D_N$ 
since the summation index goes up only to $n$. 
Denote the expected value of $\tilde{\mathcal D}_N$  by $E_n(p)$. 
Our result shows that the original $\mathcal D_n$ satisfies a 
Chernoff-Hoeffding type bound centered at $E_n(p)$. 

\setcounter{Theorem}{5}
\begin{Theorem}\label{thm:exp_conc}
Under the above conditions and for any $n\in \mathbb Z^+$, $p\in \Delta_\cX$,
and $\gamma>0$, 
\[
\Pr\Paren{\frac{\mathcal D_n}{1+\gamma} \ge   E_n(p)}
\le 3\sqrt n e^{-\min\{\gamma^2, \gamma\} E_n(p)/3}, 
\]
and for any $\gamma\in (0,1)$, 
\[
\Pr\Paren{\frac{\mathcal D_n}{1-\gamma} \le E_n(p)}
\le 3 \sqrt n e^{-\gamma^2 E_n(p)/2}. 
\]
\end{Theorem}

\begin{proof}
To simplify our analysis, we first consider an alternative model where the sample size is an independent Poisson random variable $N$ with mean $n$. 
A nice attribute of Poisson sampling is that all the multiplicities 
$\mu_y(X^n)$ are independent of each other. 
Later, we will relate this model to the fixed-sample-size 
model and establish our claim rigorously. 

For simplicity and clarity, 
we suppress $X^n$ in $\mu_y(X^n)$ and write $\nu_y$ instead of $\mu_y$ 
when the multiplicity is obtained through Poisson sampling.
For any $i\in [n]$, denote $G_i(\{\nu_x\}_x):=\bigvee_{x\in \cX} \indic_{\nu_x = i}$. 
Instead of analyzing $\mathcal D_N$, we consider
\[
\tilde {\mathcal D}_N:=
\sum_{i=1}^n \bigvee_{x\in \cX} \indic_{\nu_x = i}
=\sum_{i=1}^n G_i(\{\nu_x\}_x). 
\]
Note that for any disjoint $I,J\subseteq[n]$, 
the functions $\sum_{i\in I}G_i(\{\nu_x\}_x)$ and $\sum_{j\in J}G_j(\{\nu_x\}_x)$ are 
discordant monotone by each argument, namely, 
when we increase the value of each $\nu_x$,
the increase in the value of one function implies the non-increase of the other. Then, by the results in~\cite{lehmann1966some}, the values of the two functions, when viewed as random variables, are negatively associated.  

Next we show that quantity $\tilde {\mathcal D}_N$ 
satisfies a Chernoff-type bound. 

Let $\gamma$ be an arbitrary positive number. Note that $G_i$ is a Bernoulli random variable with parameter
\[
q_i := \E\Br{G_i(\{\nu_x\}_x)}.
\] 
Then for the expected value of $\tilde {\mathcal D}_N$, we have
\[
E_n(p) := \E\Br{\tilde {\mathcal D}_N} 
=\E\Br{\sum_{i=1}^n G_i(\{\nu_x\}_x)}
=\sum_i q_i. 
\]
For simplicity, write $Y := \tilde {\mathcal D}_N$ and $\mu:=E_n(p)$. By Markov's inequality and the monotonicity of function $e^{t y}$ over $t>0$,
\[
\Pr\Paren{Y\ge (1+\gamma)\mu}
=
\Pr\Paren{e^{tY}\ge e^{t(1+\gamma)\mu}}
\le 
\frac{\E[e^{tY}]}{e^{t(1+\gamma)\mu}}.
\]
It suffices to bound $\E[e^{tY}]$ by a function of other parameters. 
\begin{align*}
\E[e^{tY}]
&\overset{(a)}=
\E\Br{\exp\Paren{t\Paren{\sum_{i=1}^n G_i(\{M_x\}_x)}}}\\
&\overset{(b)}=
\E\Br{\exp\Paren{t G_1(\{M_x\}_x)}\cdot \exp\Paren{t\Paren{\sum_{i=2}^n G_i(\{M_x\}_x)}}}\\
&\overset{(c)}\le 
\E\Br{\exp\Paren{t G_1(\{M_x\}_x)}}\cdot \E\Br{\exp\Paren{t\Paren{\sum_{i=2}^n G_i(\{M_x\}_x)}}}\\
&\overset{(d)}\le 
\prod_{i=1}^n\E\Br{\exp\Paren{t G_i(\{M_x\}_x)}}
\overset{(e)}= 
\prod_{i=1}^n \Paren{1+q_i (e^t-1)}\\
&\overset{(f)}\le
\prod_{i=1}^n \Paren{\exp\Paren{q_i (e^t-1)}}
\overset{(g)}=
\exp\Paren{\sum_{i=1}^nq_i (e^t-1)}\\
&\overset{(h)}=
\exp\Paren{(e^t-1)\mu},
\end{align*}
where $(a)$ follows by the definition of $Y$;
$(b)$ follows by $e^{a+b}=e^a\cdot e^b$;
$(c)$ follows by the fact that $G_1$ is negatively associated with $\sum_{i=2}^n G_i$;
$(d)$ follows by an induction argument via negative association;
$(e)$ follows by the fact that $G_i$ is a Bernoulli random variable with mean $q_i$;
$(f)$ follows by the inequality $1+x\le e^x,\forall x\ge 0$;
$(g)$ follows by $e^a\cdot e^b=e^{a+b}$;
and $(h)$ follows by $\mu=\sum_i q_i$.

Applying standard simplifications, we obtain 
\[
\Pr\Paren{ Y\ge (1+\gamma)\mu}
\le
e^{-\min\{\gamma^2, \gamma\} \mu/3},\ \forall \gamma> 0,
\]
and
\[
\Pr\Paren{ Y\le (1-\gamma)\mu}
\le
e^{-\gamma^2\! \mu/2},\ \forall \gamma\in (0,1).
\]
The proof will be complete upon noting that:
1) the probability that $N=n$ is at least $1/(3\sqrt n)$;\linebreak
2) conditioning on $N=n$ transforms 
the sampling model to that with a fixed sample size $n$.
\end{proof}
As a corollary, the value of $\mathcal D_n$ is often close to $E_n(p)$. 
\begin{Corollary}
Under the same conditions as above and for any $n\in \mathbb Z^+$, $p\in \Delta_\cX$,
with probability at least $1-6/\sqrt n$, 
\[
\frac12 E_n(p)-4\log n \le \mathcal D_n \le 2 E_n(p)+3\log n.  
\] 
\end{Corollary}
\begin{proof}
To establish the lower bound, note that if $E_n(p)\ge 3\log n$,
setting $\gamma = 1$  in Theorem~\ref{thm:exp_conc} yields
\[
\Pr\Paren{ \mathcal D_n\ge 2E_n(p)+3\log n}
\le
\Pr\Paren{ \mathcal D_n\ge 2E_n(p)}
\le
3\sqrt ne^{-E_n(p)/3}
\le
\frac{3}{\sqrt n},
\]
else if $E_n(p)< 3\log n$, setting $\gamma=(3\log n)/E_n(p)$ yields
\[
\Pr\Paren{ \mathcal D_n\ge 2E_n(p)+3\log n}
\le
\Pr\Paren{ \mathcal D_n\ge E_n(p)+3\log n}
\le
3\sqrt n e^{-(3\log n)/3}
=
\frac{3}{\sqrt n}.
\]
As for the upper bound, if $E_n(p)\ge 8\log n$,
\[
\Pr\Paren{  \mathcal D_n+4\log n\le \Paren{1-\frac{1}{2}}E_n(p)}
\le 
\Pr\Paren{  \mathcal D_n \le \Paren{1-\frac{1}{2}}E_n(p)}
\le
3\sqrt n e^{-\mu/8}
\le
\frac{3}{\sqrt n},
\]
and for any $E_n(p)< 8\log n$, 
\[
\Pr\Paren{ \mathcal D_n+4\log n\le \Paren{1-\frac{1}{2}}E_n(p)}
\le 
\Pr\Paren{ \mathcal D_n< 0}
=0
\le
\frac{3}{\sqrt n}.
\]
Combining these tail bounds through the union bound completes the proof.
\end{proof}
In addition to the above, we establish an Efron-Stein type inequality.
\begin{Theorem}\label{thm:var_conc}
For any distribution $p$ and $\mathcal D_n\sim p$, 
\[
\Var(\mathcal D_n) \le \EE[\mathcal D_n]. 
\]
\end{Theorem}
\begin{proof}

First, note that for any $j,t\in [n]$ and $j\not=t$,
\begin{align*}
C_{j, t}
&:=
\text{Cov}\Paren{\indic_{\varphi_j(X^n)>0}, \indic_{\varphi_t(X^n)>0}}\\
&=
\Pr\Paren{\varphi_j(X^n), \varphi_t(X^n)>0}
-\Pr\Paren{\varphi_j(X^n)>0}\cdot \Pr\Paren{\varphi_t(X^n)>0}\\
&=
\Paren{\Pr\Paren{\varphi_j(X^n)>0 \vert \varphi_t(X^n)>0}
-\Pr\Paren{\varphi_j(X^n)>0}}\cdot \Pr\Paren{\varphi_t(X^n)>0}\\
&=
\Paren{\Pr\Paren{\varphi_j(X^n)>0 \vert \varphi_t(X^n)>0}
-\Pr\Paren{\varphi_j(X^n)>0 \vert \varphi_t(X^n)=0}}\\
&\quad \cdot \Pr\Paren{\varphi_t(X^n)=0}\cdot \Pr\Paren{\varphi_t(X^n)>0}\\
&\le
0
\end{align*}

Therefore, the variance of the profile dimension $\mathcal D_n$ satisfies
\begin{align*}
\Var\Paren{\mathcal D_n}
&=
\Var\Paren{\sum_{i=1}^n \indic_{\varphi_i(X^n)>0}}\\
&\le
\sum_{i=1} \Var\Paren{\indic_{\varphi_i(X^n)>0}}
+\sum_{j\not=t} \text{Cov}\Paren{\indic_{\varphi_j(X^n)>0}, \indic_{\varphi_t(X^n)>0}}\\
&\le
\sum_{i=1}\E\Br{\indic_{\varphi_i(X^n)>0}}
+\sum_{j\not=t} C_{j, t}\\
&\le 
\sum_{i=1}\E\Br{\indic_{\varphi_i(X^n)>0}}\\
&=
\E\Br{\mathcal D_n}. \qedhere
\end{align*} 
\end{proof}

\subsection{Profile Entropy and Its Connection to Dimension} 

For a distribution $p\in \Delta_\cX$ and sampling parameter $n$,
the \emph{profile entropy} with respect to 
$(p, n)$ is the entropy $H(\Phi^n)$ of the sample profile $\Phi^n\sim p$.
By Shannon's source coding theorem, 
profile entropy $H(\Phi^n)$ is the \emph{information-theoretic
limit} of the minimal expected codeword length (MECL) for the lossless compression of the sample profile. Hence, characterizing its value is of fundamental importance.
But as one may expect, the distribution of $\Phi^n$ is sophisticated and over a 
large alphabet. 

More concretely, by the
formula of~\cite{hardy1918asymptotic},
the number $\mathbb{P}(n)$ of integer partitions of $n$, 
which happens to equal to the number
of length-$n$ profiles, satisfies the equation
\[
\log \mathbb{P}(n)= 2\pi\sqrt{\frac{n}{6}}(1+o(1)).
\]
Despite the complex statistical dependency landscape
and the exponentially large alphabet size,
below we establish that for any distribution and sample size, 
the profile entropy is often of the same order as the profile size, with high probability. 
Specifically, 
\setcounter{Theorem}{0}
\begin{Theorem}\label{thm:entro_eql_dim}
For any distribution $p\in \Delta_\cX$ and $\Phi^n\sim p$,
with probability at least $1-\mathcal{O}(1/\sqrt n)$, \vspace{-0.35em}
\[
\C{H\!(\Phi^n)} = \tilde{\Theta}(\mathcal D(\Phi^n)), \vspace{-0.35em}
\]
where the notation $\tilde{\Theta}(\boldsymbol \cdot)$ hides logarithmic factors of $n$.
\vspace{-0.5em}
\end{Theorem}

We decompose the proof of the theorem into three steps. 
First, we show that $\C{H\!(\Phi^n)} \le \tilde{\Theta}(\mathcal D(\Phi^n))$ with high probability, 
which is a simple consequence of Shannon's source coding theorem and 
Theorem~\ref{thm:exp_conc} (which shows that $\mathcal D(\Phi^n)$ highly 
concentrates around its expectation). 
Then, we introduce a simple quantity ${H^{\mathcal S}_n(p)}$  
 that approximates the expectation of 
 $\mathcal D(\Phi^n)$ to within logarithmic factors of $n$.
Finally, leveraging this approximation guarantee, we establish the other direction of the theorem. This step is more involved due to the aforementioned complications. 

\subsubsection*{A. Bounding Profile Entropy by Its Dimension}
By the tail bounds (Theorem~\ref{thm:exp_conc}) 
and trivial lower bound of $1$ on the profile dimension,  with probability at least $1-\mathcal{O}(1/\sqrt n)$,
the expectation of $\mathcal D(\Phi^n)$ satisfies 
\[
\EE[\mathcal D(\Phi^n)]\le \tilde{\mathcal{O}}(\mathcal D(\Phi^n)).
\]
By the block profile compression algorithm presented in Section~\ref{sec:pro_compression} of the main paper, 
storing profile $\Phi^n\sim p$ losslessly takes \vspace{-0.75em}
\[
\mathcal{O}(\log n)\cdot\EE[\mathcal D(\Phi^n)]
+\mathcal{O}\Paren{\frac1{\sqrt n}}\cdot \log \mathbb{P}(n)
\le
\mathcal{O}(\log n)\cdot\EE[\mathcal D(\Phi^n)]
\]
nats space in expectation. By Shannon's source coding theorem, 
the expected space to losslessly storing a random variable 
is at least its entropy. Hence, with probability at least $1-\mathcal{O}(1/\sqrt n)$,
\[
H(\Phi^n)\le \mathcal{O}(\log n)\cdot \EE[\mathcal D(\Phi^n)]
\le \tilde{\mathcal{O}}(\mathcal D(\Phi^n)).
\]
Again, noting that $\mathcal D(\Phi^n)\ge 1$ completes the proof. 

\subsubsection*{B. Simple Approximation Formula for Profile Dimension}
It remains to show that $\C{H\!(\Phi^n)} \ge \tilde{\Omega}(\mathcal D(\Phi^n))$,
with high probability. 
To proceed further, we note that $\mathcal D(\Phi^n)=\mathcal D_n \sim p$ 
is often close to $E_n(p)$, the expectation of its Poissonized version $\tilde{\mathcal D}_N$, with an exponentially small deviation probability. 
Hence, to approximate $\mathcal D_n$, 
it suffices to accurately compute $E_n(p)$. 
By independence and the linearity of expectations,
\[
E_n(p)
=
\EE[\tilde{\mathcal D}_N] 
= 
\sum_{i=1}^n \Paren{1-\prod_{x\in \cX} \Paren{1- e^{-np_x} \frac{(np_x)^i}{i!}}}.
\]
The expression is exact but does not relate to $p$ in a simple manner. 
For an intuitive approximation, we 
partition the unit interval into a sequence of ranges,
\[
I_j:=\left((j-1)^2\frac{\log n}{n},\ j^2 \frac{\log n}{n}\right], 1\le j\le \sqrt{\frac{n}{\log n}},
\]
denote by $p_{I_j}$ the number of probabilities $p_x$ belonging to $I_j$, 
and relate $E_n(p)$ to an induced shape-reflecting quantity, \vspace{.25em}
\[
H^\mathcal{S}_n(p):= \sum_{j\ge 1} \min\Brace{p_{I_j}, j\cdot \log n}, 
\vspace{0.5em}
\]
the sum of the effective number of probabilities lying within each range~\cite{hao2019doubly}. 
To compute ${H^\mathcal{S}_n(p)}$, we simply count the number of probabilities in each $I_j$. 
Our main result shows that $H^\mathcal{S}_n(p)$ 
well approximates $E_n(p)$ over the entire $\Delta_\cX$, 
up to logarithmic factors of $n$. 
\setcounter{Theorem}{7}
\begin{Theorem}\label{thm:hs_app_en}
For any $n\in \mathbb Z^+$ and $p\in \Delta_\cX$, 
\[
\frac{1}{\sqrt{\log n}} \cdot \Omega(H^\mathcal{S}_n(p))
\le 
E_n(p)
\le
\mathcal{O}(H^\mathcal{S}_n(p)). 
\]
\end{Theorem}
\begin{proof}
The fact that $\mathcal{O}({H^{\mathcal S}_n(p)})$ upperly bounds 
$\EE[\tilde{\mathcal D}_N]$ simply follows by the concentration of 
Poisson variables, and is established in~\cite{hao2019doubly}. 
Below we show that the quantity also serves as a lower bound. 
By construction, for any given sampling parameter $n$, index $j$, 
and symbol $x$ with probability $p_x\in I_j$, the corresponding 
symbol multiplicity $\mu_x\sim \Poi(np_x)$. \vspace{-0.15em} 
Hence, we can express the expectation of 
$\tilde{\mathcal D}_N$ as
\begin{align*}
\EE\Br{\tilde{\mathcal D}_N}
&=\EE\Br{\sum_{i=1}^n \bigvee_x \indic_{\mu_x=i}}\\
&=\sum_{i=1}^n \EE\Br{1-\bigwedge_x \indic_{\mu_x\not=i}}\\
&=\sum_{i=1}^n \Paren{1-\EE\Br{\prod_x \indic_{\mu_x\not=i}}}\\
&=\sum_{i=1}^n \Paren{1-\prod_x \EE\Br{\indic_{\mu_x\not=i}}}\\
&=\sum_{i=1}^n \Paren{1-\prod_x \Paren{1- e^{-np_x} \frac{(np_x)^i}{i!}}}.
\end{align*}
This proves the aforementioned formula. Then,
for every sufficiently large index $j$ and $i\in S_j:=[(j-1)^2, j^2]\log n$, define a sequence of intervals,
\[
I_j^i:=\frac{i}{n}+\Br{-j, j}\frac{\sqrt{\log n}}{n}.
\]

Then for any $i\in S_j$ and $p_x\in I_j^i \cap I_j$, the corresponding Poisson probability satisfies
\begin{align*}
e^{-np_x} \frac{(np_x)^i}{i!}
&=
 e^{-i} \frac{i^i}{i!} \cdot \Paren{e^{i-np_x}\cdot \frac{(np_x)^i}{i^i}}\\
&=
 e^{-i} \frac{i^i}{i!} \cdot \Paren{e^{-(np_x-i)}\cdot\Paren{1+\frac{np_x-i}{i}}^i}\\
&=
 e^{-i} \frac{i^i}{i!} \cdot \exp\Paren{{-(np_x-i)}+i\cdot \log\Paren{1+\frac{np_x-i}{i}}}\\
&\ge 
\frac{1}{3\sqrt{i}}\cdot \exp\Paren{-\frac{2i}{3}\cdot \Paren{\frac{np_x-i}{i}}^2}\\ 
&\ge 
\frac{1}{9\sqrt{i}}
\ge \frac{1}{9j\sqrt{\log n}}.
\end{align*}
Now we analyze the contribution of indices $i\in S_j$ 
to the expected value of $\tilde{\mathcal D}_N$. 
For clarity, we divide our analysis into two cases: 
$p_{I_j}\ge j \log n$ and $p_{I_j}< j \log n$. 

Consider the collection $\mathcal P_j$ of probabilities $p_x\in I_j$, 
and the collection $\mathcal I_j$ of intervals $I_j^i, i\in S_j$. 
By construction, each probability in $\mathcal P_j$ is contained in 
at least $j\sqrt{\log n}$ many intervals in $\mathcal I_j$. 
Hence the total number of probabilities (repeatedly counted) included in $\mathcal I_j$ is at least 
$p_{I_j}\cdot j\sqrt{\log n}$. Note that the number of intervals in $\mathcal I_j$
is less than $2j\log n$. We claim that there exists one (or more) interval $I_j^{i'}\in\mathcal I_j$ 
containing at least $p_{I_j}/(2\sqrt{\log n})$ probabilities. 
By construction, there are at least $j\sqrt{\log n}/2$ neighboring intervals of $I_j^{i'}$ that contain 
at least $p_{I_j}/(4\sqrt{\log n})$ probabilities. 
The contribution of these these intervals to the expected
value of $\tilde{\mathcal D}_N$ is at least ${j\sqrt{\log n}}/{2}$ times
\begin{align*}
1-\Paren{1- \frac{1}{9j\sqrt{\log n}}}^{\frac{p_{I_j}}{4\sqrt{\log n}}}
&\ge 
1-\exp\Paren{\frac{p_{I_j}}{4\sqrt{\log n}}\log \Paren{1- \frac{1}{9j\sqrt{\log n}}}}\\
&\ge
1-\exp\Paren{-\frac{p_{I_j}}{40 j\log n}}\\
&\ge 
\Theta\Paren{\frac{p_{I_j}}{j\log n}},
\end{align*}
where the last step holds if $p_{I_j}\le j\log n$. This yields a lower bound of $\Theta(p_{I_j}/\sqrt{\log n})$.

It remains to consider the $p_{I_j}> j\log n$ case. 
Again, the total number of probabilities included in $\mathcal I_j$ is at least 
$p_{I_j}\cdot j\sqrt{\log n}$. Furthermore, 
each interval $I_j^i$ contains at most $p_{I_j}$ probabilities 
and there are less than $2j\log n$ intervals. 
Therefore, the number of intervals that contain at least 
$j\sqrt{\log n}/4$ probabilities is at least $j\sqrt{\log n}/2$. Otherwise, the number of probabilities included in $\mathcal I_j$ is less than
\[
\frac{j\sqrt{\log n}}{4}\cdot 2j \log n + p_{I_j} \cdot \frac{j\sqrt{\log n}}{2}
\le
p_{I_j}\cdot j\sqrt{\log n},
\]
which leads to a contradiction. Analogously, the contribution of these these intervals to the expected
value of $\tilde{\mathcal D}_N$ is at least ${j\sqrt{\log n}}/{2}$ times
\begin{align*}
1-\Paren{1- \frac{1}{9j\sqrt{\log n}}}^{\frac{j\sqrt{\log n}}{4}}
&\ge 
1-\exp\Paren{\frac{j\sqrt{\log n}}{4} \log \Paren{1- \frac{1}{9j\sqrt{\log n}}}}\\
&\ge
1-\exp\Paren{-\frac{1}{40}}\\
&= 
\Theta\Paren{1},
\end{align*}
which yields a lower bound of $\Theta(j\sqrt{\log n})$ 
on the expected value of $\tilde{\mathcal D}_N$. 

Consolidating the previous results shows that
\[
\EE\Br{\tilde{\mathcal D}_N}
\ge
\frac{1}{\sqrt{\log n}}\cdot \Omega(\sum_{j\ge 1}\min\Brace{p_{I_j}, j\cdot \log n}).\qedhere
\]
\end{proof}

\subsubsection*{C. Bounding Profile Dimension by Its Entropy }
Next, we establish that for any distribution $p\in \Delta_\cX$, $\Phi^n\sim p$, with probability at least $1-\mathcal{O}(1/\sqrt n)$,
\[
\C{H\!(\Phi^n)} \ge \tilde{\Theta}(\mathcal D(\Phi^n)). 
\]
Let $p$ be an arbitrary distribution in $\Delta_\cX$. 
Recall that we partition the interval $(0, 1]$ into a
sequence of sub-intervals,
\[
I_j:=\left((j-1)^2\frac{\log n}{n}, j^2\frac{\log n}{n}\right],\quad 1\le j\le \sqrt{\frac{n}{\log n}},
\]
and denote by $p_{I_j}$ the number of probabilities $p_x$ in $I_j$.

Our current objective is to bound $H(\Phi^n\sim p)$ from below 
by a nontrivial multiple of $H^{\mathcal S}_n(p)$. 
For simplicity of derivations, we will adopt the standard Poisson 
sampling scheme and make the sample size an independent 
Poisson variable $N\sim \Poi(n)$.
For notational simplicity, we will suppress $X^N$ in all the expressions
and write the profile as $\varphi:=\Phi^N$ by slightly abusing the notation.  

Note that the profile can be equivalently expressed as a length-$n$ vector
\[
\varphi = (\varphi _1, \ldots, \varphi _n),
\]
where $\varphi_i$ denotes the number of symbols
appearing exactly $i$ times. 

For a sufficiently large absolute constant $c$, 
decompose $\varphi$ into $c$ parts according to $I_j$ 
such that the $t$-th part ($t = 1,\ldots, c$)
consists of $\varphi_i$'s satisfying $i\in n I_j$ with $j\equiv t \mod c$. 
Since by definition, 
\[
{H^{\mathcal S}_n(p)} = \sum_{j\ge 1} \min\{p_{I_j}, j\cdot \log n\}, 
\]
one of the $c$ parts corresponds to a partial sum of at least ${H^{\mathcal S}_n(p)}/c$. 
Without loss of generality, we assume that it is the second part, i.e., \vspace{-0.5em}
\[
\sum_{j\equiv 1\!\!\!\! \mod c} \!\! \min\{p_{I_j}, j\cdot \log n\}
\ge
\frac{{H^{\mathcal S}_n(p)}}{c}. 
\]
Apply standard Poisson tail probability bounds, e.g., 
\begin{Lemma}
Let $Y$ be a Poisson or binomial random variable with mean value $\lambda$.
Then,
\[
\Pr(X\le \lambda(1-\delta))\le \exp\Paren{-\frac{\delta^2\lambda}2 \lambda},\ \ \forall \delta\in [0,1],
\]
and
\[
\Pr(X\ge \lambda(1+\delta))\le \exp\Paren{-\frac{\delta^2\lambda}{2+2\delta/3}},\ \ \forall \delta\ge 0.
\]
\end{Lemma}
For any $j\equiv 1\!\!\mod c$ and with probability at least $1-1/n^4$,
one can express the truncated profile $(\varphi_i)_{i\in n I_j}$ over $I_j$
as a function of $\mu_x$ for $x$ satisfying $np_x\in I_{j'}, j'\in (j-c/2, j+c/2)$. 

Basically, this says that for every $x$, the number of its appearance 
is not too far away from the expected value. 
By the union bound, this is true for all $j\equiv 1\!\!\mod c$ 
with probability at least $1-1/n^3$, as $j$ can take only $n$ possible values.
Denote the last event by $A$. 

To proceed, we recall the
formula of~\cite{hardy1918asymptotic} on
the number $\mathbb{P}(n)$ of integer partitions of $n$, 
which happens to equal to the number
of length-$n$ profiles: 
\[
\log \mathbb{P}(n)= 2\pi\sqrt{\frac{n}{6}}(1+o(1)).
\]
Below, we will use a weaker version that works for any $n$:
\[
\log \mathbb{P}(n)\le \sqrt{3{n}}.
\]

Then, conditioning on $A$, the truncated profiles $(\varphi_i)_{i\in n I_j}$ for 
$j\equiv 1\!\!\mod c$ are independent. Since conditioning reduces entropy,
\begin{align*}
H(\varphi)
&\ge 
H((\varphi_i)_{i\in n I_j, j \equiv 1\!\!\!\!\!\mod\! c})\\
&\ge 
H((\varphi_i)_{i\in n I_j, j \equiv 1\!\!\!\!\!\mod\! c}|\indic_{A})\\
&\ge 
H((\varphi_i)_{i\in n I_j, j \equiv 1\!\!\!\!\!\mod\! c}|\indic_{A}=1)\cdot \Pr(A)\\
&=
\sum_{j \equiv 1\!\!\!\!\!\mod\! c}H((\varphi_i)_{i\in n I_j}|\indic_{A}=1)\cdot \Pr(A)\\
&=
\sum_{j \equiv 1\!\!\!\!\!\mod\! c}H((\varphi_i)_{i\in n I_j}|\indic_{A})
-\sum_{j\equiv 1\!\!\!\!\!\mod\! c}H((\varphi_i)_{i\in n I_j}|\indic_{A}=0)\cdot (1-\Pr(A))\\
&\ge
\sum_{j \equiv 1\!\!\!\!\!\mod\! c}(H((\varphi_i)_{i\in n I_j})-H(\indic_{A}))
-\frac{1}{n^3}\sum_{j\equiv 1\!\!\!\!\!\mod\! c}H((\varphi_i)_{i\in n I_j}|\indic_{A}=0)\\
&\ge
-n H(\indic_{A})
+\sum_{j \equiv 1\!\!\!\!\!\mod\! c}H((\varphi_i)_{i\in n I_j})
-\frac{1}{n^3}\cdot n \cdot \log(\exp(\Theta(\sqrt n)))\\
&=
-\mathcal{O}\Paren{\frac{1}{\sqrt n}}+\sum_{j \equiv 1\!\!\!\!\!\mod\! c}H((\varphi_i)_{i\in n I_j}),
\end{align*}
where the third last step follows by 
\[
H(X|Y)=H(X)-I(X,Y)=H(X)-H(Y)+H(Y|X)\ge H(X)-H(Y);
\] 
the second last follows by $H(X)\le \log k$ for any $X$ with a support size of $k$, and
the fact that there are at most $\exp(3\sqrt m)$ many profiles of length $m$, as we explained above;
and the last step follows by the elementary inequality 
\[
H(\Bern(\theta))\le 2(\log 2)\sqrt{\theta(1-\theta)},\ \forall \theta\in[0,1].
\] 
Our new objective is to bound $H((\varphi_i)_{i\in n I_j})$ from below. We will find a 
sub-interval $I_j^s$ of $I_j$ and bound $H((\varphi_i)_{i\in n I_j^s})$ in the rest of the section, since \vspace{-0.25em}
\[
H((\varphi_i)_{i\in n I_j})\ge H((\varphi_i)_{i\in n I_j^s}).
\]
For all $j\equiv 1\!\!\mod c$, our lower bound is simply 
\[
H((\varphi_i)_{i\in n I_j^s})\ge  
\Omega\Paren{\frac{1}{\sqrt{\log n}}\min\Brace{p_{I_j}, j\cdot \log n}},
\]
which, together with 
$\sum_{j\equiv 1\!\!\! \mod c}  \min\{p_{I_j}, j\cdot \log n\}
\ge {H^{\mathcal S}_n(p)}/c$, implies that
\begin{align*}
H(\varphi)
&\ge 
-\mathcal{O}\Paren{\frac{1}{\sqrt n}}+\sum_{j \equiv 1\!\!\!\!\!\mod\! c}H((\varphi_i)_{i\in n I_j})
\ge 
\Omega\Paren{\frac{1}{\sqrt{\log n}}}\cdot T_n.
\end{align*} 
Henceforth, we assume that $j$ is sufficiently large and denote $L_j:=j\sqrt{\log n}$. 

For any $j$ and every integer $i\in S_j:=[(j-1)^2, j^2]\log n$, define a sequence of intervals,
\[
I_j^i:=\frac{i}{n}+\frac{L_j}{n}\Br{-1, 1}.
\]
Then for any $i\in S_j$ and $p_x\in I_j^i \cap I_j$, the corresponding Poisson probability satisfies
\begin{align*}
e^{-np_x} \frac{(np_x)^i}{i!}
&=
 e^{-i} \frac{i^i}{i!} \cdot \exp\Paren{{-(np_x-i)}+i\cdot \log\Paren{1+\frac{np_x-i}{i}}}\\
&\ge 
\frac{1}{3\sqrt{i}}\cdot \exp\Paren{-\frac{2i}{3}\cdot \Paren{\frac{np_x-i}{i}}^2}\\
&\ge 
\frac{1}{9\sqrt{i}}
\ge \frac{1}{9L_j}.
\end{align*}
On the other hand, the following upper bound holds. 
\begin{align*}
e^{-np_x} \frac{(np_x)^i}{i!}
&=
 e^{-i} \frac{i^i}{i!} \cdot \exp\Paren{{-(np_x-i)}+i\cdot \log\Paren{1+\frac{np_x-i}{i}}}\\
&\le 
e^{-i} \frac{i^i}{i!}
\le 
\frac{1}{\sqrt {2\pi i} }
\le
\frac1{2L_j}.
\end{align*}
In other words, for any $p_x, i/n\in I_j$ that differ by at most $L_j/n$,
\[
\Pr(\Poi(np_x)=i)\in\frac{1}{L_j}\Br{\frac{1}{9}, \frac{1}{2}}.
\]
Partition $I_j$ into sub-intervals of equal length $L_j/n$. 
The partition has a size of at most $2\sqrt{\log n}$. 
Assign each probability $p_x\in I_j$ a length-$L_j/n$ interval $I_{p_x}$ centered at $p_x$.
Then, each interval $I_{p_x}$ covers at least one of the sub-intervals in the partition. 
Since there are exactly $p_{I_j}$ intervals $I_{p_x}$, 
one can find a partition sub-interval $I_j^s$ contained in 
at least $p_{I_j}/(2\sqrt{\log n})$ of them.
Denote by $\cX_s$ the collection of symbols corresponding to these intervals. 

Next, we bound from below the entropy of 
the truncated profile $(\varphi_i)_{i\in n I_j^s}$ over $n I_j^s$. 
Denote by $j_s$ the left end point of $n I_j^s$. 
By the chain rule of entropy for multiple random variables,
\[
H((\varphi_i)_{i\in n I_j^s}) 
= 
\sum_{i = j_s}^{j_s+L_j-1} H(\varphi_i|\varphi_{j_s},\ldots, \varphi_{i-1}).
\]
Consider a particular term on the right-hand side with $i\in [j_s, j_s+L_j-1]$. 
By the conditional independence and fact that conditioning reduces entropy,
\begin{align*}
H(\varphi_i|\varphi_{j_s},\ldots, \varphi_{i-1})
&\ge 
H(\varphi_i| \varphi_{j_s},\ldots, \varphi_{i-1}; \indic_{j_s\le \mu_x \le i-1}, x\in \cX)
\\
&=
H(\varphi_i| \indic_{j_s\le \mu_x \le i-1}, x\in \cX)
\\
&=
H(\varphi_i| \indic_{j_s\le \mu_x \le i-1}, x\in \cX_s; 
\indic_{j_s\le \mu_x \le i-1}, x\not \in  \cX_s)
\end{align*} 
To characterize the condition, 
we define a random variable
\[
K_i^s:= \sum_{x\in \cX_s}  \indic_{j_s\le \mu_x \le i-1}.
\]
Note that 
$\EE[\indic_{j_s\le \mu_x \le i-1}] 
= \sum_{t = j_s}^{i-1} \Pr(\Poi(np_x)=t)
\le (i-j_s)/(2L_j)$,
which is at most $1/10$ for $i\le j_s+L_j/5$. 
The following lemma transforms this into a high-probability statement. 
\begin{Lemma}
Let $Y_i, i\in [1,m]$ be independent indicator random variables. 
Let $Y:=\sum_i Y_i$ denote their sum and 
$\lambda:= \EE[Y]$ denote the expected sum. 
Then for $c>0$, we have
\[
\Pr(Y\ge \lambda(1+c)) \le \exp(-\lambda c^2/(2+2c/3)). 
\]
\end{Lemma}
Below we consider only $i\le j_s+L_j/5$.
Note that $c/(2+2c/3)$ is increasing for $c>0$.

Since $\EE[K_i^s]
=
\sum_{x\in \cX_s} 
\EE[\indic_{j_s\le \mu_x \le i-1}]
\le |\cX_s|/10$, 
\[
\Pr(K_i^s\ge |\cX_s|/2) \le \exp(-36/35)<1/2. 
\] 
where we set $c=4$ in the above lemma and assume that $|\cX_s|\ge 3$ 
(assuming only $|\cX_s|\ge 1$, the upper bound becomes $3/4$). 
Recall that
\begin{align*}
H(\varphi_i|\varphi_{j_s},\ldots, \varphi_{i-1})
&\ge 
H(\varphi_i| \indic_{j_s\le \mu_x \le i-1}, x\in \cX_s; 
\indic_{j_s\le \mu_x \le i-1}, x\not \in  \cX_s)\\
&=
\sum_{(c_x)_{x\in \cX}\in\{0,1\}^\cX}H(\varphi_i| \indic_{j_s\le \mu_x \le i-1}=c_x, x\in \cX_s
)\\
&\hspace{9em}\times \Pr(\indic_{j_s\le \mu_x \le i-1}=c_x, x\in \cX_s).
\end{align*}
Denote by $V_s\subseteq \{0, 1\}^{\cX}$ the collection of $(c_x)_{x\in \cX}$ satisfying $\sum_{x\in \cX_s} c_x < |\cX_s|/2$. 
The above derivation shows that 
\[
\sum_{(c_x)_{x\in \cX}\in V_s} 
\Pr(\indic_{j_s\le \mu_x \le i-1}=c_x, x\in \cX_s)
\ge 
\frac12.
\]
By independence, for any $(c_x)_{x\in \cX}\in V_s$, we have
\begin{align*}
(\varphi_i| \indic_{j_s\le \mu_x \le i-1}=c_x, x\in \cX_s)
&= 
\sum_{x\in \cX: c_x=0} (\indic_{\mu_x = i}| \indic_{j_s\le \mu_x \le i-1}=0)
\\
&=
\sum_{x\in \cX_s: c_x=0} (\indic_{\mu_x = i}| \indic_{j_s\le \mu_x \le i-1}=0)
\\&
\hspace{8em}
+\!\!\!\!
\sum_{x\not \in \cX_s: c_x=0} (\indic_{\mu_x = i}| \indic_{j_s\le \mu_x \le i-1}=0).
\end{align*}
For any $x\in \cX_s$ with $c_x=0$, the corresponding indicator variable satisfies
\begin{align*}
\EE[\indic_{\mu_x = i}| \indic_{j_s\le \mu_x \le i-1}=0]
&=
\frac{\Pr(\indic_{\mu_x = i}\text{ and } \mu_x\not \in [j_s, i-1])}{\Pr(\mu_x\not \in [j_s, i-1])}
\\
&=
\frac{\Pr(\indic_{\mu_x = i})}{1-\Pr(\mu_x \in [j_s, i-1])}
\\
&=
\frac{\frac{1}{L_j}\Br{\frac{1}{9}, \frac{1}{2}}}{1-\Br{0, \frac{L_j}{5}}\cdot \frac{1}{L_j}\Br{\frac{1}{9}, \frac{1}{2}}}
\\
&=
\frac{1}{L_j}\Br{\frac{1}{9}, \frac{5}{9}}.
\end{align*}
On the other hand, for any $x\not\in \cX_s$, 
\begin{align*}
e^{-np_x} \frac{(np_x)^i}{i!}
&\le 
e^{-i} \frac{i^i}{i!}
\le 
\frac{1}{\sqrt {2\pi i} }
\le
\frac1{2L_j}.
\end{align*}
Therefore, the corresponding indicator variable satisfies
\begin{align*}
\EE[\indic_{\mu_x = i}| \indic_{j_s\le \mu_x \le i-1}=0]
&=
\frac{\Pr(\indic_{\mu_x = i})}{1-\Pr(\mu_x \in [j_s, i-1])}
\le
\frac{\frac{1}{L_j}\Br{0, \frac{1}{2}}}{1-\Br{0, \frac{L_j}{5}}\cdot \frac{1}{L_j}\Br{0, \frac{1}{2}}}
\le
\frac{5}{9}\cdot \frac{1}{L_j}.
\end{align*}
To summarize, we have shown that $(\varphi_i| \indic_{j_s\le \mu_x \le i-1}=c_x, x\in \cX_s)$
is the sum of $|\cX|$ independent Bernoulli random variables. 
Among these Bernoulli variables, at least $|\cX_s|/2\ge p_{I_j}/(2\sqrt{\log n})$ have a bias of $\frac{1}{L_j}\Br{\frac{1}{9}, \frac{5}{9}}$, 
while others have a bias of at most $\frac{5}{9}\cdot \frac{1}{L_j}$. 

The following lemma, recently established by~\cite{hillion2019proof}, shows the relation among the entropy values of sums of independent Bernoulli random variables with different bias parameters. 
\begin{Lemma}
Let $X_t, Y_t, t\in [m]$ be independent indicator random variables.
Denote by $X$ and $Y$ the sums of $X_t$'s and $Y_t$'s, respectively. 
If $\EE[X_t]\le \EE[Y_t]\le 1/2, \forall t\in [m]$, 
\[
H(\sum_t X_t)\le H(\sum_t Y_t).  
\]
\end{Lemma}
This lemma, together with the previous results, shows that
\begin{align*}
H(\varphi_i| \indic_{j_s\le \mu_x \le i-1}=c_x, x\in \cX_s)
&\ge 
H(\bin(p_{I_j}/(2\sqrt{\log n}), 1/(9L_j)).
\end{align*}
The next lemma further bounds the entropy of a binomial random variable.
\begin{Lemma}
For any $m>1$ and $q\in [1/m, 1-1/m]$,
\[
H(\bin(m, q))\ge \frac{1}{2}\log\Paren{(2\pi)^{1-(1-q)^m-q^m} m q(1-q)}-\frac{1}{12m}.
\vspace{-0.5em}
\]
\end{Lemma}
\begin{proof}
By definition, the left-hand side satisfies
\begin{align*}
H(\bin(m, q))
&=
-\sum_{t=0}^m \binom{m}{t}q^t(1-q)^{m-t} \log\Paren{\binom{m}{t}q^t(1-q)^{m-t}}
\\
&=
-\sum_{t=0}^m \binom{m}{t}q^t(1-q)^{m-t} 
(t\log q+(m-t)\log (1-q)\\
&
\hspace{12em} +\log m!-\log t!-\log (m-t)!)\\
&=
m H(\Bern(q))
-\log m! +\sum_{t=0}^m \binom{m}{t}q^t(1-q)^{m-t}(\log t!+\log (m-t)!).
\end{align*}
By Stirling's formula, for any $t\ge 1$,
\[
\log t! \ge  \Paren{t+\frac12}\log t+\frac12 \log (2\pi) -t.
\]
Substituting the right-hand side into the above equation yields
\begin{align*}
S_{m}(q)
:=\sum_{t=0}^m \binom{m}{t}q^t(1-q)^{m-t}\log t!
&\ge 
\frac12(1-(1-q)^m) \log (2\pi)
-m q\\
&\hspace{5em}
+\sum_{t=1}^m \binom{m}{t}q^t(1-q)^{m-t}\Paren{t+\frac12}\log t.
\end{align*}
Let $g(x):=0$ for $x\in [0,1)$ and $g(x):=(x+1/2)\log x$ for $x\ge 1$. 
Simple calculus shows that the function is concave.  
Applying the concavity of $g$ to the last sum yields
\begin{align*}
\sum_{t=1}^m \binom{m}{t}q^t(1-q)^{m-t}\Paren{t+\frac12}\log t
&\ge
g\Paren{\sum_{t=0}^m \binom{m}{t}q^t(1-q)^{m-t}\cdot t}
=\Paren{mq+\frac12}\log(mq),
\end{align*}
where the last step follows by the fact that $mq\ge 1$. 
A similar inequality holds for the weighted sum of  $\log (m-t)!$. 
Consolidating these inequalities, we obtain
\begin{align*}
S_{m}(q)+S_{m}(1-q)
&\ge
\Paren{mq+\frac12}\log(mq)+\Paren{m(1-q)+\frac12}\log(m(1-q))\\
&
+\frac12(1-(1-q)^m) \log (2\pi)
-m q
+\frac12(1-q^m) \log (2\pi)
-m (1-q)\\
&=
(m+1)\log m - m H(\Bern(q))
+\frac12\log(q(1-q))\\
&\hspace{9em}
+\frac12(2-(1-q)^m-q^m) \log (2\pi)-m.
\end{align*}
On the other hand, for the $\log m!$ term,
\[
\log m! \le  \Paren{m+\frac12}\log m+\frac12 \log (2\pi) -m+\frac{1}{12m}.
\]
Substituting the previous term bounds into the $H(\bin(m,q))$ expression yields
\begin{align*}
H(\bin(m, q))
&=
m H(\Bern(q))
-\log m! +S_{m}(q)+S_{m}(1-q)\\
&\ge 
\frac{1}{2}\log\Paren{(2\pi)^{1-(1-q)^m-q^m} m q(1-q)}-\frac{1}{12m}. \qedhere
\end{align*}
\end{proof}
Before continuing, we remark that the bound 
in the above lemma has the right dependence on $mq(1-q)$
 in the sense that if we fix $q$
 and increase $m$, the lower bound converges to $\frac12\log(\Theta( mq(1-q)))$. 
Another point to mention is that the above bound covers $q\in [1/m, 1-1/m]$,
while Lemma~\ref{other_case} appearing later in this section covers 
$q\not\in [1/m, 1-1/m]$.
Note that the dependence on $mq(1-q)$ changes from logarithmic to 
linear, showing an ``elbow effect'' around $1/m$. 

Assume that $p_{I_j}/(2\sqrt{\log n})\ge 9L_j$, then for any $(c_x)_{x\in \cX}\in V_s$,
\[
H(\varphi_i| \indic_{j_s\le \mu_x \le i-1}=c_x, x\in \cX_s)
\ge
H(\bin(p_{I_j}/(2\sqrt{\log n}), 1/(9L_j))
\ge
\frac{1}{2}.
\]
Consolidating this with the previous results yields that
\begin{align*}
H(\varphi_i|\varphi_{j_s},\ldots, \varphi_{i-1})
&\ge 
\sum_{(c_x)_{x\in \cX}\in V_s}\frac{1}{2}\cdot \Pr(\indic_{j_s\le \mu_x \le i-1}=c_x, x\in \cX_s)
\ge \frac12\cdot \frac12
=\frac14,
\end{align*}
where we utilize $p_{I_j}/(2\sqrt{\log n})\ge 9L_j\ge 9$ and $(1-q)^m+q^m<1/e$ for $\forall m\ge 3, q\in [1/m, 1/2]$. 
We can then bound the quantity of interest as follows.
\begin{align*}
H((\varphi_i)_{i\in n I_j^s}) 
&= 
\sum_{i = j_s}^{j_s+L_j-1} H(\varphi_i|\varphi_{j_s},\ldots, \varphi_{i-1})\\
&\ge
 \sum_{i = j_s}^{j_s+L_j/5} H(\varphi_i|\varphi_{j_s},\ldots, \varphi_{i-1})\\
&\ge 
\frac{L_j}{5}\cdot \frac14 = \frac{L_j}{20}\\
&=
\frac{1}{20\sqrt{\log n}}\min\Brace{p_{I_j}, j\cdot \log n}.
\end{align*}
On the other hand, if $9L_j\ge p_{I_j}/(2\sqrt{\log n})\gg 1$, 
we can further ``compress'' the truncated profile $(\varphi_i)_{i\in n I_j^s}$ over $n I_j^s$
to reduce the effective value of $L_j$. 
Specifically, for any integer $t<L_j$, 
we define the $t$-compressed version of $(\varphi_i)_{i\in n I_j^s}$ as \vspace{-0.5em}
\[
(\varphi_i)_{i\in n I_j^s}^t
:=
\Paren{\sum_{i = j_s+(\ell-1)t}^{j_s+\ell t-1} \varphi_i}_{\ell\in [L_j/t]}.
\]
Note that for each $t$, the length of $(\varphi_i)_{i\in n I_j^s}^t$ is $L_j^t:= L_j/t$.
For each entry in the compressed version, 
we can again express the entry as the sum of independent indicator random variables. 
Specifically, 
\[
\sum_{i = j_s+(\ell-1)t}^{j_s+\ell t-1} \varphi_i
=
\sum_{x\in \cX} \indic_{\mu_x\in [j_s+(\ell-1)t, j_s+\ell t-1]}.
\]
Furthermore, for any $x\in \cX_s$, the expectation of each indicator variable satisfies 
\begin{align*}
\EE[\indic_{\mu_x\in [j_s+(\ell-1)t, j_s+\ell t-1]}]
&=
\sum_{i=j_s+(\ell-1)t}^{j_s+\ell t-1} e^{-np_x}\frac{(np_x)^i}{i!}\\
&=
\frac{t}{L_j}\Br{\frac{1}{9}, \frac{1}{2}}
=
\frac{1}{L_j^t}\Br{\frac{1}{9}, \frac{1}{2}}.
\end{align*}
Similarly, for any $x\in \cX$, 
we have $\EE[\indic_{\mu_x\in [j_s+(\ell-1)t, j_s+\ell t-1]}]\le 1/(2L_j^t)$.

Now, choose $t$ large enough so that 
$18L_j^t\ge p_{I_j}/(2\sqrt{\log n})\ge 9L_j^t$.
Following the reasoning in the previous case shows that \vspace{-0.5em}
\[
H((\varphi_i)_{i\in n I_j^s}) \ge H((\varphi_i)_{i\in n I_j^s}^t)
\ge \Omega\Paren{\frac{1}{\sqrt{\log n}}\min\Brace{p_{I_j}, j\cdot \log n}}.
\]
It remains to consider the case of $\mathcal{O}(\sqrt{\log n})\ge p_{I_j}\ge 1$, 
for which we adopt our previous analysis. 

Again, partition $I_j$ into sub-intervals of equal length $L_j/n$. 
Then, assign each probability $p_x\in I_j$ a length-$L_j/n$ 
interval $I_{p_x}$ centered at $p_x$.
By construction, each interval $I_{p_x}$ covers 
at least one of the sub-intervals in the partition. 
Redefine any of these covered sub-intervals as $I_j^s$. 
Denote by $\cX_s$ the collection of symbols corresponding to the covering intervals. 

Note that $\mathcal{O}(\sqrt{\log n})\ge p_{I_j}\ge |\cX_s|\ge 1$.  
For any $i\in [j_s, j_s\!+\!L_j/5]$, the previous analysis shows that
\begin{align*}
H(\varphi_i|\varphi_{j_s},\ldots, \varphi_{i-1})
\ge H(\bin(|\cX_s|, 1/(9L_j))
\cdot \Paren{1-3/4}. 
\end{align*}
We bound the right-hand side with the following lemma.
\begin{Lemma}\label{other_case}
For any $m\ge 1$, and $q\le \min\{1/2, 1/m\}$ or $q\ge \max\{1/2, 1-1/m\}$,
\[
H(\bin(m, q))\ge \frac{m}{4}\min\{q,1-q\}\ge 
\frac{1}{4} mq(1-q).  \vspace{-0.5em}
\]
\end{Lemma}
\begin{proof}
By symmetry, we need to consider only the case of $q\in [0,1/m]$. 
\begin{align*}
H(\bin(m, q))
&\ge 
H(\indic_{\bin(m, q)\ge 1})\\
&=
H(((1-q)^m, 1-(1-q)^m))\\
&\ge
-(1-q)^m (m\log(1-q))\\
&\ge 
-\frac{m}{4} \log(1-q)\\
&\ge
\frac{m}{4}\cdot q.\qedhere
\end{align*}
\end{proof}
Consolidating the lemma and the chain rule of entropy yields, 
\begin{align*}
H((\varphi_i)_{i\in n I_j^s}) 
&= 
\sum_{i = j_s}^{j_s+L_j-1} H(\varphi_i|\varphi_{j_s},\ldots, \varphi_{i-1})\\
&\ge
 \sum_{i = j_s}^{j_s+L_j/5} H(\varphi_i|\varphi_{j_s},\ldots, \varphi_{i-1})\\
&\ge 
\frac{L_j}{5}\cdot \frac{|\cX_s|}{4\cdot 9\cdot L_j}\cdot \Paren{1-\frac34} 
= \frac{|\cX_s|}{720}\\
&=
 \Omega\Paren{\frac{1}{\sqrt{\log n}}\min\Brace{p_{I_j}, j\cdot \log n}}.
\end{align*}

Alternatively, we can use the fact that adding 
independent random variables does not decrease entropy, 
i.e., $H(Y+Z)\ge H(Y)$ for any independent variables $Y$ and $Z$. 
Note that 
\[
(\varphi_i)_{i\in n I_j^s}^t = \sum_{x\in\cX} (\indic_{\mu_x=i})_{i\in I_j^s}. 
\vspace{-0.25em}
\]
Let $y$ be an arbitrary symbol that belongs to $\cX_s$. Then, 
\[
H((\varphi_i)_{i\in n I_j^s}) 
\ge H((\varphi_i)_{i\in n I_j^s}^t)
\ge H((\indic_{\mu_y=i})_{i\in I_j^s})
\ge H((\indic_{\mu_y=j_s}, \indic_{\mu_y=j_s+1})).
\]
By the previous derivations, 
both $\Pr(\mu_y=j_s)$ and $\Pr(\mu_y=j_s+1)$ 
belong to $\frac{1}{L_j}[1/9,1/2]$.  
Hence,
\[
H((\varphi_i)_{i\in n I_j^s}) 
\ge H\Paren{\Bern\Paren{\frac2{11}}}
\ge \frac25
=
\Omega\Paren{\frac{1}{\sqrt{\log n}}\min\Brace{p_{I_j}, j\cdot \log n}}. 
\]
Note that this argument does not apply to other cases, since 
\[
H((\indic_{\mu_y=i})_{i\in I_j^s})=\mathcal{O}(\log L_j)=\mathcal{O}(\log n), \vspace{-0.25em}
\] 
while $\min\Brace{p_{I_j}, j\cdot \log n}$ can be as large as 
$\tilde{\Theta}(n^{1/3})$ in general.

The proof is complete upon noting that indices with $j=\mathcal{O}(1)$ 
corresponds to a total contribution of at most $\mathcal O(1)$ to $H^{\mathcal S}_n(p)$
and $H^{\mathcal S}_n(p) = \tilde\Theta(\EE[\mathcal D(\varphi)])=\tilde\Theta(D(\varphi))$,
with probability at least $1-\mathcal{O}(1/\sqrt n)$. 

\paragraph{Summary} The simple expression shows that 
$H^\mathcal{S}_n(p)$ characterizes the 
variability of ranges the actual probabilities spread over.
As Theorem~\ref{thm:hs_app_en} shows, 
$H^\mathcal{S}_n(p)$ closely approximates $E_n(p)$, 
the value around which $\mathcal D_n\sim p$ concentrates 
(Theorem~\ref{thm:exp_conc}) and $H(\Phi^n)$ lies 
(Thoerem~\ref{thm:entro_eql_dim}). 
Henceforth, we use $H^\mathcal{S}_n(p)$ as a proxy for both 
$H(\Phi^n)$ and $\mathcal D_n$,  
and study its attributes and values. 

\vfill
\pagebreak
\subsection{Symmetric Property Estimation and Sufficiency of Profiles}

The rest of Section~\ref{sec:dim_and_entro} 
shows that the PML plug-in estimator possesses the 
amazing ability of adapting to the simplicity of data 
distributions in inferring all symmetric properties,
over any label-invariant classes. 
For clarity, we divide the full proof into three parts:
a) the sufficiency of profiles for estimating symmetric properties;
b) the standard ``median trick'' often used to boost the confidence of learning algorithms;
c) the PML method and its competitiveness to the min-max estimators. 
The proof utilizes several previously established results.

\paragraph{Sufficiency of profiles} We first show that profile-based estimators are sufficient for estimating symmetric properties. Recall that a distribution collection $\cP\subseteq \Delta_\cX$ is \emph{label-invariant} if for any $p\in \cP$, 
the collection $\cP$ contains all its symbol-permuted versions.
Then,
\setcounter{Theorem}{1}
\begin{Theorem}\label{thm1}
Let $f$ be a symmetric functional over a label-invariant distribution collection $\cP\subseteq \Delta_\cX$. For any accuracy $\ve>0$ 
and tolerance $\delta\in(0,1)$, if there exists an estimator $\hat{f}$ such that
\[
\Pr_{X^n\sim p}
\Paren{\Abs{\hat f(X^n)-f(p)}> \ve}< \delta, \ \forall p\in \cP,
\]
there is an estimator $\hat f_\varphi$ over $\Phi$ satisfying 
\[
\Pr_{X^n\sim p}
\Paren{\Abs{\hat f_\varphi(\varphi(X^n))-f(p)}> \ve}< \delta, \ \forall p\in \cP.
\]\vspace{-1.5em}
\end{Theorem} 
Note that both estimators can have independent randomness. 

\begin{proof}
First we show that given estimator $\hat f$, 
 there is an estimator $\hat f_{s}$ which is symmetric, i.e., invariant with respect to domain-symbol permutations, and achieves the same guarantee. 
To see this, consider a random permutation $\tilde \sigma$ chosen uniformly randomly 
from the collection of permutations over the underlying alphabet. 
Let $\hat f_s:=\hat f\circ \tilde \sigma$. Then for any $p\in \cP$, 
\begin{align*}
\Pr_{X^n\sim p}
\Paren{\Abs{\hat f_s(X^n)-f(p)}> \ve}
&\overset{(a)}=
\Pr_{X^n\sim p}
\Paren{\Abs{\hat f\circ \tilde \sigma(X^n)-f(p)}> \ve}\\
&\overset{(b)}=
\sum_\sigma\Pr_{X^n\sim p}
\Paren{\Abs{\hat f\circ \sigma(X^n)-f(p)}> \ve \big\vert\ \tilde{\sigma} = \sigma}
\cdot \Pr\Paren{\tilde \sigma = \sigma}\\
&\overset{(c)}=
\sum_\sigma\Pr_{X^n\sim p}
\Paren{\Abs{\hat f\circ \sigma(X^n)-f(p)}> \ve }
\cdot \Pr\Paren{\tilde \sigma = \sigma}\\
&\overset{(d)}=
\sum_\sigma\Pr_{X^n\sim \sigma(p)}
\Paren{\Abs{\hat f(X^n)-f(\sigma(p))}> \ve }
\cdot \Pr\Paren{\tilde \sigma = \sigma}\\
&\overset{(e)}<
\sum_\sigma
\delta \cdot \Pr\Paren{\tilde \sigma = \sigma}\\
&\overset{(f)}=
\delta,
\end{align*} 
where $(a)$ follows by the definition of $\hat f_s$;
$(b)$ follows by the law of total probability;
$(c)$ follows by the independence between $\tilde \sigma$ and $X^n$;
$(d)$ follows by the symmetry of $f$ and the equivalence of applying $\sigma$ to $X^n$ and to $p$; 
$(e)$ follows by the fact that $\sigma (p)\in\cP$ and the guarantee satisfied by the estimator $\hat f$;
and $(f)$ follows by the law of total probability.  

Before we proceed further, we introduce the following definitions. 
For any sequence $x^n$, the \emph{sketch} of a symbol $x$ in $x^n$ is the set of 
indices $i\in [n]$ for which $x_i=x$. 
The \emph{type} of a sequence $x^n$ is the 
set $\tau(x^n)$ of sketches of symbols appearing in $x^n$. 

Since $\hat f_s$ is symmetric, there exists a mapping $\hat f_\tau$ over types satisfying
$\hat f_s = \hat f_\tau \circ \tau$. 
Due to the \iid assumption on the sample generation process, 
given the profile of a sample sequence, all the different types corresponding to this 
profile are equally likely. 
Let $\Lambda$ be a mapping that recovers this relation, i.e., $\Lambda$ maps each profile 
uniformly randomly to a type having this profile.

Then, for any $p\in \cP$ and $X^n\sim p$,
\[
\hat f_s (X^n) 
= \hat f_\tau \circ \tau (X^n) 
=  \hat f_\tau \circ \Lambda \circ \varphi (X^n).
\]
Consequently, the mapping $\hat f_\varphi:=\hat f_\tau \circ \Lambda$ is a profile-based estimator that satisfies
\[
\Pr_{X^n\sim p}
\Paren{\Abs{\hat f_\varphi(\varphi(X^n))-f(p)}> \ve}
=
\Pr_{X^n\sim p}
\Paren{\Abs{\hat f_s(X^n)-f(p)}> \ve}
<
\delta,\
\forall p\in\cP. \qedhere
\]
\end{proof}

\subsection{Median Trick}
The following argument is standard and often used to boost 
the confidence of learning algorithms. 
\begin{Lemma}[Median trick]\label{lem:median_trick}
Let $\alpha,\beta\in (0,1)$ be real parameters satisfying $1/10\ge \alpha>\beta$. 
For an accuracy $\ve>0$ and a distribution set $\cP\subseteq \Delta_\cX$,
if there exists an estimator $\hat{f}_A$ such that
\[
\Pr_{X^n\sim p}
\Paren{\Abs{\hat f_A(X^n)-f(p)}> \ve}< \alpha, \ \forall p\in \cP,
\]
we can construct another estimator $\hat f_B$ that takes a sample of size 
$m: = \left\lceil\frac{4n}{\log \frac{1}{2\alpha}} \log \frac1{\beta}\right\rceil$ and achieves 
\[
\Pr_{Y^m\sim p}
\Paren{\Abs{\hat f_B(Y^{m})-f(p)}> \ve}< \beta, \ \forall p\in \cP.
\]
\end{Lemma}
\begin{proof}
Given $t\in\mathbb{N}$ \iid copies of $\hat{f}_A(X^n)$, the probability that less than 
half of them satisfy the inequality in the parentheses is at least
\[
\Pr\Paren{ \sum_{i=1}^t \indic_{A_i} < \frac{t }{2} 
\text{ for $A_i$'s} \text{ satisfying } \Pr(A_i)< \alpha}
\ge
\Pr\Paren{
\bin\Paren{t, \alpha}
<
\frac t2
}.
\]
By the law of total probability, the right-hand side equals to 
\begin{align*}
1-\Pr\Paren{
\bin\Paren{t, \alpha}
\ge
\frac t2}
&\ge 
1-\exp\Paren{\Paren{\Paren{\frac{1}{2\alpha}-1}-\frac{1}{2\alpha}\log \frac{1}{2\alpha}}\cdot \alpha t}\\
&\ge
1-\exp\Paren{-\frac{t}{4}\log \frac{1}{2\alpha}}
,
\end{align*}
where the first step follows by the Chernoff bound of binomial random variables, 
and the second step follows by 
$\alpha \le1/10$ and the inequality $c-1-\frac c2 \log c>0, \forall c\ge5$.

Set $t:=\left\lceil\frac{4}{\log \frac{1}{2\alpha}} \log \frac1{\beta}\right\rceil$, 
the right-hand side is at least $1-\beta$. 

Therefore, given a sample of size $m = t\cdot n$, we can partition it into $t$ sub-samples of equal size, apply the estimator $\hat f_A$ to each subsample, and define the median of the corresponding estimates~as~$\hat f_B$. 

By the previous reasoning, this estimator satisfies
\[
\Pr_{Y^m\sim p}
\Paren{\Abs{\hat f_B(Y^{m})-f(p)}> \ve}< \beta, \ \forall p\in \cP. \qedhere
\]
\end{proof}

\subsection{Profile Maximum Likelihood and Its Adaptiveness}\label{sup:1.5}

For every profile $\bphi$ of length $n$ and distribution collection $\cP\subseteq\Delta_\cX$, 
the \emph{profile maximum likelihood} (PML) estimator~\cite{orlitsky2004modeling} over $\cP$ maps $\bphi$ to a distribution 
\[
\cP_\bphi:=\argmin_{p\in \cP} \Pr_{X^n\sim p}\!\Paren{\va(X^n)=\bphi}, 
\]
that maximizes the probability of observing the profile~$\bphi$.

For any property $f$, let $\Ve_f(n, \delta, \cP)$ denote 
the smallest error that can be achieved by any estimator 
with a sample size $n$ and tolerance $\delta$ on the error probability.
This definition is equivalent to that of the sample complexity. 
In the following, we show that the PML estimator is adaptive to the 
simplicity of underlying distributions in inferring all symmetric properties,
over any label-invariant $\cP$.  
 
For brevity, set $\delta=1/10$ and 
suppress both $\delta$ and $\cP$ in $\Ve_f(n, \delta, \cP)$.

\begin{Theorem}[Adaptiveness of PML]\label{thm:adp_pml}
Let $f$ be a symmetric property and $\cP\subseteq\Delta_\cX$
be a label-invariant distribution collection. 
For any $p\in \cP$ and $\Phi^n\sim p$, 
with probability at least $1-\mathcal{O}(1/\sqrt n)$, \vspace{-0.5em}
\[
\Abs{f(p)-f(\cP_{\Phi^n}\!)}\le 
2\Ve\!_f\!\Paren{\!
\frac{\tilde\Omega\!\Paren{n}}{\C{H\!(\Phi^n)}} 
\!}. 
\]
\end{Theorem}
\begin{proof}
For any tolerance $\delta\in (0,1)$ and 
distribution $p\in \Delta_\cX$, define the 
\emph{$(\delta, n)$-typical cardinality of profiles 
with respect to $p$}
 as the smallest cardinality $C_{\delta, n}(p)$ of a set of length-$n$ profiles such 
that the probability of observing a sample from $p$ with a profile in this set is at least $1-\delta$. The following lemma 
provides a tight characterization of $C_{\delta, n}(p)$ in terms 
of the dimension of $\Phi^n\sim p$. 
\begin{Lemma}\label{typical_profile}
For any $p\in\Delta_\cX$ and $\Phi^n\sim p$,  
with probability at least $1-6/\sqrt{n}$,
\[
C_{\frac{6}{\sqrt{n}}, n}(p)
\le 
n^{8\Paren{\mathcal D(\Phi^n) + 20\log n}}.
\]
\end{Lemma}
The proof of the lemma follows by recursively applying Theorem~\ref{thm:exp_conc}.
Specifically, let $d:=2E_n(p)+3\log n$, which is at least $\mathcal D_n\sim p$,
with probability at least $1-6/\sqrt n$. Then, 
\[
C_{\frac{6}{\sqrt{n}}, n}(p) 
\le \binom{n}{d} \binom{n+d-1}{d-1}
\le n^{2d-1}
\le n^{2\Paren{2E_n(p)+3\log n}} 
\le n^{8{\mathcal D}(\Phi^n) + 20\log n},
\]
where the last inequality holds with with probability at least $1-6/\sqrt n$. 

Now let $f$ be a symmetric functional over $\cP$. According to Theorem~\ref{thm1}, for any parameters $\ve>0$ and $\delta\in(0,1)$, if there exists an estimator $\hat{f}$ such that 
\[
\Pr_{X^n\sim p}
\Paren{\Abs{\hat f(X^n)-f(p)}> \ve}< \delta, \ \forall p\in \cP,
\]
there is an estimator $\hat f_\varphi$ over $\Phi$ satisfying 
\[
\Pr_{X^n\sim p}
\Paren{\Abs{\hat f_\varphi(\varphi(X^n))-f(p)}> \ve}< \delta, \ \forall p\in \cP.
\]
For an arbitrary length-$n$ profile $\boldsymbol\phi$ that satisfies $\Pr_{\Phi^n\sim p}(\Phi^n= \boldsymbol\phi)\ge 2\delta$, these error bounds yield
\[
\Pr
\Paren{\Abs{\hat f_\varphi(\boldsymbol\phi)-f(p)}> \ve}< \frac{1}{2}.
\]
and since 
$\Pr_{\Phi^n\sim \cP_{\boldsymbol\phi}}\!\Paren{\Phi^n=\boldsymbol\phi}\ge 
\Pr_{\Phi^n\sim p}(\Phi^n= \boldsymbol \phi)\ge 2\delta$ by the definition of PML, 
\[
\Pr\Paren{\Abs{\hat f_\varphi(\boldsymbol\phi)-f(\cP_{\boldsymbol\phi})}> \ve} 
< \frac{1}{2}.
\]
By the union bound and triangle inequality, 
\[
\Pr\Paren{\Abs{f(p)-f(\cP_{\boldsymbol\phi})}> 2\ve} < 1
\iff
\Abs{f(p)-f(\cP_{\boldsymbol\phi})}\le 2\ve\ \text{ surely}.
\]
Furthermore, 
by Lemma~\ref{typical_profile},
with probability at least $1-6/\sqrt{n}$,
the total probability of length-$n$ profiles $\boldsymbol\phi$ 
satisfying $\Pr_{\Phi^n\sim p}(\Phi^n = \boldsymbol\phi)< 2\delta$  is at most 
\[
2\delta\cdot C_{\frac{6}{\sqrt{n}}, n}(p)
+\frac{6}{\sqrt{n}}
\le 
2\delta\cdot n^{8\mathcal D(\Phi^n) + 20\log n}+\frac{6}{\sqrt{n}},
\]
which basically upperly bounds the probability that 
\[
\Abs{f(p)-f(\cP_{\Phi^n})}> 2\ve.
\]
Next we will assume that
there exists an estimator $\hat{f}$ satisfying
\[
\Pr_{X^m\sim p}
\Paren{\Abs{\hat f(X^m)-f(p)}> \ve}< \delta, \ \forall p\in \cP.
\]
By Lemma~\ref{lem:median_trick}, if $\delta\le 1/10$, we can construct another 
estimator $\hat f'$ that takes a sample of size 
$n = \frac{4m}{\log \frac{1}{2\delta}} \log \frac1{\delta'}$ 
($n$ is assumed to be an integer here) 
and achieves a higher-confidence guarantee 
\[
\Pr_{X^n\sim p}
\Paren{\Abs{\hat f'(X^{n})-f(p)}> \ve}< \delta', \ \forall p\in \cP.
\]
Then by the above reasoning, with probability at least $1-6/\sqrt n$, 
\begin{align*}
\Pr_{\Phi^n\sim p}\Paren{\Abs{f(p)-f(\cP_{\Phi^n})}>2\ve}
&\le 
2\delta'\cdot n^{8\mathcal D(\Phi^n) + 20\log n}+\frac{6}{\sqrt{n}}\\
&= 
2\exp\Paren{-\frac{n}{4m}\log\frac{1}{2\delta}
+\Paren{8\mathcal D(\Phi^n) + 20\log n}\log n
}+\frac{6}{\sqrt{n}}.
\end{align*}
For the first term on the right hand side to 
vanish as quickly as $1/\sqrt{n}$, it suffices to have 
\begin{align*}
\frac{n}{4m}\log\frac{1}{2\delta}\ge 20\cdot \mathcal D(\Phi^n)\log n
&\iff
\frac{n}{\mathcal D(\Phi^n)\log n} \ge 80\cdot \frac{m}{\log \frac{1}{2\delta}},
\end{align*}
and simultaneously have 
\begin{align*}
\frac{n}{4m}\log\frac{1}{2\delta}\ge 40\cdot \log^2 n
&\iff
\frac{n}{\log^2 n} \ge 160\cdot\frac{m}{\log \frac{1}{2\delta}}.
\end{align*}
Simplify the expressions and apply the union bound. It suffices to have both
\[
\frac{\tilde \Theta (n)}{\mathcal D(\Phi^n)} \ge  \frac{m}{\log \frac{1}{\delta}}
\text{ and }
n\ge 8m.
\]
If $n$ satisfies these conditions, with probability at least $1-\Theta(1/\sqrt n)$,
\[
\Abs{f(p)-f(\cP_{\Phi^n})}\le 2\ve.
\]

\paragraph{Summary} 
We have shown the following result,
which is a strengthened version of Theorem~\ref{thm:adp_pml}.  
The result shows that the PML plug-in estimator possesses the 
amazing ability of adapting to the simplicity of data 
distributions in inferring all symmetric properties,
over any label-invariant classes.

\emph{ If there exists an estimator $\hat{f}$ such that
\[
\Pr_{X^m\sim p}
\Paren{\Abs{\hat f(X^m)-f(p)}> \ve}< \delta, \ \forall p\in \cP,
\]
for any $p\in \cP$ and $\Phi^n\sim p$ where the sample size $n$ satisfies both
\[
\frac{\tilde \Theta (n)}{\mathcal D(\Phi^n)} \ge  \frac{m}{\log \frac{1}{\delta}}
\text{ and }
n\ge 8m,
\]
with probability at least $1-\Theta(1/\sqrt n)$,
\[
\Abs{f(p)-f(\cP_{\Phi^n})}\le 2\ve. \vspace{-1.5em}
\]}
\subsection*{Alternative statement}
Fix $\cP$ and assume that $\delta\le1/10$. 
Recall that $\ve\!_f(n, \delta)$ denotes the smallest error 
that can be achieved by the best estimator using 
a size-$n$ sample with a $(1-\delta)$-confidence guarantee. 
Below we provide an alternative statement of the above result,
which is more compact in its form. 

Then, draw a profile $\Phi^n\sim p$. 
With probability at least $1-\Theta(1/\sqrt n)$, 
\[
\Abs{f(p)-f(\cP_{\Phi^n})}\le 
2\text{\large$\ve$}\!_f\!\Paren{
\frac{\tilde\Theta\!\Paren{n} \log \frac{1}{\delta}}{\mathcal{D}(\Phi^n)}
\Lambda\
\frac{n}{8} 
,\delta
}. \vspace{-0.5em}
\]
Consolidating this with Theorem~\ref{thm:entro_eql_dim} yields that 
\[
\Abs{f(p)-f(\cP_{\Phi^n})}\le 
2\text{\large$\ve$}\!_f\!\Paren{
\frac{\tilde\Theta\!\Paren{n} \log \frac{1}{\delta}}{\C{H\!(\Phi^n)}}
\Lambda\
\frac{n}{8}
,\delta
}. 
\]
Setting $\delta=1/10$ and suppressing it in expressions,
we establish the desired guarantee:  
For any $p\in \cP$ and $\Phi^n\sim p$, 
with probability at least $1-\mathcal{O}(1/\sqrt n)$, \vspace{-0.5em}
\[
\Abs{f(p)-f(\cP_{\Phi^n}\!)}\le 
2\Ve\!_f\!\Paren{\!
\frac{\tilde\Omega\!\Paren{n}}{\C{H\!(\Phi^n)}} 
\!}. \qedhere
\]
Below are some comments in order.
\begin{enumerate}
\item The theorem holds for any symmetric properties, 
while nearly all previous works require the property to possess certain forms and be smooth.
\item The theorem trivially implies a weaker result in~\cite{acharya2017unified} 
where $\C{H\!(\Phi^n)}$ is replaced by $\mathcal{O}(\sqrt n)$, an upper bound due to the 
formula of~\cite{hardy1918asymptotic}.
\item There is a polynomial-time approximation algorithm~\cite{charikar2019bethe} 
achieving the same guarantee as that stated in the theorem.
\end{enumerate}

\subsection{PML and Sorted Distribution Estimation}\label{sec:pml_sort}
The arguments below basically follow by~\cite{hao2019broad} and~\cite{hao2019unified}, 
in which we assume that $|\cX|=\mathcal{O}(n\log n)$.

Let $f$ be a function in $\mathcal{L}_1$, the collection of Lipschitz functions over $[0,1]$. Without loss of generality, we also assume that $f(0)=0$.  Let $\eta\in(0,1)$ be a threshold parameter to be determined later. An \emph{$\eta$-truncation} of $f$ is a function
\[
f_\eta(z):= f(z)\indic_{z\leq \eta}+f(\eta) \indic_{z> \eta}. 
\]
One can easily verify that $f_\eta\in\mathcal{L}_1$. 
Next, we find a finite subset of $\mathcal{L}_1$ so that the $\eta$-truncation of any $f\in\mathcal{L}_1$ is close to at least one of the functions in this subset. 

For an integer parameter $s>3$ to be chosen later. Partition the interval $[0,\eta]$ into $s$ disjoint sub-intervals of equal length, and define the sequence of end points as $z_j:=\eta\cdot j/s, j\in \lceil s\rfloor$ where $\lceil s\rfloor:=\{0,1,\ldots, s\}$. Then, for each $j\in\lceil s\rfloor$, 
we find the integer $j'$ such that $|f_\eta(z_j)-z_{j'}|$ is minimized and denote it by $j^*$. \vspace{-0.1em} Since $f_\eta$ is 1-Lipschitz, we must have $|j^*|\in\lceil j\rfloor$. Finally, we connect the points $Z_j:=(z_j, z_{j^*})$ sequentially. This curve is continuous and corresponds to a particular $\eta$-truncation $\tilde{f}_\eta\in \mathcal{L}_1$, which we refer to as the \emph{discretized $\eta$-truncation} of $f$. Intuitively, we have constructed an $(s+1)\times(s+1)$ grid and ``discretized'' function $f$ by finding its closest approximation in $\mathcal{L}_1$ whose curve only consists of edges and diagonals of the grid cells. By construction, 
\[
\max_{z\in[0,1]}|f_\eta(z)-\tilde{f}_\eta(z)|\leq \frac{\eta}{s}.
\]
Therefore, for any $p\in\cP:=\Delta_\cX$, the corresponding properties of $f_\eta$ and $\tilde{f}_\eta$ satisfy
\[
|f_\eta(p)-\tilde{f}_\eta(p)|\leq |\cX|\cdot \frac{\eta}{s},
\]
where we slightly abuse the notation 
and write $\tilde{f}_\eta(p):=\sum_x \tilde{f}_\eta(p_x)$. 
Note that $|j^*|\in\lceil j\rfloor$ for all $j\in\lceil s\rfloor$, and $\tilde{f}_\eta(z)=z_{s^*}$ for $z\geq \eta$. While there are infinitely many $\eta$-truncations, the cardinality of the discretized $\eta$-truncations of functions in $\mathcal{L}_1$ is at most 
\[
\prod_{j=0}^{s}(2j+1)=(s+1) \prod_{j=0}^{s-1}(2j+1)(2s-2j+1)\leq {(s+1)}^{2s+1}=e^{(2s+1)\log (s+1)}\leq e^{3s\log s}.
\]
Now we consider the task of estimating $\tilde{f}_\eta(p)$ from $X^n\sim p$.
By construction, the real function $\tilde{f}_\eta(z)$ is a constant for $z\ge \eta$. 
In addition, the function is Lipschitz and hence for an absolute constant $C$ and an
arbitrary interval $I:=[a, b]\subseteq [0,1]$, 
one can construct an explicit polynomial $g(z)$ of degree at most $d\in\mathbb N$,
satisfying \vspace{-0.5em}
\[
|g(z)-\tilde{f}_\eta(z)|\le C\frac{\sqrt{|b-a|(x-a)}}{d},\forall x\in I.
\]
Combining these facts with Theorem 5 in~\cite{hao2019unified} shows that
there exists an estimator $\hat f_\eta(X^n)$ that for all 
$p\in \Delta_\cX$ and $\varepsilon$ satisfying 
$n=\Omega(|\cX|/(\varepsilon^2\log |\cX|))$ and $\varepsilon >1/n$,
\[
\Pr_{X^n\sim p}\Paren{\Abs{\hat f_\eta(X^n)-\tilde f_\eta(p)}>\varepsilon}
\le
\exp(-\Theta(\varepsilon^2/(n^\lambda \eta) )),
\]
where $\lambda \in (0,1)$ is an 
absolute constant bounded away from $0$ 
whose value can be arbitrarily small. 

Then, by setting $\eta \le \mathcal{O}(\varepsilon^2/n^{1/2+\lambda})$,
where the asymptotic notation hides a sufficiently small absolute constant,
the right-hand side is at most $\exp(-4\sqrt n)$. Then, Theorem~\ref{thm1} in this paper and Theorem 3 in~\cite{acharya2017unified} imply that the 
PML distribution $\cP_{\varphi(X^n)}$ satisfies 
\[
\Pr_{X^n\sim p}\Paren{\Abs{\tilde f_\eta(\cP_{\varphi(X^n)})-\tilde f_\eta(p)}>2\varepsilon}
\le
\exp(-\sqrt n), \forall p\in \Delta_\cX. 
\]
Consider any $p\in\Delta_\cX$ and $X^{n}\sim p$ 
with a profile $\varphi:=\varphi(X^n)$. Consolidate 
the previous results, and apply the union bound 
and triangle inequality. With probability at least $1-\exp\Paren{3s\log s-\sqrt n}$, the PML plug-in estimator will satisfy
\begin{align*}
|f_\eta(p)-f_\eta(p_{\varphi})|
&\leq  |f_\eta(p)-\tilde{f}_\eta(p)|+|\tilde{f}_\eta(p)-\tilde{f}_\eta(p_{\varphi})|+ |\tilde{f}_\eta(p_{\varphi})-f_\eta(p_{\varphi})|
\leq 2|\cX|\cdot\frac{\eta}{s}+2\varepsilon,
\end{align*}
for \emph{all} functions $f$ in $\mathcal{L}_1$.

Next we consider the ``second part'' of a function $f\in\mathcal{L}_1$, namely,
\[
\bar{f}_{\eta}(z):=f(z)-f_{\eta}(z)=(f(z)-f(\eta)) \indic_{z> \eta}.
\]
Again, we can verify that $\bar{f}_{\gamma}\in \mathcal{L}_1$. To establish the corresponding guarantees, we make use of the following result. Since the profile probability is invariant to symbol permutation, for our purpose, we can assume that $p(y)\leq p(z)$ iff $\cP_{\varphi}(x)\leq \cP_{\varphi}(y)$, for all $x, y\in \cX$. Under this assumption, the following lemma relates~$\cP_{\varphi}$~to~$p$. 
\begin{Lemma}\label{lem:pmlcon}
For any $\eta\in(0,\mathcal{O}(1/\sqrt n))$, distribution $p\in\Delta_\cX$ and sample $X^n\sim p$ with profile $\varphi$,
\[
\Pr\Paren{\sum_x|\max\{\cP_{\varphi}(x), \eta\}-\max\{p(x), \eta\}|>
\Theta\Paren{\sqrt{\frac{1}{\eta n}}}}
\le \exp(-\sqrt n).
\]
\end{Lemma}
The proof of this lemma follows from 
1) the fact that empirical distribution satisfies such a guarantee with the probability bound being $\exp(-4\sqrt n)$, where we ignore labelings
and sort the empirical probabilities according to $p$; 
2) the Cauchy-Schwarz inequality applied to bound the expected error
 of the empirical estimator and McDiarmid's inequality 
 used to bound the error probability; 
3) a variant of Theorem~3 in~\cite{acharya2017unified} that addresses distribution estimation~\cite{das2012competitive}.

Hence, with probability at least $1-\exp(-\sqrt n)$,
\begin{align*}
|\bar{f}_\eta(p)-\bar{f}_\eta(\cP_{\varphi})|
&=|\sum_{x}\bar{f}_\eta(p(x))-\bar{f}_\eta(\cP_{\varphi}(x))|\\
&\leq 
\sum_x |\bar{f}_\eta(\max\{p(x), \eta\})- \bar{f}_\eta(\max\{\cP_{\varphi}(x), \eta\})|\\
&\le
\sum_x|\max\{\cP_{\varphi}(x), \eta\}-\max\{p(x), \eta\}|\\
&\le 
\Theta\Paren{\sqrt{\frac{1}{\eta n}}},
\end{align*}
for \emph{all} functions $f$ in $\mathcal{L}_1$.

Consolidate the previous results. By the triangle inequality and the union bound, with probability at least $1-\exp\Paren{3s\log s-\sqrt n}-\exp(-\sqrt n)$,
\begin{align*}
|f(p)-f(p_{\varphi})|
&\leq |f_\eta(p)-f_\eta(\cP_{\varphi})|+|\bar{f}_\eta(p)-\bar{f}_\eta(\cP_{\varphi})|\\
&\leq 2|\cX|\cdot\frac{\eta}{s}+2\varepsilon+\Theta\Paren{\sqrt{\frac{1}{\eta n}}}
,
\end{align*}
for \emph{all} functions $f$ in $\mathcal{L}_1$. Now we can conclude that $\ell_1^{\text{\tiny{<}}}\Paren{p,p_{\varphi}}$ is also at most the error bound on the right-hand side. The reason is straightforward: Since with high probability, the above guarantee holds for all functions in $\mathcal{L}_1$, it must also hold for the function that achieves the supremum in
\[
 \sup_{f\in\mathcal{L}_1} \Abs{f(p)- f(\cP_{\varphi})}=\ell_1^{\text{\tiny{<}}}\Paren{p,\cP_{\varphi}}.
\]
It remains to balance the error bounds on the estimation and deviation probability.
Recall that we assume $|\cX|\le \mathcal{O}(n\log n)$ since otherwise the theorem is trivial to prove. Set $s=\tilde\Theta(\sqrt n)$ such that $3s\log s<\sqrt n/2$.
Then, the confidence lower bound becomes $1-\exp(\sqrt n/2)-\exp(\sqrt n)$,
and the deviation bound reduces to $\tilde{\mathcal{O}}(\sqrt n\eta)+\Theta(\sqrt{1/(\eta n)})+2\varepsilon$. The previous derivations also require  
that $\eta \le \mathcal{O}(\varepsilon^2/n^{1/2+\lambda})$ and 
$\eta\in(0,\mathcal{O}(1/\sqrt n))$. Setting $\eta=\Theta(1/n^{3/4})$ 
yields the desired result. 

\end{proof}
\vfill
\pagebreak

\subsection{PML and Uniformity Testing}\label{sec:pml_uniform}
For a finite domain $\cX$, denote by $u_\cX$ the uniform distribution over $\cX$.  Given an error parameter $\varepsilon>0$ and a sample $X^n$ from an unknown distribution $p\in \Delta_\cX$, \emph{uniformity testing}~\cite{goldreich2011testing} aims to distinguish between the null hypothesis
\[
p=u_{\cX},
\]
and the alternative hypothesis
\[
\norm{p-u_\cX}_1>\varepsilon. 
\]
In the work of~\cite{hao2019broad}, 
it is shown that the following simple PML-based algorithm achieves the 
optimal $\Theta(\sqrt {|\cX|}/\varepsilon^2)$ 
sample complexity~\cite{paninski2008coincidence} for uniformity testing,
up to logarithmic factors of 
the alphabet size $|\cX|$.
Note that we instantiate the distribution collection $\cP$ as $\Delta_\cX$,
and use $0$ and $1$ to indicate whether $H_0$
or $H_1$ is accepted. 

\begin{figure}[ht]
\begin{center}
\boxed{
\begin{aligned}
&\text{\bf Input:}\ \ \text{parameters }|\cX|, \varepsilon, \text{and a sample }X^n\sim p \text{ with profile } \varphi.\\
&\ \text{\bf if } {\max}_x \mu_x(X^n)\geq 3\max\{1, n/|\cX|\}\log |\cX| \text{ \bf  then return } 1\text{;}\\
&\ \text{\bf elif } \norm{\cP_{\varphi}-u_\cX}_2\geq  3\varepsilon/ (4\sqrt{|\cX|}) \text{ \bf  then return }1\text{;}\\ 
&\ \text{\bf else}\text{ \bf return }0
\end{aligned}
}
\end{center}
\caption{Uniformity tester $T_{\text{\tiny PML}}$}
\label{fig:1}
\end{figure}

In this section, we present another intriguing connection between the PML estimator and the uniformity testing problem. For any profile $\varphi$ of length $n$, denote
\[
T(\varphi) =\frac{|\cX|(\sum_{\mu=1}^n \varphi_\mu \cdot \mu^2-n)}{n^2-n}.
\]
Then for any accuracy $\varepsilon>0$, the following uniformity tester~\cite{diakonikolas2016collision}
is sample optimal up to logarithmic factors. 
\begin{itemize}
\item If $T(\varphi(X^n))\ge 1+3\varepsilon^2/4$, return 1;
\item Else, return 0.
\end{itemize}
The following lemma connects the above algorithm to the PML. 

\begin{Lemma}~\cite{chan2015phase}
For any profile $\varphi:=\varphi(x^n)$ 
that corresponds to a non-constant sequence $x^n\in \cX^*$, 
\begin{itemize}
\item If $T(\varphi)> 1$, then $u_\cX$ is a local minimum of the PML optimization problem
\[
\cP_\cX 
=\max_{p\in \cP=\Delta_\cX} 
\Pr_{Y^n\sim p}(\varphi(Y^n) = \varphi);
\]
\item Else, $u_\cX$ is a local maximum.
\end{itemize}
\end{Lemma}

\vfill
\pagebreak

\section{Attributes of Profile Entropy and Dimension}
\vspace{0.25em}
Let $p\in \Delta_\cX$ be an arbitrary discrete distribution. 
Recall that in Section~\ref{sec:dim_and_entro}, 
we partition the unit interval into a sequence of ranges,
\[
I_j:=\left((j-1)^2\frac{\log n}{n},\ j^2 \frac{\log n}{n}\right], 1\le j\le \sqrt{\frac{n}{\log n}},
\]
denote by $p_{I_j}$ the number of probabilities $p_x$ belonging to $I_j$, 
and relate $E_n(p)$ to an induced shape-reflecting quantity,
\[
H^\mathcal{S}_n(p):= \sum_{j\ge 1} \min\Brace{p_{I_j}, j\cdot \log n},
\]
the sum of the effective number of probabilities lying within each range. 

The simple expression of $H^\mathcal{S}_n(p)$ shows that 
it characterizes the 
variability of ranges the actual probabilities spread over.
As Theorem~\ref{thm:hs_app_en} shows, 
$H^\mathcal{S}_n(p)$ closely approximates $E_n(p)$, 
the value around which $\mathcal D_n\sim p$ concentrates 
(Theorem~\ref{thm:exp_conc}) and $H(\Phi^n)$ lies 
(Thoerem~\ref{thm:entro_eql_dim}). 
In this section, we use $H^\mathcal{S}_n(p)$ as a proxy for both 
$H(\Phi^n)$ and $\mathcal D_n$,  
and study its attributes and values. 

To further our understanding of profile entropy and dimension, 
we investigate the analytical attributes of $H^\mathcal{S}_n(p)$
concerning monotonicity and Lipschitzness. 
Then, we present tight upper and lower bounds on the value of 
$H^\mathcal{S}_n(p)$ for 
a variety of distribution families. 
\vspace{0.5em}
\subsection{Monotonicity}
Among the many attributes that ${H^\mathcal{S}_n(p)}$ possesses,
monotonicity is perhaps most intuitive. 
One may expect a larger value of ${H^\mathcal{S}_n(p)}$ as the sample size $n$ increases, 
since additional observations reveal more information on the variability of probabilities. 
Below we confirm this intuition.
\setcounter{Theorem}{8} 
\begin{Theorem}\label{thm:pro_monot}
For any $n\ge m\gg1$ and $p\in \Delta_\cX$, 
\[
{H^\mathcal{S}_n(p)}\ge H^\mathcal{S}_m(p). 
\]
\end{Theorem}
Besides the above result that lowerly bounds ${H^\mathcal{S}_n(p)}$ with 
$H^\mathcal{S}_m(p)$ for $m\le n$,
a more desirable result is to upperly bound ${H^\mathcal{S}_n(p)}$ 
with a function of ${H^\mathcal{S}_m(p)}$.
 Such a result will enable us to draw a sample of size $m\le n$, 
obtain an estimate of ${H^\mathcal{S}_m(p)}$ from $\mathcal D_{m}$, 
and use it to bound the value of ${H^\mathcal{S}_n(p)}$ and thus of $\mathcal D_{n}$ 
for a much larger sample size $n$. 

With such an estimate, 
we can perform numerous tasks such as \emph{predicting}
the performance of PML in estimating symmetric properties 
when more observations are available,
and the space needed for storing a longer sample profile. 
These applications are closely related to the 
recent works on \emph{learnability estimation}~\cite{kong2018estimating, kong2019sublinear}.

The next theorem provides a simple and tight upper bound 
on ${H^\mathcal{S}_n(p)}$ in terms of ${H^\mathcal{S}_m(p)}$. 
\begin{Theorem}
For any $n\ge m\gg1$ and $p\in \Delta_\cX$, 
\[
{H^\mathcal{S}_n(p)}\le
\sqrt{\frac{n\log n}{m\log m}}\cdot {H^\mathcal{S}_m(p)}. 
\]
\end{Theorem}
\paragraph{Implications} 
Before proceeding to the proof, we first present two simple implications.
\begin{enumerate}
\item  If for $m=\Omega(n^{0.01})$, we have ${H^{\mathcal S}_m(p)}\ll \sqrt{m}$, then ${H^{\mathcal S}_n(p)}\ll \sqrt{n}$.

\item For any two integers $m\le n$ and distribution $p$, 
\[
\frac{{H^{\mathcal S}_m(p)}}{\sqrt{m\log m}} \ge \frac{{H^{\mathcal S}_n(p)}}{\sqrt{n\log n}}. 
\]
In other words, the sequence $A_m := {{H^{\mathcal S}_m(p)}}/{\sqrt{m\log m}}$, $m\le n$, is monotonically decreasing and converges to $A_n$. As we increase the value of $m$, $(\sqrt{n\log n}\cdot A_m)$, which can be viewed as our estimate of ${H^{\mathcal S}_n(p)}$, is getting more and more accurate. 
For the purpose of adaptive estimation, if $n = 2^t$, we can choose the sequence $m = 2^0, 2^1, \ldots, 2^t$.
\end{enumerate}
\begin{proof}
For clarity, we denote by $p(m,j)$ the value of $p_{I_j}$ 
corresponding to ${H^{\mathcal S}_m(p)}$, and $p(n, j)$ the value of $p_{I_j}$ 
corresponding to ${H^{\mathcal S}_n(p)}$. Furthermore, denote 
$r:=\sqrt{(n/m)((\log m)/\log n)}$, which we assume is an integer.
Then by the definition of $H^{\mathcal S}_{\boldsymbol\cdot}$,
\begin{align*}
r{H^{\mathcal S}_m(p)}
&= r\sum_{j\ge1}\min\Brace{p(m,j), j\cdot \log m}\\
&= \sum_{j\ge1}\min\Brace{r\cdot\!\!\sum_{i=rj-r+1}^{rj} p(n,i),\ rj\cdot \log m}\\
&\ge \sum_{j\ge1}\sum_{t=0}^{r-1}\min\Brace{\sum_{i=rj-r+1}^{rj} p(n,i),\ (rj-t)\cdot \log m}\\
&\ge \sum_{j\ge1}\sum_{t=0}^{r-1}\min\Brace{p(n,rj-t),\ (rj-t)\cdot \log m}\\
&=\sum_{i\ge1} \min\Brace{p(n,i),\ i\cdot \log m}\\
&\ge \frac{\log m}{\log n}\cdot {H^{\mathcal S}_n(p)}.
\end{align*}
The lower-bound part basically follows by reversing the above inequalities.
\begin{align*}
{H^{\mathcal S}_n(p)}
&=\sum_{i\ge1} \min\Brace{p(n,i),\ i\cdot \log n}\\
&= \sum_{j\ge1}\sum_{t=0}^{r-1}\min\Brace{p(n,rj-t),\ (rj-t)\cdot \log n}\\
&\ge \sum_{j\ge1}\sum_{t=0}^{r-1}\min\Brace{p(n,rj-t),\ (rj-r+1)\cdot \log n}\\
&\ge \sum_{j\ge1}\min\Brace{\sum_{t=0}^{r-1} p(n,rj-t),\ (rj-r+1)\cdot \log n}\\
&= \sum_{j\ge1}\min\Brace{p(m,j),(rj-r+1)\cdot \log m}\\
&\ge {H^{\mathcal S}_m(p)}. 
\end{align*}
This completes the proof of the theorem.
\end{proof}

\subsection{Lipschitzness}

Viewing ${H^\mathcal{S}_n(p)}$ as a distribution property, we establish its Lipschitzness 
with respect to a weighted Hamming distance and the $\ell_1$ distance. 
Given two distributions $p, q\in \Delta_\cX$, 
the vanilla \emph{Hamming distance} is denoted by 
\[
h(p,q) := \sum_{x\in \cX} \indic_{p_x\not= q_x}. 
\]
The distance is suitable for being a statistical distance since there may 
be many symbols at which the two distributions differ, yet those symbols
account for only a negligible total probability and has little effects on 
many induced statistics.
To address this, we propose a \emph{weighted Hamming distance} 
\[
h_{_{\mathcal W}}\!(p,q) 
:= 
\sum_{x\in \cX} \max\{p_x, q_x\}\boldsymbol\cdot\indic_{p_x\not= q_x}. 
\]
The next result measures the Lipschitzness of 
$H^\mathcal{S}_n$ under~$h_{_{\mathcal W}}$.
\begin{Theorem}
For any integer $n$, and distributions $p$ and $q$,
if $h_{_{\mathcal W}}\!(p,q)\le \ve$ for some $\ve\ge 1/n$, 
\[
\Abs{{H^\mathcal{S}_n(p)}-{H^\mathcal{S}_n(q)}}\le\tilde{\mathcal{O}}(\sqrt{\ve\! n}). 
\]
\end{Theorem}
\begin{proof}
Recall that the quantity of interest is
\[
H^\mathcal{S}_n(p):= \sum_{j\ge 1} \min\Brace{p_{I_j}, j\cdot \log n}.
\]\par\vspace{-0.75em}
Given the bound of $h_{_{\mathcal W}}\!(p,q)\le \ve$, 
we denote by $\mathcal Y\subseteq \cX$ the collection of symbols $x$ at which
$p_x\not=q_x$.  
By definition, we have both $\sum_{x\in \mathcal Y} p_x\le \ve$
 and $\sum_{x\in \mathcal Y} q_x\le \ve$. 
 Below, we show that these symbols 
 modify the value of $H^\mathcal{S}_n(p)$ 
 by at most $\tilde{\mathcal{O}}(\sqrt{\ve\! n})$. 
 By symmetry, the same claim also holds for the distribution $q$. 
 Combined, these two claims yields the desired result. 
 
 First, we consider $x\in \mathcal Y$ satisfying 
 $p_x=0$ or $p_x\in I_1=(0, (\log n)/n]$. 
 Such a symbol either does not contribute the value of 
 $H^\mathcal{S}_n(p)$, or affects only the value 
 of the first term $\min\Brace{p_{I_1}, \log n}$, 
 which is at most $\log n$. Hence the claim holds for this case.  
 
 Next, consider symbols $x\in \mathcal Y$ satisfying 
$p_x\in I_j=((j-1)^2\frac{\log n}{n},\ j^2 \frac{\log n}{n}]$ for some $j\ge 2$
and denote the collection of them by $\mathcal Z\subseteq \mathcal Y$. 
By the above assumption, we have $\sum_{x\in \mathcal Z} p_x\le \ve$.
To maximize their impact on $H^\mathcal{S}_n(p)$ under this constraint,
we should set their values to be 
\[
p_j:=(j-1)^2\frac{\log n}{n},\ j=2,\ldots J,
\]
for some $J$ to be determined, where each $p_j$ repeats exactly $j\log n$ times. Then, the symbols in $\mathcal Z$ 
contributes at most $\sum_{j=2}^J j\log n=(\log n)(J-1)(J+2)/2$
to $H^\mathcal{S}_n(p)$, and the above constraint on the total probability mass bounds transforms to 
\[
\ve
\ge 
\sum_{x\in \mathcal Z} p_x
\ge
\sum_{j=2}^J (j\log n)\cdot (j-1)^2\frac{\log n}{n}
\ge
\frac{(\log n)^2}{12n} J(J^2-1)(-2+3J).
\]
Therefore in this case, the contribution 
is again $\tilde{\mathcal{O}}(\sqrt{\ve\!n})$,
which completes the proof. \qedhere
\end{proof}

Replacing $\max\{p_x, q_x\}$ with $|p_x-q_x|$ results in a 
common similarity measure -- the $\ell_1$ distance. The next theorem is an analog to the above under this classical distance. 
\begin{Theorem}
For any integer $n$, and distributions $p$ and $q$,
if $\ell_1(p,q)\le \ve$ for some $\ve\ge 0$, 
\[
\Abs{{H^\mathcal{S}_n(p)}-c{H^\mathcal{S}_n(q)}}\le\mathcal{O}((\ve\! n)^{2/3}),
\]
where $c$ is a constant in $[1/3, 3]$. 
Note that the inequality is significant iff $\varepsilon\le \tilde\Theta(1/n^{1/4})$, 
since the value of ${H^\mathcal{S}_n(p)}$ is at most $\mathcal{O}(\sqrt{n\log n})$ for all $p$.
\end{Theorem}
By symmetry, it suffices to prove the following lemma. 

\begin{Lemma}
For any integer $n$, and distributions $p$ and $q$,
if $\ell_1(p,q)\le \ve$ for some $\ve\ge 0$, 
\[
{H^{\mathcal S}_n(p)}\le 3H^{\mathcal S}_n(q)+\mathcal{O}((\varepsilon n)^{2/3}). 
\]
\end{Lemma}
\begin{proof}
Consider the task of modifying $p$ by at most $\varepsilon$ and maximizing the increase in ${H^{\mathcal S}_n(p)}$. 
For each $j$ and each probability $p_x\in j$, denote by $p_x'$ the modified value. 
Depending on the location of~$p_x'$, there are three types of possible modifications as illustrated below. 
\begin{itemize}
\item
For the first type, we still have $p_x'\in I_j$. This does not change the value of $p_{I_j}$ 
and hence does not increase ${H^{\mathcal S}_n(p)}$. 

\item 
For the second type, we have $p_x'\in I_{j-1}$ or $p_x'\in I_{j+1}$. If $p_{I_j}\le j\cdot \log n$,
this will decrease the value of $\min\{p_{I_j}, j\cdot\log n\}$ by $1$ and increase the value of 
$\min\{p_{I_{j-1}}, (j-1)\cdot\log n\}$ or $\min\{p_{I_{j+1}}, (j+1)\cdot\log n\}$ by at most one. 
Hence in this case, the value of ${H^{\mathcal S}_n(p)}$ can only decrease. 
If $p_{I_j}> j\cdot \log n$, then $\min\{p_{I_j}, j\cdot\log n\} = j\cdot\log n$. For a particular $j$, 
all such modifications can increase the value of ${H^{\mathcal S}_n(p)}$ by at most $(j-1)\log n+(j+1)\log n = 2j\log n$,
which is twice the value of $\min\{p_{I_j}, j\cdot\log n\}$. Hence, all such modifications, when combined, increase the value of 
${H^{\mathcal S}_n(p)}$ by at most $2{H^{\mathcal S}_n(p)}$. 

\item 
For the third type, we have $p_x'\in I_{i}$ and $|i-j|\ge2$. If $i<j$, we require 
a probability mass of at least $((j-1)^2\log n-i^2\log n)/n\ge(i\log n)/n$, where $j\ge 3$. 
If $i>j$, we require a probability mass of at least $((i-1)^2\log n-j^2\log n)/n\ge(i\log n)/n$. 
The number of such modifications that could lead to an increase in the value of ${H^{\mathcal S}_n(p)}$ is at most 
$i\log n$. For each $i$, let $c_i$ denote the number of such modifications that 
will lead to an increase of ${H^{\mathcal S}_n(p)}$. Then, the total increase is $\sum_i c_i$, 
each $c_i$ is at most $i\log n$,  
and the total required probability mass required is at least $\sum_i c_i\cdot (i\log n)/n \le \varepsilon$. 

Let $\{c_i\}$ be the optimal solution that maximizes $\sum_i c_i$.
Assume that there are two indices $i<j$ satisfying $c_i<i\log n$ and $c_j>0$. 
Then, if we replace $c_i$ and $c_j$ by $c_i+1$ and $c_j-1$, respectively, $\sum_{i} c_i$ will not change and $\sum_i c_i\cdot (i\log n)/n$ will decrease. Hence, we can assume that there exists $i'$ satisfying $c_i=i\log n,\forall i< i'$ and $c_i=0,\forall i>i'$. In addition, assuming $\varepsilon n\ge \log n$ implies that $i'\ge 2$. Hence, we have 
$
\sum_i c_i\le (\log n){i'(i'+1)}/{2}
$
and 
\[
\sum_i c_i\le 3.5\cdot \Paren{\frac{n\varepsilon}{\sqrt{\log n}}}^{2/3}. \qedhere
\]
\end{itemize}
\end{proof}

\subsection*{Profile Entropy for Structured Families}\label{sec:pro_struct}
Following the study of attributes of profile entropy, 
we derive below nearly tight bounds on the ${H^\mathcal{S}_n(p)}$ 
values of three important structured families, 
 log-concave, power-law, and histogram.  
These bounds tighten up and significantly improve those
 in~\cite{hao2019doubly}, and show the ability of profile 
entropy in charactering natural shape constraints.

For the remaining sections, 
we follow the convention and specify structured distributions over $\cX=\mathbb Z$. 

\subsection{Log-Concave Distributions}

We say a discrete distribution $p\in\Delta_{\mathbb{Z}}$ 
is \emph{log-concave} if $p$ has a contiguous support over $\mathbb Z$ and 
the inequality $p_x^2\ge p_{x-1} p_{x+1}$ holds for all symbols $x\in \mathbb Z$.

The log-concave family encompasses a broad range of discrete distributions,
such as Poisson, hyper-Poisson, Poisson binomial, 
binomial, negative binomial, geometric, and hyper-geometric,
with wide applications 
to numerous research areas, including 
statistics~\cite{saumard2014log}, computer science~\cite{lovasz2007geometry}, 
economics~\cite{an1997log}, algebra, and geometry~\cite{stanley1989log}.

The next result upperly bounds the profile entropy of log-concave families, 
and is \emph{tight} up to logarithmic factors of~$n$.

\begin{Theorem}\label{thm:logconc_bound}
For any $n\in\mathbb Z$ and distribution $p\in \Delta_{\mathbb Z}$, 
if $p$ is log-concave and has a variance of~$\sigma^2$, \vspace{-0.25em}
\[
{H^\mathcal{S}_n(p)}\le \mathcal{O}(\log n)\Paren{1+\min\Brace{\sigma, \frac{n}{\sigma}}}. 
\]
\end{Theorem}
\begin{proof}
The $\mathcal{O}(\log n)(1+\sigma)$ upper bound is established in~\cite{hao2019doubly} using some concentration attributes
 of the log-concave distributions.  

For the other component, we can assume that $\sigma\ge \sqrt n$
and $n$ is larger than some absolute constant. 
Then by~\cite{diakonikolas2016efficient}, the maximum probability 
$p_{\text{max}}$ of $p$ belongs to $[1/(8\sigma), 1/\sigma]$. 
Hence, the last index $J$ for which $p_{I_{J}}\not= 0$ satisfies
\[
(J-1)^2\frac{\log n}{n} \le \frac1\sigma
\iff
J\le \sqrt{\frac{n}{\sigma\log n}}+1. \vspace{-0.25em}
\]
Hence, we have \vspace{-0.25em}
\[
H^\mathcal{S}_n(p)
= 
\sum_{j\ge 1} \min\Brace{p_{I_j}, j\cdot \log n}
\le 
\log n+\sum_{j= 1}^{\sqrt{{n}/{(\sigma\log n)}}+1} j\cdot \log n
\le 
\mathcal{O}(\log n)\Paren{1+\frac{n}{\sigma}}
.\qedhere
\]
\end{proof}\vspace{-0.25em}
This upper bound is uniformly better than the 
$\min\{\sigma, (n^2/\sigma)^{1/3}\}$ bound in~\cite{hao2019doubly}. 
Theorem~\ref{thm:Gaussian} further 
shows that it is optimal up to logarithmic factors of $n$. 

A similar bound holds for $t$-mixtures of log-concave distributions. 
More concretely, 
\begin{Theorem}\label{thm:mix_logcon}
For any integer $n$ and distribution $p\in \Delta_{\mathbb Z}$, 
if $p$ is a $t$-mixture of log-concave distributions each has a 
variance of $\sigma^2_i$, where $i=1,\ldots, t$, \vspace{-0.25em}
\[
{H^\mathcal{S}_n(p)}
\le \mathcal{O}(\log n)
\Paren{1+\min\Brace{\sum_i\sigma_i, \max_i\Brace{\frac{n}{\sigma_i}}}}. 
\]
\end{Theorem}

\subsection{Discretization of Continuous Distributions}\label{sec:discrete}

The introduction about log-concave families 
covers numerous classical discrete distributions, 
yet leaves many more continuous ones untouched~\cite{bagnoli2005log}. 
Below, we present a discretization procedure that 
preserves distribution shapes such as monotonicity, 
modality, and log-concavity. 
Applying this procedure to the Gaussian distribution $\mathcal{N}(\mu,\sigma^2)$ 
further shows the optimality of Theorem~\ref{thm:logconc_bound}.

Let $X$ be a continuous random variable with density function $f(x)$.
For any $x\in \mathbb R$, denote by $\D{x}$ the closest integer $z$ such that $x\in(z-1/2, z+1/2]$.
The distribution of $\D{X}$ is over $\mathbb Z$ and satisfies
\[
p(z):=\int_{z-\frac12}^{z+\frac12} f(x) dx,\ \forall z\in\mathbb Z.
\]
We refer to this random variable $\D{X}$ as the discretized version of $X$. 

\paragraph{Shape Preservation}
By definition, one can readily verify that the above transformation preserves 
several important shape characteristics of distributions,
such as monotonicity, modality, and $k$-modality 
(possibly yields a smaller $k$). 
The following theorem covers log-concavity. 
\begin{Theorem}\label{thm:logconc_pre}
For any continuous random variable $X$ over $\mathbb R$ with
a log-concave density $f$, the distribution $p\in \Delta_{\mathbb Z}$ 
associated with $\D{X}$ is also log-concave. 
\end{Theorem}

To show this, we need the following basic lemma about concave functions. 
\begin{Lemma}
For real numbers $x_1,x_2,y_1,$ and $y_2$ satisfying $x_1\le x_2$, $y_1\le y_2$, $x_1< y_1$, and $x_2<y_2$,
\[
\frac{f(y_1)-f(x_1)}{y_1-x_1}\ge \frac{f(y_2)-f(x_2)}{y_2-x_2}.
\]
\end{Lemma}
Following the lemma, for $x,y\in \mathbb R$ such that $|x-y|\le 1$, 
and any function $f$ that is log-concave,
\[
\log f(x+1)-\log f(x)\le \log f(y)-\log f(y-1) \iff f(x+1)f(y-1)\le f(x)f(y).
\]
\begin{proof}
By definition $p$ is log-concave if $p$ has a consecutive support and $p(z)^2\ge p(z+1)p(z-1), \forall z$. For $\D{X}$, the first condition is satisfied since $X$ is has a continuous support on $\mathbb R$, 
and $p(z)$ is positive as long as $f(x)>0$ for a non-empty sub-interval of $(z-1/2,z+1/2]$.
 
Below we show that $p$ satisfies the second condition. Specifically, for any $z\in\mathbb Z$,
\begin{align*}
p(z-1) p(z+1)
&=
\Paren{\int_{z-\frac32}^{z-\frac12} f(x) dx} \Paren{\int_{z+\frac12}^{z+\frac32} f(x) dx}\\
&=
\Paren{\int_{z-\frac12}^{z+\frac12} f(x-1) dx} \Paren{\int_{z-\frac12}^{z+\frac12} f(x+1) dx}\\
&=
\int_{z-\frac12}^{z+\frac12}\int_{z-\frac12}^{z+\frac12} f(x-1)  f(y+1) dx dy\\
&\le
\int_{z-\frac12}^{z+\frac12}\int_{z-\frac12}^{z+\frac12} f(x)  f(y) dx dy\\
&=
\Paren{\int_{z-\frac12}^{z+\frac12} f(x) dx}^2\\
&=
p(z)^2,
\end{align*}
where the inequality follows by the above lemma and its implication.
\end{proof}

\paragraph{Moment  preservation} 
Denote by $p$ the distribution of $\D{X}$ for $X\sim f$. 
Let $\mu$ and $\sigma^2$ be the mean and variance of 
density $f$, given that they exist.  
The theorem below shows that distribution $p$ has,
within small additive absolute constants, a mean of $\mu$ 
and variance of $\Theta(\sigma^2)$. 

\begin{Theorem}\label{thm:moment_pre}
Under the aforementioned conditions, 
the mean of $\D{X}$ satisfies \vspace{-0.75em}
\[
\EE\D{X} = \mu\pm \frac12, \vspace{-0.5em}
\]
and the variance of $\D{X}$ satisfies 
\[
\sigma^2/2-1\le \EE(\D{X}-\EE \D{X})^2 \le 2\sigma^2+1. \vspace{-0.25em}
\]
\end{Theorem}

\begin{proof}
First consider the mean value of $\D{X}$ for $X\sim f$. We have 
\[
\EE\D{X} = \EE[\D{X}-X]+\EE[X] = \mu\pm \frac12.
\]
Next consider the variance of $\D{X}$. Applying the inequality $(a+b)^2\le 2(a^2+b^2)$ yields
\begin{align*}
\EE(\D{X}-\EE \D{X})^2
&=
\int_{-\infty}^{\infty} \Paren{\D{x}-\EE \D{X}}^2 \cdot f(x) dx\\
&=
\int_{-\infty}^{\infty} \Paren{\D{x}-x-(\D{X}-\EE X)+ x- \EE X}^2 \cdot f(x) dx\\
&\le 
2\int_{-\infty}^{\infty} \Paren{\Paren{\D{x}-x-(\D{X}-\EE X)}^2+\Paren{x- \EE X}^2} \cdot f(x) dx\\
&\le 
2\int_{-\infty}^{\infty} \Paren{1+\Paren{x- \EE X}^2} \cdot f(x) dx\\
&=
2+2\EE(X-\EE X)^2\\
&=
2(1+\sigma^2).
\end{align*}
By the symmetry in the above reasoning, we also have
\[
\sigma^2= \EE(X-\EE X)^2 \le 2(1+\EE(\D{X}-\EE \D{X})^2).
\]
Consolidating these inequalities shows that 
\[
\sigma^2/2-1\le \EE(\D{X}-\EE \D{X})^2 \le 2\sigma^2+1. \qedhere
\]
\end{proof}

\subsection{Optimality of Theorem~\ref{thm:logconc_bound}}\label{sup:2.5}
By the above formula, the discretized Gaussian $\D{\mathcal{N}}\!(\mu,\sigma^2)$ has a distribution in the form of \vspace{-0.25em}
\[
p_{\!_G}(z)\!:=\frac{1}{\sqrt{2\pi}\sigma} \!\int_{z-\frac12}^{z+\frac12}\!\exp\Paren{-\frac{(x-\mu)^2}{2\sigma^2}} dx,\ \forall z\in\mathbb Z.
\]
Consolidating Theorem~\ref{thm:logconc_pre} and~\ref{thm:moment_pre} shows that $p_{\!_G}\!$ 
is a log-concave distribution with a variance of $\Theta(\sigma^2)\pm 1$.
Consequently, Theorem~\ref{thm:logconc_bound} yields the following upper bound:
\[
H^\mathcal{S}_n(p_{\!_G})\le \mathcal{O}(\log n)\Paren{1+\min\Brace{\sigma, \frac{n}{\sigma}}}. \vspace{-0.25em}
\]
In the following, we show that
\begin{Theorem}\label{thm:Gaussian}
Under the aforementioned conditions,
\[
H^\mathcal{S}_n(p_{\!_G})\ge \mathcal{O}(\log n)^{-1}\!\Paren{1+\min\Brace{\sigma, \frac{n}{\sigma}}}. 
\]
\end{Theorem}
The optimality of Theorem~\ref{thm:logconc_bound} follows by these inequalities. 

\begin{proof}
At it is clear from the context, we will write $p$ instead of $p_{\!_G}$. Recall that
 \[
 {H^{\mathcal S}_n(p)}= \sum_{j\ge 1}\min\Brace{p_{I_j}, j\cdot \log n},
 \]
where $p_{I_j}$ denotes the number of probabilities 
belonging to $I_j=((j-1)^2, j^2]\cdot (\log n)/n$.
Considering part of the distribution can only reduce the value of ${H^{\mathcal S}_n(p)}$.
Hence, we focus on the symbols in the range $(\mu+1, \infty)\cap \mathbb Z$, over which the probability mass function $p(z)$ is monotone. 

We will further assume that $n/\log n\gg \sigma\gg \log n$, 
since otherwise the right-hand side of the inequality reduces 
to $\mathcal{O}(1)$, and the result follows by the fact that 
$ {H^{\mathcal S}_n(p)}\ge 1$ for all $n$ and $p$.  

In addition, we focus on $j\gg 1$ in the following argument, as the 
contribution to ${H^{\mathcal S}_n(p)}$ from these indices is 
no more than the total ${H^{\mathcal S}_n(p)}$.

Given these assumptions, we have 
\begin{align*}
p(z)\in I_j
&\iff 
\frac{1}{\sqrt{2\pi}\sigma} \exp\Paren{-\frac{(z \pm 1/2-\mu)^2}{2\sigma^2}}
\in \left((j-1)^2\frac{\log n}{n}, j^2\frac{\log n}{n}\right]\\
&\iff
z \pm 1/2-\mu\in \sqrt{2}\sigma \left[ \sqrt{c(\sigma, n)-2\log j}, \sqrt{c(\sigma, n)-2\log (j-1)}\right),
\end{align*}
where $c(\sigma, n):=\log\Paren{n/(\sqrt{2\pi}\sigma\log n)}$
and the interval is well-defined iff
\begin{align*}
c(\sigma, n)
\ge 2\log j
&\iff
{\frac{n}{\sqrt{2\pi}\sigma\log n}}
\ge j^2\\
&\iff
\sqrt{\frac{n}{\sqrt{2\pi}\sigma\log n}}
\ge j\\
&\ \Longleftarrow
\sqrt{\frac{n}{\sigma\log n}}
\ge 2 j.
\end{align*}

For clarity, we divide our analysis into two cases:
 $\sqrt n\ge\sigma\gg \log n$ and $n/\log n\gg\sigma> \sqrt{n}$. 
 
For the first case and $j\le \sqrt{\sigma/\log n}/2\le \sqrt{{n}/{(\sigma\log n)}}/2$, 
the length $L_j$ of the above interval, 
which equals to $p_{I_j}$ up to an additive slack of $2$, satisfies
\begin{align*}
\frac {L_j}{\sqrt{2}\sigma}
&=
\sqrt{c(\sigma, n)-2\log (j-1)}-\sqrt{c(\sigma, n)-2\log j}\\
&=
\frac{2\log (j/(j-1))}{(c(\sigma, n)-2\log (j-1))+(c(\sigma, n)-2\log j)}\\
&= 
\frac{\log (j/(j-1))}{\log\Paren{n/(\sqrt{2\pi}j(j-1)\sigma\log n)}}\\
&=
\Omega\Paren{\frac{1}{\log n}\log \Paren{1+\frac{1}{j-1}}}\\
&=
\Omega\Paren{\frac{1}{j\log n}}.
\end{align*}
Therefore, we have $L_j=\Omega(\sigma/(j\log n))$.
Since $\sigma\gg \log n$ ensures $L_j\ge 3$ and $j\le \sqrt{\sigma/\log n}/2$ is equivalent to $\sigma\ge 4j^2\log n$, 
the lower bound on $L_j$ transforms into $p_{I_j}\ge \Omega(j)$. 
Hence in this case, ${H^{\mathcal S}_n(p)}$ admits the following bound \vspace{-0.5em}
\[
{H^{\mathcal S}_n(p)}
= \sum_{j\ge 1}\min\Brace{p_{I_j}, j\cdot \log n}
\ge \sum_{j=\mathcal{O}(1)}^{\sqrt{\sigma/\log n}/2}\Omega(j)
=\Omega\Paren{\frac{\sigma}{\log n}}. 
\]
In the $n/\log n\gg \sigma> \sqrt{n}$ case, we have 
$\sqrt{\sigma/\log n}>\sqrt{n/(\sigma\log n)}$. 
Repeating the previous reasoning for $j\le \sqrt{{n}/{(\sigma\log n)}}/2$, 
we again obtain $L_j=\Omega\Paren{{\sigma}/{(j\log n)}}$ 
and $p_{I_j}\ge \Omega(j)$. 

Therefore, \vspace{-0.5em}
\[
{H^{\mathcal S}_n(p)}
= \sum_{j\ge 1}\min\Brace{p_{I_j}, j\cdot \log n}
\ge \sum_{j=\mathcal{O}(1)}^{\sqrt{n/(\sigma\log n)}/2}\Omega(j)
=\Omega\Paren{\frac{n}{\sigma\log n}}. 
\]
Finally, note that in the first case, $\min\{\sigma, n/\sigma\}=\sigma$,
while in the second, $\min\{\sigma, n/\sigma\}=n/\sigma$. 

Consolidating these results yields the desired lower bound
\[
{H^{\mathcal S}_n(p)}\cdot \mathcal{O}(\log n)\ge 1+\min\Brace{\sigma, \frac{n}{\sigma}}. 
\qedhere
\]
\end{proof}

\subsection{Power-Law Distributions}
We say that a discrete distribution $p\in\Delta_{\mathbb{Z}}$ 
is a \emph{power-law with power $\alpha>0$} 
if $p$ has a support 
of $[k]:=\{1,\ldots, k\}$ for some $k\in \mathbb Z^+\cup \{\infty\}$ 
and $p_x\propto x^{-\alpha}$ for all $x\in [k]$.

Power-law is a ubiquitous structure appearing in many situations of scientific interest,
ranging from natural phenomena such as the initial mass function of stars~\cite{kroupa2001variation},
species and genera~\cite{humphries2010environmental}, rainfall~\cite{machado1993structural}, 
population dynamics~\cite{taylor1961aggregation},
and brain surface electric potential~\cite{miller2009power}, 
to man-made circumstances such as 
the word frequencies in a text~\cite{baayen2002word}, income rankings~\cite{druagulescu2001exponential}, company sizes~\cite{axtell2001zipf}, 
and internet topology~\cite{faloutsos1999power}. 

Unlike log-concave distributions that concentrate around their mean values, power-laws are known to possess ``long-tails'' and always log-convex. Hence, one may expect the profile entropy of power-law distributions to behave differently from that of log-concave ones. 
The next theorem justifies this intuition and provides tight upper bounds.

\begin{Theorem}
For a power-law distribution $p\in\Delta_{[k]}$ with power $\alpha$, 
we have 
\[
{H^\mathcal{S}_n(p)}\le 7\log n+e^2\cdot \min\{k, \mathcal U_{n}^k(\alpha)\}, 
\]
where 
\[
\mathcal U_{n}^k(\alpha)
\!:= \!
\begin{cases}
n^{\frac1{1+\alpha}}&\!\!\!\! \text{if } \alpha\ge 1\!+\!\frac{1}{\log k};\\
\Paren{\frac{n}{\log n}}^{\frac1{1+\alpha}} &\!\!\!\! \text{if } 1\le \alpha <1\!+\!\frac{1}{\log k};\\ 
\sqrt n\Paren{\frac{k}{\sqrt n}\wedge \Paren{\frac{\sqrt n}{k}}^{\!\frac{1-\alpha}{1+\alpha}}} &\!\!\!\! \text{if } 0\le\alpha<1.
\end{cases}
\]
\end{Theorem}
The above upper bound fully characterizes the profile entropy of power-laws 
and surpasses the basic $\{k, \sqrt {n\log n}\}$ bound for both 
$k\gg\!\sqrt n$ and $k\ll\!\sqrt n$. 
In comparison, \cite{hao2019doubly} yields a $\mathcal{O}(n^{\min\{1/(1+\alpha), 1/2\}})$ 
upper bound, 
which improves over $\sqrt {n\log n}$ only for $\alpha>1$
and is worse than that above for all $\alpha<1\!+\!{1}/{\log k}$.

\begin{proof}
For the ease of exposition, write the probability of symbol $i$
assigned by distribution $p$ as $p_i:=c_\alpha^{-1}\cdot i^{-\alpha}$, 
where $c_\alpha$ is a normalizing constant (implicitly depends on $k$) and $k$ can be infinite. 
Recall that the quantity of interest is
 \[
 {H^{\mathcal S}_n(p)}= \sum_{j\ge 1}\min\Brace{p_{I_j}, j\cdot \log n}.
 \]
First consider $p_{I_j}$ for a sufficiently large $j$ and note that
\begin{align*}
p_i\in I_j
&\iff 
\frac{1}{c_\alpha i^\alpha}
\in \left((j-1)^2\frac{\log n}{n}, j^2\frac{\log n}{n}\right]\\
&\iff
i
\in 
I_j':= \left[ \Paren{j^2c(\alpha, n)}^{-\frac{1}{\alpha}},
\Paren{(j-1)^2c(\alpha, n)}^{-\frac{1}{\alpha}}
\right),
\end{align*}
where $c(\alpha, n):= (c_\alpha\log n)/{n}$. 
\vspace{-0.15em} Observe that the length $L_j$ of interval $I_j'$, which differs from the 
value of $p_{I_j}$ by at most $2$, 
is proportional to $(j-1)^{-2/\alpha}-j^{-2/\alpha}$, 
and hence is a decreasing function of $j$. 
Furthermore, each term $\min\{p_{I_j}, j\cdot \log n\}\approx \min\{L_j, j\cdot \log n\}$ is basically the minimum between 
this decreasing function and $j\log n$, an increasing function of $j$. 
This naturally calls for determining the value of $j$ at which the two functions are equal. 
Concretely,
\begin{align*}
\Paren{(j-1)^2c(\alpha, n)}^{-\frac{1}{\alpha}}-
\Paren{j^2c(\alpha, n)}^{-\frac{1}{\alpha}}
=j\log n
&\Longrightarrow
j\ge\Paren{c(\alpha, n)\cdot (\log n)^{\alpha}}^{\frac{-1}{2+2\alpha}}\ge j-1.
\end{align*}
Let $J$ denote the middle quantity on the right-hand side (implicitly depends on $\alpha$ and $n$).
We can decompose the summation 
${H^{\mathcal S}_n(p)}$ into two parts. The first part consists of indices $j\le J$, \vspace{-0.3em}
\[
{H^{\mathcal S}_{n, 1}(p)}:=\sum_{j=1}^{J-1} \min\Brace{p_{I_j}, j\cdot \log n}
\le (\log n)\sum_{j=1}^{J} j =\frac{\log n}{2}(J+1)J
\le \frac12\Paren{\frac{c_\alpha}{n}}^{\frac{-1}{1+\alpha}}.
\]
Correspondingly, the second part consists of indices $j\ge J$. 
For these indices $j$, we have $L_j\le j\cdot \log n$. 
Recall that $I_j'$ specifies the range of $i$ satisfying $p_i\in I_j$. 
Then the second part satisfies
\[
{H^{\mathcal S}_{n, 2}(p)}:=\sum_{j=J}^{n} \min\Brace{p_{I_j}, j\cdot \log n}
\le 7\log n+\sum_{j=J\lor4}^{n} L_j
\le
7\log n+\Paren{1-\frac1{J\lor4}}^{-\frac2\alpha} \Paren{\frac{c_\alpha}{n}}^{\frac{-1}{1+\alpha}}
,
\]
\vspace{-0.15em}where the first inequality follows by the fact that the intervals $I_j'$ are consecutive. 
Also note that the boundary case $j=J$ is covered in both $H^{\mathcal S}_{n,1}$ and $H^{\mathcal S}_{n,2}$ under different conditions. 
Depending on the value of the normalizing constant $c_\alpha$, 
the following implications are immediate and apply to all power parameters $\alpha>0$.
If $c_\alpha\le \tilde{\mathcal{O}}(n)\alpha^{{1}/{(1+\alpha)}}$, \vspace{-0.5em}
\[
{H^{\mathcal S}_n(p)}\le H^{\mathcal S}_{n,1}(p)+H^{\mathcal S}_{n,2}(p) 
\le 7\log n+e\Paren{\frac{n}{c_\alpha}}^{\frac{1}{1+\alpha}}.
\]
If $c_\alpha\ge \tilde{\Omega}(n)$, then ${H^{\mathcal S}_n(p)}$ is 
bounded by $\log n$ since the probabilities $p_i$ are at most $1/c_\alpha$.

These bounds can be refined given the knowledge of $k$. 
Note that for $\alpha\le1$, the support size $k$ must be finite for the normalizing constant to be well-defined. 
On the other hand, for $\alpha>1$, it is common to assume $k=\infty$, which we adopt here.

Then, for $\alpha> 1$, the above bound derivations yield
\[
{H^{\mathcal S}_n(p)}= H^{\mathcal S}_{n,1}(p)+H^{\mathcal S}_{n,2}(p)
\le 7\log n+2\Paren{\frac{n}{c_\alpha}}^{\frac{1}{1+\alpha}}
\le 7\log n+2\Paren{n(\alpha-1)}^{\frac{1}{1+\alpha}},
\]
where we lower bound $c_\alpha$ by 
$\max\{1, 1/(\alpha-1)\}$.

Next, we improve the upper bound for $\alpha<1$.
Note that the normalizing constant admits
\[
\frac{k^{1-\alpha}}{1-\alpha} +\frac{\alpha}{1-\alpha}\ge 1+ \int_{1}^{k} x^{-\alpha} dx \ge c_\alpha 
= \sum_{i=1}^k i^{-\alpha} \ge \int_{1}^{k+1} x^{-\alpha} dx 
= \frac{(k+1)^{1-\alpha}}{1-\alpha}-\frac{1}{1-\alpha}.
\] 
Then for $k\ge \sqrt{n}$, the previous upper bound yields
\[
{H^{\mathcal S}_n(p)}-7\log n
\le e\Paren{\frac{n}{c_\alpha}}^{\frac{1}{1+\alpha}}
\le e\Paren{\frac{n(1-\alpha)}{(k+1)^{1-\alpha}-1}}^{\frac{1}{1+\alpha}}
\le e\Paren{\frac{en}{n^{\frac{1-\alpha}{2}}}}^{\frac{1}{1+\alpha}}
\le e^2\sqrt{n}
,
\]
where we utilize the inequality $((k+1)^{1-\alpha}-1)/(1-\alpha)\ge (k+1)^{1-\alpha}/e$. 
Furthermore, one can bound ${H^{\mathcal S}_n(p)}$ by $k$ since it is at most the sum of $p_{I_j}$. 
Combined, these two results yield 
\[
{H^{\mathcal S}_n(p)}
\le 7\log n+e^2\sqrt n\Paren{\min\Brace{\frac{k}{\sqrt n}, \Paren{\frac{\sqrt n}{k}}^{\frac{1-\alpha}{1+\alpha}}}}.
\]
To complete the picture, we consider the case of $\alpha=1$. 
Note that $c_\alpha=\sum_{i=1}^k i^{-1}>\log k$. 
Hence for $\alpha=1$, the above reasoning implies that \vspace{-0.3em}
\[
{H^{\mathcal S}_n(p)}\le \min\Brace{k, 2\sqrt{\frac{n}{\log k}}}. 
\]
Finally, we also analyze the case of $\alpha>1$ with truncation at $k$. 
Our analysis mainly relies on the following inequality that provides a tight lower bound on $c_\alpha$. 
\[
c_\alpha = \sum_{i=1}^k i^{-\alpha} \ge \int_{1}^{k} x^{-\alpha} dx 
=\frac{1-k^{-(\alpha-1)}}{\alpha-1}\ge
\begin{cases}
\frac{\log k}{e} & \text{if } \alpha-1<\frac{1}{\log k};\\
\frac{1-1/e}{\alpha-1} & \text{if } \alpha-1\ge\frac{1}{\log k}.
\end{cases}
\]

To summarize, for different values of $\alpha$, the quantity ${H^{\mathcal S}_n(p)}$ admits
\[
{H^{\mathcal S}_n(p)}-  
7\log n
\le
k\wedge
e^2 \Paren{n\cdot U(\alpha, k)}^{\frac{1}{1+\alpha}},
\]
where 
\[
U(\alpha, k):=
\begin{cases}
\alpha-1& \text{if } \alpha\ge 1+\frac{1}{\log k};\\
\vspace{0.1em}
\frac{1}{\log k} & \text{if } 1\le \alpha <1+\frac{1}{\log k};\\ 
k^{\alpha-1} & \text{if }  0\le\alpha<1.
\end{cases}
\]
Note that unless $\alpha<1$ and $k\approx\sqrt n$, 
all the bounds are better than $\Theta(\sqrt n)$. 
In addition, the derivation above already shows 
that the bounds are tight up to logarithmic factors. 
Reorganizing the terms yields the desired result:
\emph{
For any power-law distribution $p$ with power $\alpha\ge 0$, 
\[
\frac{{H^{\mathcal S}_n(p)}-  
7\log n}{e^2}
\le 
\begin{cases}
n^{\frac1{1+\alpha}}& \text{if } \alpha\ge 1+\frac{1}{\log k};\\
\Paren{\frac{n}{\log n}}^{\frac1{1+\alpha}} & \text{if } 1\le \alpha <1+\frac{1}{\log k};\\ 
\sqrt n\Paren{\frac{k}{\sqrt n}\wedge \Paren{\frac{\sqrt n}{k}}^{\!\frac{1-\alpha}{1+\alpha}}} & \text{if } 0\le\alpha<1.
\end{cases}
\]}
\!As a remark, for $\alpha=0$, the distribution becomes a uniform distribution with support size $k$. 
The above result also covers this case, for which the upper bound simplifies to  $k\land (n/k)$.
\end{proof}

\subsection{Histogram Distributions}
A distribution $p\in \Delta_\cX$ is a \emph{$t$-histogram distribution} 
if there is a partition of $\cX$ into $t$ parts such that $p$ has 
the same probability value over all symbols in each part.   

Besides the long line of research on histograms reviewed in~\cite{ioannidis2003history}, 
the importance of histogram distributions rises 
with the rapid growth of data sizes in numerous 
engineering and science 
applications in the modern era. 

For example, 
in scenarios where processing the complete data set is inefficient or even impossible, 
a standard solution is to partition/cluster the data into groups 
according to the task specifications and element similarities, 
and randomly sample from each group to obtain a subset of the data to use. 
This naturally induces a histogram distribution,
with each data point being a symbol in the support. 

The work of~\cite{hao2019doubly} studies the class of $t$-histogram distributions 
and obtains the following upper bound \vspace{-0.25em}
\[
{H^\mathcal{S}_n(p)}\le \tilde{\mathcal{O}}\Paren{\min \Brace{(nt^2)^{\frac13}, \sqrt{n}}}. \vspace{-0.25em}
\]
Our contribution is establishing its optimality. 
\begin{Theorem}\label{thm:histgram}
For any $t, n\in \mathbb Z^+$, there exists 
a $t$-histogram distribution $p$ such that \vspace{-0.25em}
\[
{H^\mathcal{S}_n(p)}\ge \tilde{\Omega}\Paren{\min \Brace{(nt^2)^{\frac13}, \sqrt{n}}}. 
\]\par\vspace{-1em}
\end{Theorem}
Note that uniform distributions correspond to $1$-histograms, 
for which the bounds reduce to $\tilde{\Theta}(n^{1/3})$.
\begin{proof}
Again, recall that the quantity of interest is
 \[
 {H^{\mathcal S}_n(p)}= \sum_{j\ge 1}\min\Brace{p_{I_j}, j\cdot \log n}.
 \]
Our construction depends on the value of $t$ as follows.
Let $A\cdot\{B\}$ denote the length-$A$ constant sequence of value $B$. 
If $t=1$, then the distribution $p$ has the following form 
\[
p := \tilde{\Theta}(n^{1/3})\cdot \{p_0\in I_{n^{1/3}}\},
\]
where $p_0$ is a properly chosen probability in $I_{n^{1/3}}$ so that $p$ is well-defined, and the range of support of distribution $p$ is irrelevant for our purpose and hence unspecified.
If $2 \le t<n^{1/4}/(2\sqrt{\log n})$, then for some parameter $s\ge 0$ 
to be determined,
the distribution $p$ has the following form 
\[
p := L\cdot\Brace{\frac{1}{n^{2}}}\bigcup
\Paren{\bigcup_{j=s+1}^{s+t-1} 
\Paren{(j\log n)\cdot \Brace{j^2\frac{\log n}{n}}}},
\]
where the probability values are sorted according to the ordering they appear above, and $L$ is a properly chosen to make the probabilities sum to $1$. 
For the distribution to be well-defined, we require
\begin{align*}
\sum_{j=s+1}^{s+t-1} 
(j\log n)\cdot \Paren{j^2\frac{\log n}{n}}
\le 1
&\ \Longleftarrow\ 
 t (s + t)^3\le \frac{n}{\log^2 n}
\ \Longleftarrow\ 
s\le \Paren{\frac{n}{t\log^2 n}}^{1/3}\!\!-t,
\end{align*}
where the last inequality is valid given that $t<n^{1/4}/(2\sqrt{\log n})$. 
Let $s$ be the maximum integer satisfying the inequality above. 
Then, the quantity ${H^{\mathcal S}_n(p)}$ admits the lower bound
\[
{H^{\mathcal S}_n(p)}
\ge \sum_{j=s+1}^{s+t-1} (j\log n)
\ge \frac{(2s+t)(t-1)}{2}\log n
\ge \frac{1}{4}\Paren{\frac{n}{t\log^2 n}}^{1/3}\!\!t\log n
=\Omega((nt^2\log n)^{1/3}). 
\]
Finally, if $t\ge n_0:=n^{1/4}/(2\sqrt{\log n})$, 
then the distribution $p$ has the following form 
\[
p := (t-n_0+1)\cdot\Brace{p_0}\bigcup
\Paren{\bigcup_{j=1}^{n_0-1} 
\Paren{(j\log n)\cdot \Brace{j^2\frac{\log n}{n}}}},
\]
where $p_0$ is a properly chosen to make the probabilities sum to $1$. 
By the reasoning for the last case, the distribution $p$ is well-defined. 
In addition, the quantity ${H^{\mathcal S}_n(p)}$ satisfies
\[
{H^{\mathcal S}_n(p)}
\ge \sum_{j=1}^{n_0-1} (j\log n)
\ge \frac{n_0(n_0-1)}{2}\log n
\ge \Omega(\sqrt n). 
\]
Consolidating these results yields the desired lower bound. 
\end{proof}

\section{Competitive Estimation of Distributions and Their Entropy}

\subsection{Competitive Distribution Estimation} 

Competitive estimation calls for an estimator that competes 
with the instance-by-instance performance
of a genie knowing more information, but reasonably restricted. 
Denote by $\ell_{\text{\tiny KL}}(p,q)$ the KL divergence.
Introduced in~\cite{orlitsky2015competitive}, 
the formulation considers the collection $\mathcal{N}$ 
of all natural estimators, and 
shows that a simple variant $\hat p^{\text{\tiny GT}}$ of the Good-Turing 
estimator achieves 
\[
\ell_{\text{\tiny KL}}(p,\hat p^{\text{\tiny GT}}_{\!_{X^n}})\!
-\!\min_{\hat p\in \mathcal{N}}\ell_{\text{\tiny KL}}(p,\hat p_{\!_{X^n}})
\le \mathcal{O}\Paren{\frac{1}{n^{1/3}}}
\]
for every distribution $p$ and with high probability. 
We refer to the left-hand side as the \emph{excess loss} of estimator 
$\hat p_{\text{\tiny GT}}$ with respect to the best natural estimator, 
and note that it vanishes at a rate independent of $p$. 
For a more involved estimator in~\cite{acharya2013optimal}, the excess loss vanishes 
at a faster rate of $\tilde{\mathcal{O}}(\min\{1/\sqrt n, |\cX|/n\})$, which 
is optimal up to logarithmic factors for every estimator and 
the respective worst-case distribution. 
For the $\ell_1$ distance, the work of~\cite{valiant2016instance} derives a similar result.

These estimators track the loss of the best natural estimator for each distribution.
Yet an equally important component, the excess loss bound, is still of the worst-case nature. 
For a fully adaptive guarantee, \cite{hao2019doubly} design an estimator $\hat p^\star$ 
that achieves a $\mathcal D_n/n$ excess loss, i.e., 
\[
\ell_{\text{\tiny KL}}(p,\hat p^\star_{\!_{X^n}})\!
-\!\min_{\hat p\in \mathcal{N}}\ell_{\text{\tiny KL}}(p,\hat p_{\!_{X^n}})
\le\tilde{\mathcal{O}}\Paren{\frac{\mathcal D_n}{n}}, 
\]
 for every $p$ and $X^n\sim p$, with high probability. 
 Utilizing the adaptiveness of $\mathcal D_n$ to the simplicity of distributions, 
the paper derives excess-loss bounds for several important distribution families, 
and proves the estimator's optimality under various of classical and
modern learning frameworks. 

\paragraph{New results}
While the work of~\cite{hao2019doubly} provides an appealing upper bound 
on the excess loss, 
it is not exactly clear how good this bound is as a matching lower bound is missing. 
In this work, we complete the picture by showing 
that the $\mathcal D_n/n$ bound is essential for competitive 
estimation and optimal up to 
logarithmic factors of $n$. 
\setcounter{Theorem}{4}
\begin{Theorem}[Minimal excess loss] 
For any $n, \mathcal{D}\in \mathbb N$ and 
distribution estimator $\hat p'$, 
there is a distribution $p$ such that
with probability at least $9/10$, we have both  \vspace{-0.25em}
\[
\mathcal{O}(\log n+\mathcal{D}) \ge  \mathcal D_n   \vspace{-0.25em}
\] 
and  \vspace{-0.25em}
\[
\ell_{\text{\tiny KL}}(p,\hat p'_{\!_{X^n}})\!
-\!\min_{\hat p\in \mathcal{N}}\ell_{\text{\tiny KL}}(p,\hat p_{\!_{X^n}})
\ge 
\Omega\Paren{\frac{\mathcal{D}}{n}}. 
\]\par \vspace{-0.5em}
\end{Theorem}
According to Theorem~\ref{thm:entro_eql_dim}, we can replace $\mathcal D_n$ by multiples 
$\tilde\Theta(H(\Phi^n))$ of the profile entropy 
 in both the upper and lower bounds. \vspace{-0.5em}
\begin{proof}
Denote $s:=(D/\log n)^{1/2}$, $I:=\{s, s+1,\ldots, 2s\}$, 
and $P:=\cup_{i\in I} P_i:=\cup_{i\in I} U_i/n$ where
\[
U:=\bigcup_{i\in I} U_i:=\bigcup_{i\in I} \{i^2 \log^2 n, i^2 \log^2 n+1, \ldots, i^2\log^2 n+i\log n\},
\]
where $D\lesssim\sqrt{n/\log n}$ for the total to be at most $n$. 
Let $A\cdot\{B\}$ denote the length-$A$ constant sequence of value $B$. 
Let $C$ be the set of distributions in the form of 
\[
p:=L\cdot\Brace{\frac{1}{n^{2}}}\bigcup
\Paren{\bigcup_i (i\log n)\cdot \Brace{q_i\text{ or }q_i': nq_i=i^2 \log^2 n, nq_i'= i^2 \log^2 n+i\log n}}.
\] 
where the probability values are sorted according to the ordering they appear above, $L$ is a proper variable that makes the probabilities sum to 1,
and the range of support of distribution $p$ is irrelevant for our purpose and hence unspecified. 
Equip a uniform prior over $C$ (equivalently, construct a random distribution).
We have several claims in order:
\begin{itemize}
\item 
For any $i\in I$ and $\mu\in U_i$, by the construction and independence, 
\begin{align*}
\Pr(\varphi_\mu =1| q_i \text{ is chosen})
&\approx
(i\log n)\cdot\Paren{
\Pr(\Poi(nq_i)=\mu)\cdot \Paren{\Pr(\Poi(nq_i)\not=\mu)}^{i\log n-1}}\\
&\approx
(i\log n)\cdot\Paren{
\frac{1}{\sqrt{nq_i}}\cdot \Paren{1-\frac{1}{\sqrt{nq_i}}}^{i\log n-1}}\\
&\ge 
\Omega(1). 
\end{align*}
Similarly, we have $\Pr(\varphi_\mu =1| q_i' \text{ is chosen})\ge \Omega(1)$. Hence,
\[
\Pr(\varphi_\mu =1)\ge \Omega(1).
\]
\item 
For any $i\in I$ and $\mu\in U_i$, by Bayes' rule,
\begin{align*}
\Pr( q_i \text{ is chosen}|\varphi_\mu =1)
=
\frac{\Pr(\varphi_\mu =1| q_i \text{ is chosen})\cdot 0.5}{\Pr(\varphi_\mu =1)}
\ge
\Omega(1). 
\end{align*}
Similarly, we have $\Pr( q_i' \text{ is chosen}|\varphi_\mu =1)\ge \Omega(1)$. 
\item For any $i\in I$ and $\mu\in U_i$, the value of $M_\mu$, 
the total probability of symbols appearing $\mu$ times, is $q_i$
if $\varphi_\mu=1$ and $q_i$ is chosen; and is $q_i'$
if $\varphi_\mu=1$ and $q_i$ is chosen. 
Any estimator $E_\mu$ will incur an expected absolute error of $\Omega(i(\log n)/n)$ 
in estimating $M_\mu$ given $\varphi_\mu=1$.
\item Note that for any $\alpha\in[0,1]$ and $x,y>0$,
\[
\alpha(y-z)^2+(1-\alpha)(z-x)^2
\ge
\alpha(1-\alpha)(x-y)^2.
\]
 
\item Therefore, the expected squared Hellinger distance $H^2(\cdot, \cdot)$ of any estimator $E_\mu$ 
in estimating $(M_\mu)_{\mu\ge 0}$ satisfies
\begin{align*}
\frac{1}{2}\sum_{\mu\ge 0} \EE\Paren{\sqrt{E_\mu}-\sqrt{M_\mu}}^2
&\ge 
\frac{1}{2}\sum_{i\in I}\sum_{\mu\in U_i} \EE\Br{\Paren{\sqrt{E_\mu}-\sqrt{M_\mu}}^2 \big\vert \varphi_\mu=1} \Pr(\varphi_\mu=1)\\
&=
\frac{1}{2}\sum_{i\in I}\sum_{\mu\in U_i}
 \EE\Br{\Paren{\frac{E_\mu-M_\mu}{\sqrt{E_\mu}+\sqrt{M_\mu}}}^2 \bigg\vert \varphi_\mu=1} \Pr(\varphi_\mu=1)\\
&\ge
\sum_{i\in I}(i\log n)\cdot \Omega\Paren{\frac{(i\log n)/n}{\sqrt{i^2(\log^2 n)/n}}}^2\\
&\ge
s\cdot \Omega\Paren{\frac{s\log n}{n}}\\
&= 
\Omega\Paren{\frac{D}{n}}.
\end{align*}
\item Consequently, by the inequality $\infdiv{P}{Q}\ge 2H^2(P, Q)$, 
\[
\EE\Br{\infdiv{E}{M}}\ge \EE\Br{2H^2(E,M)}\ge \Omega\Paren{\frac{D}{n}}.
\]
\item Finally, the value of $\EE[D(X^n)]$ is 
at most $\mathcal{O}(\log n+s(s\log n))=\mathcal{O}(\log n+D)$. \qedhere
\end{itemize}
\end{proof}
\vspace{-0.75em}

\subsection{Competitive Entropy Estimation}

The next theorm shows that for \emph{every} distribution and among all plug-in entropy estimators,
the distribution estimator in~\cite{hao2019doubly}
is as good as the one that performs best in estimating the actual distribution.

Denote by $\mathcal{N}$ the collection of all natural estimators. 
Write $|H(p)-H(q)|$ as $\ell_H(p,q)$ for compactness 
and the KL-divergence between $p,q\in\Delta_\cX$ 
as $\ell_{\text{\tiny KL}}(p,q)$.
\setcounter{Theorem}{3}
\begin{Theorem}[Competitive entropy estimation]\label{thm:comp_entro}
For any distribution $p$, sample $X^n\sim p$ with profile $\Phi^n\!:=\va(X^n)$,
and $\hat p_{\!_{X^n}}^{_{\mathcal{N}}}:=\argmin_{\hat p\in \mathcal{N}}\ell_{\text{\tiny KL}}(p,\hat p_{\!_{X^n}})$, 
we have \vspace{-0.5em}
\[
\ell_H(p,\hat p^\star_{\!_{X^n}})\!-\!\ell_{H}(p,\hat p_{\!_{X^n}}^{_{\mathcal{N}}})
\le\tilde{\mathcal{O}}\Paren{\!\sqrt{\frac{\C{H\!(\Phi^n)}}{n}}}\!. 
\]\par\vspace{-0.75em}
with probability at least $1-\mathcal{O}(1/n)$.
\end{Theorem}

\vspace{-0.5em}
\begin{proof}
Given any natural estimator and a sample $X^n\sim p$, we denote by $q$ the distribution estimate. The entropy of $q$ differs from the true entropy by
\begin{align*}
H(q)-H(p)
&=
-\sum_x q_x\log q_x + \sum_x p_x\log p_x\\
&=
\sum_x p_x\log p_x -\sum_x p_x\log q_x + \sum_x p_x\log q_x - \sum_x q_x\log q_x\\
&=
\sum_x p_x\log\frac{p_x}{q_x} + \sum_x (p_x-q_x) \log q_x\\
&=
\ell_{\text{\tiny KL}}(p,q) + \sum_x (p_x-q_x) \log q_x.
\end{align*}
Denote by $P_\mu(X^n)$ and $Q_\mu(X^n)$ the total probability
that distributions $p$ and $q$ assign to symbols with multiplicity $\mu$. 
Since $q$ is induced by a natural estimator, 
we also write $q_\mu(X^n)$ for the probability that $q$
assigns to \emph{each} symbol with multiplicity $\mu$ in $X^n$.
Recall that prevalence $\varphi_\mu(X^n)$ denotes the number
of symbols with multiplicity $\mu$ in $X^n$.
Therefore, $Q_\mu(X^n)=\varphi_\mu(X^n)\cdot q_\mu(X^n)$. 

Henceforth, whenever it is clear from the context, we suppress $X^n$ in related expressions. 
Then, the second term on the right-hand side satisfies
\begin{align*}
\sum_x (p_x-q_x) \log q_x
&=
\sum_x (\sum_\mu \indic_{\mu_x=\mu}\cdot p_x-\sum_\mu \indic_{\mu_x=\mu}\cdot q_\mu) \log (\sum_\mu \indic_{N_x=\mu}\cdot q_\mu)\\
&=
\sum_x \sum_\mu \indic_{\mu_x=\mu}\cdot (p_x-q_\mu) \log q_\mu\\
&=
\sum_\mu (\sum_x \indic_{\mu_x=\mu}\cdot p_x-\sum_x \indic_{\mu_x=\mu}\cdot q_\mu) \log q_\mu\\
&=
\sum_\mu \Paren{P_\mu-Q_\mu} \log q_\mu.
\end{align*}
Let $q_{\text{min}}$ be the smallest nonzero probability of $q$. By the triangle inequality and Pinsker's inequality,
\begin{align*}
\Abs{\sum_\mu \Paren{P_\mu-Q_\mu} \log q_\mu}
&
\le \sum_\mu \Abs{\Paren{P_\mu-Q_\mu} \log q_\mu}\\
&
\le  |\log q_{\text{min}}|\sum_\mu \Abs{P_\mu-Q_\mu}\\
&
\le  |\log q_{\text{min}}|\sqrt{2\ell_{\text{\tiny KL}}(P,Q)}. 
\end{align*}
By definition, $\hat p_{\!_{X^n}}^{_{\mathcal{N}}}\!=\!\argmin_{\hat p\in \mathcal{N}}\ell_{\text{\tiny KL}}(p,\hat p_{\!_{X^n}})$. \vspace{-0.25em}
Now we show that if a  
symbol $x$ has multiplicity~$\mu$,
the estimator $\hat p^{_{\mathcal{N}}}$ will assign
a probability mass of $P_\mu/\varphi_\mu$. 
In other words, $\hat P^{_{\mathcal{N}}}_\mu = P_\mu$ 
since $p^{_{\mathcal{N}}}\in {\mathcal{N}}$.
Indeed, the corresponding KL-divergence values differ by 
\begin{align*}
\sum_x p_x\log\frac{p_x}{q_x}
-
\sum_x \sum_\mu \indic_{\mu_x=\mu}\cdot p_x\log\frac{p_x}{P_\mu/\varphi_\mu}
&=
\sum_x p_x\log\frac{1}{q_x}
-
\sum_x \sum_\mu \indic_{\mu_x=\mu}\cdot p_x\log\frac{\varphi_\mu}{P_\mu}\\
&=
\sum_x \sum_\mu \indic_{\mu_x=\mu}\cdot p_x\log\frac{P_\mu}{\varphi_\mu q_\mu}\\
&=
\sum_\mu P_\mu\log\frac{P_\mu}{Q_\mu}=
\ell_{\text{\tiny KL}}(P,Q)\ge 0.
\end{align*}
Then, the above equalities yield that,
\begin{align*}
H(\hat p^{_{\mathcal{N}}})-H(p)
=
\ell_{\text{\tiny KL}}(p,\hat p^{_{\mathcal{N}}}) + \sum_\mu \Paren{P_\mu-\hat P^{_{\mathcal{N}}}_\mu} \log p^{_{\mathcal{N}}}_\mu
=
\ell_{\text{\tiny KL}}(p,\hat p^{_{\mathcal{N}}})
=
\min_{\hat p\in \mathcal{N}}\ell_{\text{\tiny KL}}(p,\hat p_{\!_{X^n}}).
\end{align*}
Next consider the other estimator $\hat p^\star$, which is also natural. 
Let $\mathcal D_n=\mathcal D(\Phi^n)$ be the profile dimension of $X^n$.
By the results in \cite{hao2019doubly},  estimator $\hat p^\star$ 
achieves a $\mathcal D_n/n$ excess loss, i.e., \vspace{-0.25em}
\[
\ell_{\text{\tiny KL}}(p,\hat p^\star_{\!_{X^n}})\!
-\!\min_{\hat p\in \mathcal{N}}\ell_{\text{\tiny KL}}(p,\hat p_{\!_{X^n}})
=
\ell_{\text{\tiny KL}}(P(X^n),\hat P^{\star}(X^n))
\le\tilde{\mathcal{O}}\!\Paren{\frac{\mathcal D_n}{n}}, 
\]
 for every $p$ and $X^n\sim p$, with probability at least $1-\mathcal{O}(1/n)$.
 In addition, by its construction, the minimum probability 
 $\hat p_{\text{min}}^\star(X^n)$ is at least $1/n^4$. 
Therefore, with probability at least $1-\mathcal{O}(1/n)$,
\[
\Abs{\sum_x (p_x-\hat p^\star_x) \log \hat p^\star_x}
=
\Abs{\sum_\mu \Paren{P_\mu-\hat P_\mu^\star} \log \hat p_\mu^\star}
\le  
|\log \hat p_{\text{min}}^\star|
\cdot \sqrt{2\ell_{\text{\tiny KL}}(P,\hat P_\mu^\star)}
\le
\tilde{\mathcal{O}}\!\Paren{\!\sqrt{\frac{\mathcal D_n}{n}}}\!.
\]
Finally, the triangle inequality combines the above results and yields
\begin{align*}
\ell_H(p,\hat p^\star)
-
\ell_{H}(p,\hat p^{_{\mathcal{N}}})
&
=|H(p)-H(\hat p^\star)|
-
|H(p)-H(\hat p^{_{\mathcal{N}}})|\\
&
=
\Abs{\ell_{\text{\tiny KL}}(p,\hat p^\star_x)
+\sum_x (p_x-\hat p^\star_x) \log \hat p^\star_x}
-
\Abs{\min_{\hat p\in \mathcal{N}}\ell_{\text{\tiny KL}}(p,\hat p_{\!_{X^n}})}\\
&\le 
\Abs{\ell_{\text{\tiny KL}}(p,\hat p^\star_x)
-
\min_{\hat p\in \mathcal{N}}\ell_{\text{\tiny KL}}(p,\hat p_{\!_{X^n}})}
+
\Abs{\sum_x (p_x-\hat p^\star_x) \log \hat p^\star_x}\\
&=
\ell_{\text{\tiny KL}}(P,\hat P_\mu^\star)
+
\tilde{\mathcal{O}}\!\Paren{\!\sqrt{\frac{\mathcal D_n}{n}}}\\
&\le
\tilde{\mathcal{O}}\!\Paren{\!\sqrt{\frac{\mathcal D_n}{n}}}.
\end{align*}
This together with Theorem~\ref{thm:entro_eql_dim} completes the proof. 
\end{proof}

\section{Optimal Profile Compression}
While a labeled sample contains all information, 
for many modern applications, 
such as property estimation and differential privacy, 
it is sufficient~\cite{orlitsky2004modeling} or even necessary to provide only the 
profile~\cite{suresh2019differentially}. 
Hence, this section focuses on the lossless compression of profiles.

For any distribution $p$, it is well-known that 
the minimal expected codeword length (MECL) for 
losslessly compressing a sample 
$X^n\sim p$ is approximately $nH(p)$, 
which increases linearly in $n$
as long as $H(p)$ is bounded away from zero.

On the other hand, by the 
Hardy-Ramanujan formula~\cite{hardy1918asymptotic},
the number $\mathbb{P}(n)$ of integer partitions of $n$, 
which happens to equal to the number
of length-$n$ profiles, satisfies 
\[
\log \mathbb{P}(n)= 2\pi\sqrt{\frac{n}{6}}(1+o(1)). 
\]
Consequently, the MECL for losslessly compressing the sample 
profile $\Phi^n\sim p$ is at most $\mathcal{O}(\sqrt{n})$, 
a number potentially much smaller than $nH(p)$. 

By Shannon's source coding theorem, 
the profile entropy $H(\Phi^n)$ is the \emph{information-theoretic
limit} of MECL for the lossless compression of profile $\Phi^n\sim p$.
Below, we present explicit block and sequential 
profile compression schemes achieving this 
entropy limit, up to logarithmic factors of $n$. 
\vspace{-0.25em}

\subsection{Block Compression}
The block compression algorithm we propose
is intuitive and easy to implement.

Recall that the profile of a sequence $x^n$ is the multiset $\varphi(x^n)$
of multiplicities associated with symbols in $x^n$. 
The ordering of elements in a multiset is not informative. 
Hence equivalently, we can compress $\varphi(x^n)$
into the set $\mathcal C(\varphi(x^n))$ of corresponding multiplicity-prevalence pairs,
i.e.,
\[
\mathcal C(\varphi(x^n)):= \{(\mu, \varphi_\mu(x^n)):\mu\in \varphi(x^n)\}.
\]
The number of pairs in $\mathcal C(\varphi(x^n))$ is equal to the profile dimension 
$\mathcal{D}(\varphi(x^n))$. In addition, both a prevalence and its multiplicity
are integers in $[0,n]$, and storing the pair takes $2\log n$ nats. 
Hence, it takes at most $2(\log n)\cdot \mathcal{D}(\varphi(x^n))$ nats 
to store the compressed profile. 
By Theorem~\ref{thm:entro_eql_dim}, for any distribution $p\in \Delta_\cX$
and $\Phi^n\sim p$, \vspace{-0.25em}
\[
\E[2(\log n)\cdot \mathcal{D}(\Phi^n)] 
= 
\tilde{\Theta}(\C{H(\Phi^n)}). 
\]

\subsection{Sequential Compression} 
For any sequence $x^n$, 
the setting for sequential profile compression is that at time step $t\in [n]$, 
the compression algorithm knows only $\varphi(x^t)$ 
and sequentially encodes the new information. 
This is equivalent to providing the algorithm $\mu_{x_t}(x^{t-1})$ at time step $t$.

Suppress $x, x^{t}$ in the expressions for the ease of illustration. 
For efficient compression, 
we sequentially encode the profile $\va$ into a 
\emph{self-balancing binary search tree} $\mathcal T$,
with each node storing a 
multiplicity-prevalence pair $(\mu, \va_\mu)$ and $\mu$ being the search key. 
We present the algorithm details as follows. 
\begin{algorithm}
\caption{Sequential Profile Compression}
\begin{algorithmic} 
\INPUT sequence $(\mu_{x_t}(x^{t-1}))_{t=1}^n$, tree $\mathcal T=\varnothing$
\OUTPUT tree $\mathcal T$ that encodes the input sequence
\FOR {t = 1 to n}
\IF{$\mu:= \mu_{x_t}(x^{t-1}) \in \mathcal T$}
\IF{$\mu+1\in \mathcal T$}
\STATE $\va_{\mu+1}:=\mathcal T(\mu+1)\gets \mathcal T(\mu+1)+1$
\ELSE
\STATE add $(\mu+1, 1)$ to $\mathcal T$
\ENDIF
\STATE \textbf{if} $\va_\mu=1$ \textbf{then} delete $(\mu, \varphi_\mu)$ from $\mathcal T$ 
\STATE \textbf{else} $\va_{\mu}:=\mathcal T(\mu)\gets \mathcal T(\mu)-1$ \textbf{endif}
\ELSE
\STATE \textbf{if} $1\not\in \mathcal T$ \textbf{then} add $(1, 1)$ to $\mathcal T$ 
\STATE \textbf{else} $\mathcal T(1)\gets \mathcal T(1)+1$ \textbf{endif}
\ENDIF
\ENDFOR
\end{algorithmic}
\end{algorithm}
\vspace{-1em}

The algorithm runs for exactly $n$ iterations, 
with a $\mathcal{O}(\log n)$ per-iteration time complexity.
For an \iid sample $X^n\sim p$, the expected space complexity 
is again $\tilde{\Theta}(\C{H(\Phi^n)})$.

\section{Extensions and Additional Results}

\subsection{Multi-Dimensional Profiles}
As we elaborate below, the notion of profile generalizes 
to the multi-sequence setting. 

Let $\cX$ be a finite or countably infinite alphabet. 
For every $\vec n:=(n_1,\ldots, n_d)\in \mathbb N^d$ 
and tuple $ x^{\vec n}:=(x^{n_1}_1,\ldots, x^{n_d}_d)$ of sequences in $\cX^*$, 
the \emph{multiplicity} 
$\mu_y( x^{\vec n})$ of a symbol $y\in\cX$ 
is the vector of its frequencies in the tuple of sequences. 
The \emph{profile} of $x^{\vec n}$ is the multiset  
$\va(x^{\vec n})$ of multiplicities of the observed 
symbols~\cite{acharya2010classification, das2012competitive, charikar2019efficient}, 
and its \emph{dimension} is the number $\mathcal D(x^{\vec n})$ 
of distinct elements in the multiset. 
Drawing independent samples from $\vec p:=(p_1,\ldots, p_d)\in \Delta_\cX^d$, 
the \emph{profile entropy} is simply the entropy of the joint-sample profile.

Many of the previous results potentially generalize to this multi-dimensional setting. 
For example, the $\sqrt {2n}$ bound on 
$\mathcal D(x^{\vec n})$ in the 1-dimensional case becomes 
\setcounter{Theorem}{19}
\begin{Theorem}
For any $\cX$, $\vec n$, and $x^{\vec n}\in \cX^{\vec n}$, 
there exists a positive integer $r$ such that
\begin{align*}
\sum_i n_i
\ge
d\cdot \binom{d+r-1}{d+1},
\end{align*}
and
\begin{align*}
\mathcal D
\le 
\binom{d+r}{d}-1.
\end{align*} 
\end{Theorem}
This essentially recovers the $\sqrt{2n}$ bound for $d=1$. 

\begin{proof}
For simplicity, we suppress $x^{\vec n}$ in $\mathcal{D}(x^{\vec n})$. 
Let $\Delta_d$ denote the standard $d$-dimensional simplex.  
As each multiplicity corresponds to a vector in $\mathbb N^d$,
in the ideal case, 
the profile that has the maximum dimension $\mathcal D$
corresponds to the integer vectors in the scaled simplex $(r\cdot\Delta_d)$,
for some properly chosen parameter $r$. 
For the minimum value of such a parameter $r\in \mathbb Z^+$, we have 
\begin{align*}
\sum_i n_i
&\ge \sum_{t=0}^{r-1} \binom{t+d-1}{d-1}\cdot t \\
&=d\cdot \sum_{t=1}^{r-1} \binom{t+d-1}{d} \\
&=d\cdot \sum_{(t-1)=0}^{r-2} \binom{(t-1)+d}{(t-1)}\\
&=d\cdot \binom{d+r-1}{d+1},
\end{align*}
and
\begin{align*}
\mathcal D
\le 
\sum_{t=1}^{r} \binom{t+d-1}{t}
=
\binom{d+r}{d}-1
.
\end{align*}

Consolidating these two inequalities yields the desired result. 
\end{proof}

\subsection{Discrete Multi-Variate Gaussian}
Given a mean vector $\mu\in\mathbb{Z}^d$ and covariance matrix $\Sigma\in\mathbb{R}^{d\times d}$ with \emph{eigenvalues at least $1$},  the 
corresponding \emph{discrete $d$-dimensional Gaussian} is specified by its probability mass function
\[
p(x):=\frac{1}{C} \exp\Paren{-\frac{1}{2}\Paren{x-\mu}^T\Sigma^{-1}\Paren{x-\mu}}, \forall x\in \mathbb{Z}^d.
\]
where $C>0$ is a normalizing constant. 
Note that definition is slightly different from that induced by the 
 discretization procedure presented in Section~\ref{sec:discrete}.
The reason for adopting this definition 
(which is also standard in literature) 
is to simplify the subsequent reasoning. 
Let $\sigma_1^2\le \sigma_2^2\ldots\le \sigma_d^2$ be the $d$ eigenvalues of $\Sigma$, where $\sigma_1^2\ge 1$ by assumption. 
In this section, we show that for $d\ge 9$,
\[
H^\mathcal{S}_n(p)
\le 
\mathcal{O}(\log n)\Paren{1+\min\Brace{\frac{n}{C}, \gamma_d (\alpha_{\Sigma}\cdot \beta_{d,n})^d \cdot C}}
,
\]
where $\alpha_\Sigma:=\exp\Paren{6\sigma_d^2/\sigma_1^2}$ and $\beta_{d,n}:=\sqrt{(2\log n)/d}$, and $\gamma_d$ is a constant that depends only on $d$, which appears in Lemma~\ref{lattice_bound}.
Note that the above bound resembles that for univariate log-concave distributions (Theorem~\ref{thm:logconc_bound}). This result is not included in the main paper due to the different setting. 

\paragraph{Lower bound on $\boldsymbol C$}
First we bound the value of $C$ from below in terms of its eigenvalues and other parameters.
By symmetry, we can decompose the matrix $\Sigma$ as
\[
\Sigma = V\Lambda V^T,
\]
where $\Lambda$ is a diagonal matrix with $\Lambda_{ii}=\sigma_i^2$, and
$V$ is an orthonormal matrix whose $i$-th column is the eigenvector $v_i$ associated with $\sigma_i^2$. 

Partition the real space $\mathbb{R}^d$ into unit cubes whose vertices belong to $\mathbb Z^d$. 
For any two vectors $\tilde a,\tilde b\in \mathbb R^d$ that belong to the same unit cube, 
we want to bound the ratio between $p(\tilde a)$ and $p(\tilde b)$. 
Denote $a:=\tilde a-\mu$ and $b:=\tilde b-\mu$, and express $a$ and $b$ as linear combinations of eigenvectors,
\[
a:=\sum_{i=1}^d x_i\cdot v_i \text{ and } b:=\sum_{i=1}^d y_i\cdot v_i.
\]
The log-ratio between the corresponding probabilities satisfies
\begin{align*}
-2\log \frac{p(\tilde a)}{p(\tilde b)}
&=
a^T\Sigma^{-1}a
-b^T\Sigma^{-1}b\\
&=
(a+b)^T\Sigma^{-1}(a-b)\\
&=
\Paren{\sum_i (x_i+y_i)\cdot v_i^T}V \Lambda^{-1} V^T \Paren{\sum_i (x_i-y_i)\cdot v_i}\\
&=
\Paren{\sum_i (x_i+y_i)\cdot e_i^T}\Lambda^{-1} \Paren{\sum_i (x_i-y_i)\cdot e_i}\\
&=
\sum_i \sigma_i^{-2} (x_i^2-y_i^2).
\end{align*}\par\vspace{-1em}
Note that 
$\sum_i (x_i-y_i)^2 
= \norm{a-b}_2^2 
=\sum_i (\tilde a_i-\tilde b_i)^2\le d$
since  $\tilde a-\tilde b = a-b$ and $\tilde a, \tilde b$ belong to the same unit cube. 
Hence, we bound the absolute value of the ratio by 
\begin{align*}
2\Abs{\log \frac{p(\tilde a)}{p(\tilde b)}}
&=
\Abs{\sum_i \sigma_i^{-2} (x_i^2-y_i^2)}\\
&\le
\sum_i \sigma_i^{-2} \Abs{x_i^2-(x_i-(x_i-y_i))^2}\\
&\le
2\sum_i  \sigma_i^{-2}\Paren{x_i^2+(x_i-y_i)^2}\\
&\le
2\sigma_1^{-2}\Paren{\sum_i {x_i^2}+d}\\
&=
2\sigma_1^{-2}\Paren{\norm{\tilde a-\mu}_2^2+d}.
\end{align*}

Now, consider the hyper-ellipse $E$ induced by 
\[
\Paren{x-\mu}^T\Sigma^{-1}\Paren{x-\mu}
\le 
d.
\]
For any $x\in E$, simple algebra shows that $\norm{x-\mu}_2^2\le d\sigma_d^2$. 
Hence by the previous discussion, 
for any unit cube $U$ with vertices in $\mathbb Z^d$,
there exists a vertex $v_U$ of $U$ such that 
for any $x\in U\cap E$,
\[
\Abs{\log\frac{p(x)}{p(v_U)}}
\le \sigma_1^{-2} \Paren{\norm{x-\mu}_2^2+d}
\le \sigma_1^{-2} \Paren{d\sigma_d^2+d}
\le 2d\Paren{\frac{\sigma_d}{\sigma_1}}^2.
\]
Note that $x\in E$ is equivalent to $p(x)\ge \exp(-d/2)/C$.
The probability mass over $E$ is at least
\[
\int_{x\in E} p(x) dx\ge \int_{x\in E} \frac{\exp(-d/2)}{C}
=\frac{\exp(-d/2)}{C}\cdot \text{Vol}(E)
=\frac{\exp(-d/2)}{C}\cdot \frac{(\pi d)^{d/2}}{\Gamma(d/2+1)}\prod_{i=1}^d \sigma_i.
\]
On the other hand, this probability mass is at most
\begin{align*}
\int_{x\in E} p(x) dx
&= \sum_{U} \int_{x} p(x) \cdot \indic_{x\in E\cap U} dx
&\le \sum_{U} p(v_U)\cdot \exp\Paren{2d\Paren{\frac{\sigma_d}{\sigma_1}}^2}
&\le \exp\Paren{3d\Paren{\frac{\sigma_d}{\sigma_1}}^2}.
\end{align*}
Consolidating the lower and upper bounds and multiplying both sides by $C$ yield
\begin{align*}
&
\quad \quad C
\ge 
\exp\Paren{-3d\Paren{\frac{\sigma_d}{\sigma_1}}^2}\exp\Paren{-\frac d2}\cdot \frac{(\pi d)^{d/2}}{\Gamma(d/2+1)}\prod_{i=1}^d \sigma_i\\
&\Longrightarrow
C
\ge 
\exp\Paren{-3d\Paren{\frac{\sigma_d}{\sigma_1}}^2}\cdot \frac{(\pi d/e)^{d/2}}{\sqrt{e\pi (d/2)}(d/(2e))^{d/2}}\prod_{i=1}^d \sigma_i\\
&\Longrightarrow
C
\ge 
\exp\Paren{-3d\Paren{\frac{\sigma_d}{\sigma_1}}^2}\cdot \frac{(2\pi)^{d/2}}{\sqrt{e\pi (d/2)}}\prod_{i=1}^d \sigma_i\\
&\Longrightarrow
C
\ge 
\exp\Paren{-3d\Paren{\frac{\sigma_d}{\sigma_1}}^2}\prod_{i=1}^d \sigma_i.
\end{align*}
where the first implication follows by the lemma below.
\begin{Lemma}\label{bound_gamma}
For any integer or semi-integer $x\ge 1/2$,
\[
 \sqrt{2\pi x}\Paren{\frac{x}{e}}^x\le \Gamma(x+1)\le \sqrt{e\pi x}\Paren{\frac{x}{e}}^x.
\]
\end{Lemma}

\paragraph{Upper bound}
We proceed to bound ${H^\mathcal{S}_n(p)}
=\sum_{j\ge 1}\min\Brace{p_{I_j}, j\cdot \log n}$.

Below we assume that $C<n/\log n$, 
since otherwise $p(x)\le (\log n)/n,\forall x$, yielding an $\mathcal{O}(\log n)$
upper bound on $H^\mathcal{S}_n(p)$.
Then by definition, the last index $j$ such that $p_{I_j}>0$ satisfies
\[
(j-1)^2\frac{\log n}{n}\le \frac{1}{C}
\ \ \Longrightarrow\ \
j\le 1+\sqrt{\frac{1}{C}\frac{n}{\log n}}
\le 2\sqrt{\frac{1}{C}\frac{n}{\log n}}
\]
Denote by $J$ the quantity on the right-hand side. 
Then,
\begin{align*}
\sum_{j\ge 1}\min\Brace{p_{I_j}, j\cdot\log n}
\le
\sum_{j=1}^{J} j \log n
\le
J^2\log n
\le
\frac{4n}{C}. 
\end{align*}
Furthermore, by a reasoning
similar to that above, the collection of points $x\in \mathbb Z^d$ satisfying 
$p(x)\le 1/(Cn)=p(\mu)/n\le 1/n$ contributes at most 
$\mathcal{O}(\log n)$ to $H^\mathcal{S}_n(p)$. Hence we 
need to analyze only points $x$ satisfying $p(x)>1/(Cn)$. Equivalently,
points in
\[
E^\star:=\Brace{x\in\mathbb{Z}^d: \Paren{x-\mu}^T\Sigma^{-1}\Paren{x-\mu}\le 2\log n}.
\]
Clearly, these points contribute at most $\Abs{E^\star}$ to the sum. Noting that $E^\star$ is a discrete hyper-ellipse, 
we can bound its cardinality via the following lemma~\cite{bentkus1997lattice}. 
\begin{Lemma}\label{lattice_bound} 
Let $\mu\in \mathbb R^d$ be a mean vector, and $\Sigma\in \mathbb R^{d\times d}$ be a real covariance matrix with nonzero eigenvalues
$\sigma^2_1\le \ldots \sigma^2_d$.  
For any $d\ge 9$ and $t\ge \sigma_d^2$, the discrete ellipsoid
\[
E(t):=\Brace{x\in\mathbb{Z}^d: \Paren{x-\mu}^T\Sigma^{-1}\Paren{x-\mu}\le t}
\]
admits the following inequality on its cardinality,
\[
\Abs{E(t)}
\le 
\Paren{1+\frac{\gamma_d}{t} \frac{1}{\sigma_d^2} \Paren{\frac{\sigma_d}{\sigma_1}}^{2d+4}}
\frac{(\pi t)^{d/2}}{\Gamma(d/2+1)}\prod_{i=1}^d \sigma_i,
\]
where $\gamma_d>1$
 is a constant that depends only on $d$. 
\end{Lemma}

For simplicity, write $\alpha_\Sigma:=\exp\Paren{6\sigma_d^2/\sigma_1^2}$ and $\beta_{d,n}:=\sqrt{(2\log n)/d}$. Applying the above lemma to bound $|E^\star|$ (where $t=2\log n$) and combining the result with our lower bound on $C$ yield
\begin{align*}
\Abs{E(2\log n)}
&\le
\Paren{1+\frac{\gamma_d}{2\log n} \frac{1}{\sigma_d^2} \Paren{\frac{\sigma_d}{\sigma_1}}^{2d+4}}
\frac{(2\pi \log n)^{d/2}}{\Gamma(d/2+1)}
\exp\Paren{3d\Paren{\frac{\sigma_d}{\sigma_1}}^2}C\\
&\le
\Paren{1+\frac{\gamma_d}{2\log n} \frac{1}{\sigma_d^2} \Paren{\frac{\sigma_d}{\sigma_1}}^{2d+4}}
\frac{1}{\sqrt{\pi d}}\Paren{4e\pi \frac{\log n}{d}}^{d/2}\!
e^{3d\Paren{{\sigma_d}/{\sigma_1}}^2}C\\
&\le
\Paren{1+\frac{\gamma_d}{2\log n} \Paren{\frac{\sigma_d}{\sigma_1}}^{3d}}
\Paren{\frac{2\log n}{d}}^{d/2}\! e^{5d\Paren{{\sigma_d}/{\sigma_1}}^2}\!C\\
&\le
\gamma_d \Paren{\frac{\sigma_d}{\sigma_1}}^{3d}
\Paren{\frac{2\log n}{d}}^{d/2}\! e^{5d\Paren{{\sigma_d}/{\sigma_1}}^2}\!C\\
&\le
\gamma_d \Paren{\frac{2\log n}{d}}^{d/2}\! e^{6d\Paren{{\sigma_d}/{\sigma_1}}^2}\!C\\
&=
\gamma_d\ (\alpha_{\Sigma}\cdot \beta_{d,n})^{d} C,
\end{align*}
where the second step follows by Lemma~\ref{bound_gamma}.

To summarize, we have established the desired bound 
\[
H^\mathcal{S}_n(p)
\le 
\mathcal{O}(\log n)\Paren{1+\min\Brace{\frac{n}{C}, \gamma_d (\alpha_{\Sigma}\cdot \beta_{d,n})^d \cdot C}}.
\]

\bibliographystyle{plainnat}
\bibliography{supp}
\end{document}